\documentclass{elsart} 
\usepackage{named}
\usepackage{times}
\usepackage{helvet}
\usepackage{courier}
\usepackage{amsmath} 
\usepackage{graphicx} 
\DeclareGraphicsExtensions{.pdf,.png,.jpg,.eps}
\usepackage{color}
\usepackage{url}
\usepackage{wrapfig}
\usepackage{amssymb}

\def\cal{\mathcal}

\def\ar{\leftarrow}
\def\beq{\begin{equation}}
\def\eeq#1{\label{#1}\end{equation}}
\def\ba{\begin{array}}
\def\ea{\end{array}}
\def\bi{\begin{itemize}}
\def\ei{\end{itemize}}
\def\i#1{\hbox{\it #1\/}}

\def\no{\i{not}}
\def\ar{\leftarrow}
\def\rar{\rightarrow}
\def\lrar{\leftrightarrow}
\def\Lrar{\Leftrightarrow}
\def\nec{\mbox{\large\boldmath $\;\Leftarrow\;$}}

\def\circ{\hbox{\rm CIRC}}
\def\sm{\hbox{\rm SM}}
\def\comp{\hbox{\rm COMP}}

\def\lif{\hbox{\rm IF}}
\def\dg{\hbox{\rm DG}}
\def\cm{\hbox{\rm CM}}

\def\fsm{\hbox{\rm SM}}

\def\dia{\diamond} 
\def\mu#1{\mathit{\underline{#1}}}
\def\mi#1{\mathit{#1}}
\def\proof{\noindent{\bf Proof}.\hspace{3mm}}
\def\qed{\quad \vrule height7.5pt width4.17pt depth0pt \medskip}
\def\bi{\begin{itemize}}
\def\ii{\item}
\def\ei{\end{itemize}}

\def\bC{{\bf{c}}}
\def\vbC{{\wh{\bf{c}}}}
\def\bD{{\bf{d}}}
\def\vbD{{\wh{\bf{d}}}}
\def\bF{{\bf{f}}}
\def\bG{{\bf{g}}}
\def\wh{\widehat}
\def\mvis{\!=\!}
\def\sneg{\sim\!\!}

\def\bfxi{{\boldsymbol{\xi}}}

\def\false{\hbox{\sc false}}
\def\true{\hbox{\sc true}}
\def\sort#1{{\tt #1}}

\long\def\NBB#1{}

\long\def\BOC#1\EOC{\message{(Commented text )}}
\long\def\BOCC#1\EOCC{\message{(Commented text )}}
\long\def\BOCCC#1\EOCCC{\message{(Commented text )}}
\long\def\BOCCCC#1\EOCCCC{\message{(Commented text )}}
\long\def\optional#1{}

\newtheorem{remark}{Remark}
\newtheorem{lemma}{Lemma}
\newtheorem{definition}{Definition}
\newtheorem{example}{Example}

\newcommand{\cblu}{\color{blue}}
\newcommand{\cbla}{\color{black}}

\newcommand{\cgre}{\color{green}}

\begin{document}

\begin{frontmatter}
\title{First-Order Stable Model Semantics \\ with Intensional Functions}

\author{Michael Bartholomew \and Joohyung Lee}
\address{
School of Computing, Informatics, and Decision Systems Engineering \\
Arizona State University, Tempe, USA \\
{\tt \{mjbartho,joolee\}@asu.edu}
}

\begin{abstract}
In classical logic, nonBoolean fluents, such as the location of an object, can be naturally described by functions. However, this is not the case in answer set programs, where the values of functions are pre-defined, and nonmonotonicity of the semantics is related to minimizing the extents of predicates but has nothing to do with functions. We extend the first-order stable model semantics by Ferraris, Lee, and Lifschitz to allow intensional functions -- functions that are specified by a logic program just like predicates are specified. We show that many known properties of the stable model semantics are naturally extended to this formalism and compare it with other related approaches to incorporating intensional functions. Furthermore, we use this extension as a basis for defining {\em Answer Set Programming Modulo Theories (ASPMT)}, analogous to the way that Satisfiability Modulo Theories (SMT) is defined, allowing for SMT-like effective first-order reasoning in the context of ASP. Using SMT solving techniques involving functions, ASPMT can be applied to domains containing real numbers and alleviates the grounding problem. We show that other approaches to integrating ASP and CSP/SMT can be related to special cases of ASPMT in which functions are limited to non-intensional ones.
\end{abstract} 


\end{frontmatter}

\section{Introduction}\label{sec:intro}

Answer set programming (ASP) is a widely used declarative computing paradigm oriented towards solving knowledge-intensive and combinatorial search problems~\cite{lif08,bre11}. Its success is mainly due to the expressivity of its modeling language based on the concept of a stable model \cite{gel88} as well as the efficiency of ASP solvers thanks to intelligent grounding (the process that replaces schematic variables with variable-free terms) and efficient search methods that originated from propositional satisfiability (SAT) solvers. 

The grounding and solving approach makes ASP highly effective for Boolean decision problems but becomes problematic when the domain contains a large number of numerical values or a set of real numbers. This is in part related to the limited role of functions in the stable model semantics \cite{lif88} in comparison with what is allowed in classical logic: either functions are eliminated in the process of grounding, or they are associated with fixed, pre-defined interpretations forming an Herbrand universe. Such a limitation forces us to represent {\em functional} fluents by {\em predicates}, but not by {\em functions}.
For example, the following (non-ground) ASP rule represents that the water level does not change by default, where $t$ is a variable for time stamps, $l$ is a variable for integers, $\no$ stands for default negation, and $\sneg\ $ stands for strong negation:
\beq
\ba {rcl} 
  \i{WaterLevel}(t\!+\!1, l)\ &\ar & \i{WaterLevel}(t, l),\ 
     \no\ \sneg \i{WaterLevel}(t\!+\!1, l), \\ 
   & & \i{Time}(t), \i{Level}(l).
\ea
\eeq{water-asp}
An attempt to replace the predicate $\i{WaterLevel}(t,l)$ by equality using a function, e.g.  ``$\i{WaterLevel}(t)=l$,'' does not work under the standard stable model semantics: 
``\mbox{$\no\sneg (\i{WaterLevel}(t\!+\!1)=l)$}'' is not even syntactically valid because strong negation precedes equality, rather than an ordinary ASP atom. Besides,  \hbox{$\i{WaterLevel}(t)=l$} is false under any Herbrand interpretation unless $l$ is the term $\i{WaterLevel}(t)$ itself, implying that $\i{WaterLevel}(t)=\i{WaterLevel}(t+1)$ is always false. 

While semantically correct, a computational drawback of using a rule like~\eqref{water-asp} is that a large set of ground rules needs to be generated when the water level ranges over a large integer domain. Moreover, real numbers are not supported at all because grounding cannot even be applied. 

To alleviate the ``grounding problem,'' there have been recent efforts in integrating ASP with constraint solving, where functional fluents can be represented by constraint variables and computed without fully grounding their value variables, e.g., ~\cite{mellarkod08integrating,gebser09constraint,balduccini09representing,janhunen11tight}. 
Constraint ASP solvers have demonstrated significantly better performance over traditional ASP solvers on many domains involving a large set of numbers, but they do not provide a fully satisfactory solution to the problem above because the concept of a function is not sufficiently general. 
For example, one may be tempted to rewrite rule~\eqref{water-asp} in the language of a constraint ASP solver, such as {\sc clingcon}\footnote{
\url{http://potassco.sourceforge.net/}}---a combination of ASP solver {\sc clingo} and constraint solver {\sc gecode}, as
\beq
\ba {l}
  \i{WaterLevel}(t\!+\!1)\!=^\$\! l\ \ar 
        \i{WaterLevel}(t)\!=^\$\! l,\ 
             \no\ \neg(\i{WaterLevel}(t\!+\!1)\!=^\$\! l) 
\ea
\eeq{water-clingcon} 
where $=^\$$ indicates that the atom containing it is a constraint to be processed by constraint solver {\sc gecode} and not to be processed by ASP solver {\sc clingo}. The constraint variable $\i{WaterLevel}(t)$ is essentially a function that is mapped to a numeric value. However, this idea does not work either.\footnote{%
However, there is rather an indirect way to represent the assertion in the language of {\sc clingcon} using $\i{Ab}$ predicates: 
\[
\i{WaterLevel}(t+1)\!=^\$\! l \ar \i{WaterLevel}(t)\!=^\$\! l, \no\  \i{Ab}(t).
\]}
 While it is possible to say that $\i{WaterLevel}(t)=10$ and {$\i{WaterLevel}(t+1)=\i{WaterLevel}(t)$}
are true in the language of {\sc clingcon}, negation as failure ($\no$) in front of constraints does not work in the same way as it does when it is in front of standard ASP atoms. Indeed, rule~(\ref{water-clingcon}) has no effect on characterizing the default value of $\i{WaterLevel}(t)$ and can be dropped without affecting answer sets. 
This is because nonmonotonicity of the stable model semantics (as well as almost all extensions, including those of Constraint ASP) is related to the minimality condition on predicates but has nothing to do with functions. Thus, unlike with predicates, they do not allow for directly asserting that functions have default values.
Such an asymmetric treatment between functions and predicates in Constraint ASP makes the language of Constraint ASP less general than one might desire. 



It is apparent that one of the main obstacles encountered in the above work is due to an insufficient level of generality regarding functions.
Recently, the problem has been addressed in another, independent line of research to allow general first-order functions in ASP, although it was not motivated by efficient computation. Lifschitz [\citeyear{lifschitz12logic}] called such functions ``intensional functions''--- functions whose values can be described by logic programs, rather than being pre-defined, thus allowing for defeasible reasoning involving functions in accordance with the stable model semantics.
In \cite{cabalar11functional}, based on the notions of {\em partial functions} and {\em partial satisfaction}, functional stable models were defined by imposing minimality on the values of partial functions. The semantics presented in \cite{balduccini12aconservative} is a special case of the semantics from~\cite{cabalar11functional} as shown in \cite{bartholomew13onthestable}.
On the other hand, intensional functions defined  in~\cite{lifschitz12logic} do not require the rather complex notions of partial functions and partial satisfaction but instead impose the uniqueness of values on {\em total functions} similar to the way nonmonotonic causal theories~\cite{giu04} are defined.  This led to a simpler semantics, but as we show later in this paper, the semantics is not a proper generalization of the first-order stable model semantics from~\cite{ferraris11stable}, and moreover, it exhibits some unintuitive behavior.
%
%

We present an alternative approach to incorporating intensional functions into the stable model semantics by a simple modification to the first-order stable model semantics from~\cite{ferraris11stable}. It turns out that unlike the semantics from~\cite{lifschitz12logic}, this formalism, which we call ``Functional Stable Model Semantics (FSM),'' is a proper generalization of the language from~\cite{ferraris11stable}, and avoids the unintuitive cases that the language from~\cite{lifschitz12logic} encounters.  Furthermore, unlike the one from~\cite{cabalar11functional}, it does not require the extended notion of partial interpretations that deviates from the notion of classical interpretations. Nevertheless, the semantics from~\cite{cabalar11functional} can be embedded into FSM by simulating partial interpretations by total interpretations with auxiliary constants \cite{bartholomew13onthestable}. 
%
%


Unlike the semantics from~\cite{cabalar11functional}, as FSM properly extends the notion of functions in classical logic, its restriction to background theories provides a straightforward, seamless integration of ASP and Satisfiability Modulo Theories (SMT), which we call ``Answer Set Programming Modulo Theories (ASPMT),''  
analogous to the known relationship between first-order logic and SMT. SMT is a generalization of SAT and, at the same time, a special case of first-order logic in which certain predicate and function symbols in background theories have fixed interpretations. Such background theories include difference logic, linear arithmetic, arrays, and non-linear real-valued functions.  
\begin{wrapfigure}{l}{0.5\textwidth}
{
\begin{center}
\begin{tabular}{|c|c|}
\hline
   Monotonic & Nonmonotonic \\ \hline
   FOL       & FSM \\ 
   SMT       & ASP Modulo Theories \\
   SAT       & Traditional ASP \\ \hline
\end{tabular}
\end{center}
\caption{{\small Analogy between SMT and ASPMT}}
}
\label{fig:analogy}
\end{wrapfigure}
Likewise, ASPMT can be viewed as a generalization of the traditional ASP and, at the same time, a special case of FSM in which certain background theories are assumed as in SMT. On the other hand, unlike SMT, ASPMT is not only motivated by computational efficiency, but also by expressive knowledge representation. This is due to the fact that ASPMT is a natural extension of both ASP and SMT. 
Using SMT solving techniques involving functions, ASPMT can be applied to domains containing real numbers and alleviates the grounding problem.
It turns out that constraint ASP can be viewed as a special case of ASPMT in which functions are limited to non-intensional ones.

The paper is organized as follows. Section~\ref{sec:review-sm} reviews the stable model semantics from~\cite{ferraris11stable}, which
Section~\ref{sec:stable} extends to allow intensional functions. 
Section~\ref{sec:properties} shows that many known properties of the stable model semantics are naturally established for this extension. 
Section~\ref{sec:elim-p} shows how to eliminate intensional predicates
in favor of intensional functions, and Section~\ref{sec:elim-f} shows
the opposite elimination under a specific condition. 
Section~\ref{sec:comparison-if} compares FSM to other approaches to defining intensional functions.
Section~\ref{sec:many-sorted} extends FSM to be many-sorted, and, based on it, 
Section~\ref{sec:aspmt} defines the concept of ASPMT as a special case of many-sorted FSM, and presents its reduction to SMT under certain conditions. 
%
Section~\ref{sec:comparison-casp} compares ASPMT to other approaches to combining ASP with CSP and SMT.

This article is an extended version of the conference papers \cite{bartholomew12stable,bartholomew13functional}.\footnote{Besides the complete proofs, this article contains some new results, such as the non-existence of translation from non-${\bf c}$-plain formulas to ${\bf c}$-plain formulas, the usefulness of non-${\bf c}$-plain formulas, reducibility of many-sorted FSM to unsorted FSM, and more complete formal comparison with related works.}

\section{Review: First-Order Stable Model Semantics with Intensional Predicates} \label{sec:review-sm}

The proposed definition of a stable model in this paper is a direct generalization of the one from~\cite{ferraris11stable}, which we review in this section. Stable models are defined as classical models that satisfy a certain ``stability'' condition, which is expressed by ensuring a minimality condition on predicates. 

The syntax of formulas is defined the same as in the standard first-order logic. 
A signature consists of {\em function constants} and {\em predicate constants}. Function constants of arity $0$ are called {\em object constants}, and predicate constants of arity $0$ are called {\em propositional
constants}. A {\em term} of a signature $\sigma$ is formed from object constants of $\sigma$ and object variables using function constants of $\sigma$. An {\em atom} of $\sigma$ is an $n$-ary predicate constant followed by a list of $n$ terms; {\em atomic formulas} of $\sigma$ are atoms of $\sigma$, equalities between terms of $\sigma$, and the $0$-place connective $\bot$ (falsity). First-order formulas of $\sigma$ are built from atomic formulas of $\sigma$ using the primitive propositional connectives 
$ 
   \bot,\ \land,\ \lor,\ \rar,
$
as well as quantifiers $\forall,\ \exists$.
We understand $\neg F$ as an abbreviation of $F\rar\bot$; symbol
$\top$ stands for $\bot\rar\bot$, and $F\lrar G$ stands for $(F\rar
G)\land(G\rar F)$, and $t_1\ne t_2$ stands for $\neg (t_1=t_2)$.

In~\cite{ferraris11stable}, stable models are defined in terms of
the $\sm$ operator, whose definition is similar to the $\circ$
operator used for defining circumscription \cite{mcc80,lif93e}. 
As in circumscription, for predicate symbols (constants or variables) $u$ and $p$, expression $u\le p$ is defined
as shorthand for
\hbox{$\forall {\bf x}(u({\bf x})\rar p({\bf x}))$}; 
expression $u=p$ is defined as
\hbox{$\forall {\bf x}(u({\bf x})\lrar p({\bf x}))$}. 
For lists of predicate symbols ${\bf u}=(u_1,\dots,u_n)$ and
${\bf p}=(p_1,\dots,p_n)$, expression ${\bf u}\le {\bf p}$ is defined
as \hbox{$(u_1\le p_1)\land\dots\land (u_n\le p_n)$}, 
expression ${\bf u} = {\bf p}$ is defined as 
\hbox{$(u_1 = p_1)\land\dots\land (u_n = p_n)$}, and 
expression ${\bf u}<{\bf p}$ is defined as 
\hbox{${\bf u}\le {\bf p}\land \neg({\bf u}={\bf p})$}.

For any first-order formula~$F$ and any finite list of predicate constants
${\bf p} = (p_1,\dots,p_n)$, formula $\sm[F; {\bf p}]$
is defined as
$$
   F\land\neg\exists \wh{\bf p} (\wh{\bf p}<{\bf p}\land F^*(\wh{\bf p})),
$$
where $\wh{\bf p}$ is a list of distinct predicate variables
$\wh{p}_1,\dots,\wh{p}_n$, and $F^*(\wh{\bf p})$ is defined recursively as follows:
\begin{itemize}
\item  When $F$ is an atomic formula,  $F^*(\wh{\bf p})$ is a formula 
  obtained from $F$ by replacing all predicate constants ${\bf p}$ in
  it with the corresponding predicate variables from $\wh{\bf p}$;
\item  $(G\land H)^*(\wh{\bf p}) = G^*(\wh{\bf p})\land H^*(\wh{\bf p})$;
\item  $(G\lor H)^*(\wh{\bf p}) = G^*(\wh{\bf p})\lor H^*(\wh{\bf p})$;
\item  $(G\rar H)^*(\wh{\bf p}) = (G^*(\wh{\bf p})\rar H^*(\wh{\bf p}))
                \land (G\rar H)$;
\item  $(\forall x G)^*(\wh{\bf p}) = \forall x G^*(\wh{\bf p})$;
\item  $(\exists x G)^*(\wh{\bf p}) = \exists x G^*(\wh{\bf p})$.
\end{itemize}
%
The predicate constants in ${\bf p}$ are called {\em intensional}: these are
the predicates that we ``intend to characterize'' by $F$.\footnote{%
Intensional predicates are analogous to output predicates in Datalog,
and non-intensional predicates are analogous to input predicates in
Datalog \cite{lifschitz11datalog}.
}
When $F$ is a sentence (i.e., formula without free variables), the models of the second-order sentence
$\sm[F; {\bf p}]$ are called the {\em stable} models
of~$F$ relative to ${\bf p}$: they are the models of $F$ that are ``stable'' on
${\bf p}$. 

{\em Answer sets} are defined as a special class of first-order stable models as follows. By $\sigma(F)$ we denote the signature consisting of the function and predicate constants occurring in~$F$. 
If $F$ contains at least one object constant, an Herbrand interpretation of~$\sigma(F)$ that satisfies $\sm[F; {\bf p}]$ is called an {\em answer set} of~$F$, where ${\bf p}$ is the list of all predicate constants in $\sigma(F)$.
The answer sets of a logic program~$\Pi$ are defined as the answer sets
of the FOL-representation of~$\Pi$, which is obtained from $\Pi$ by 
\begin{itemize}
\item  replacing every comma by conjunction and every $\no$ by $\neg$ \footnote{Strong negation can be incorporated by introducing ``negative" predicates as in \cite[Section~8]{ferraris11stable}, or can be represented by a Boolean function with the value $\false$ \cite{bartholomew13afunctional}. For example, $\sim p$ can be represented by $p\mvis\false$.}
\item  turning every rule $\i{Head}\ar\i{Body}$ into a formula rewriting it as the implication $\i{Body}\rar\i{Head}$, and 
\item forming the conjunction of the universal closures of these formulas.
\end{itemize} 

For example, the FOL-representation of the program
$$
\ba l 
p(a)\\
q(b)\\
r(x)\ar p(x),\no\ q(x)
\ea
$$
is
\beq
p(a)\land q(b)\land\forall x((p(x)\land\neg q(x))\rar r(x))
\eeq{ex3f}
and $\sm[F;\ p,q,r]$ is
$$
\ba l
p(a)\land q(b)\land\forall x((p(x)\land\neg q(x))\rar r(x))\\
\quad\land\neg\exists uvw\Big(\big((u,v,w)<(p,q,r)\big)\wedge u(a)\land v(b) \\
\hspace{5em} \land
\forall x\Big(\big((u(x)\land
(\neg v(x)\land\neg q(x))) \rar w(x)\big)
\land \big((p(x)\land\neg q(x))\rar r(x)\big)\Big)\Big),
\ea
$$
which is equivalent to the first-order sentence
\beq
\ba l
\forall x(p(x) \lrar x=a) \land \forall x(q(x) \lrar x=b)
\land \forall x (r(x) \lrar (p(x) \land \neg q(x)))
\ea
\eeq{ex3f-comp} 
\cite[Example~3]{fer07a}.
The stable models of $F$ are any first-order models of (\ref{ex3f-comp}). The only answer set of $F$ is the Herbrand model $\{p(a),\ q(b),\ r(a)\}$.

\begin{remark}
According to \cite{ferraris11stable}, this definition of an answer set, when applied to the syntax of logic programs, is equivalent to the traditional definition of an answer set that is based on grounding and fixpoints as in~\cite{gel88}. 

It is also noted in~\cite{ferraris11stable} that if we replace $F^*(\wh{\bf p})$ with a simpler expression $F(\wh{\bf p})$ (which substitutes $\wh{\bf p}$ for ${\bf p}$), then the definition of $\sm[F; {\bf p}]$ reduces to the definition of $\circ[F; {\bf p}]$.
\end{remark} 

The definition of a stable model above is not limited to Herbrand models, so it allows general functions as in classical first-order logic. Indeed, in Section~\ref{sec:comparison-casp},   we show that the previous approaches to combining answer set programs and constraint processing can be viewed as special cases of first-order formulas under the stable model semantics.  However, these functions are ``extensional,'' and cannot cover examples like \eqref{water-clingcon}.



\section{Extending First-Order Stable Model Semantics to Allow
  Intensional Functions} \label{sec:stable} 

In this section, 
we generalize the first-order stable model semantics to allow intensional functions in addition to intensional predicates. 



\subsection{Second-Order Logic Characterization of the Stable Model Semantics} \label{ssec:sol-sm}


We extend expression $u=c$ as
$\forall {\bf x}(u({\bf x})=c({\bf x}))$
if $u$ and $c$ are function symbols.
For lists of predicate and function symbols ${\bf u}=(u_1,\dots,u_n)$ and
${\bf c}=(c_1,\dots,c_n)$, 
expression ${\bf u} = {\bf c}$ is defined as \hbox{$(u_1 = c_1)\land\dots\land (u_n = c_n)$}.

Let $\bC$ be a list of distinct predicate and function constants, and let $\vbC$ be a list of distinct predicate and function variables corresponding to~$\bC$. 
By $\bC^\mi{pred}$ ($\bC^\mi{func}$, respectively) we mean the list of all predicate constants (function constants, respectively) in~$\bC$, and by $\vbC^\mi{pred}$ ($\vbC^\mi{func}$, respectively) the list of the corresponding predicate variables (function variables, respectively) in $\vbC$.
For any formula $F$, expression $\sm[F;\ \bC]$ is defined as
\beq
   F\land\neg\exists \vbC(\vbC<\bC\land F^*(\vbC)),
\eeq{fsm}
where $\vbC<\bC$ is shorthand for 
$(\vbC^\mi{pred}\le \bC^\mi{pred})\land\neg (\vbC=\bC)$, 
and $F^*(\vbC)$ is defined recursively in the same way as $F^*(\wh{\bf p})$ except for the base case, which is defined as follows.
\begin{itemize}
\item When $F$ is an atomic formula, $F^*(\vbC)$ is $F'\land F$ where $F'$ is obtained from $F$ by replacing all (predicate and function) constants ${\bf c}$ in it with the corresponding variables from $\vbC$.
\end{itemize}

As before, we say that an interpretation $I$ that satisfies $\sm[F; {\bf c}]$ a {\em stable model} of $F$ relative to ${\bf c}$. Clearly, every stable model of $F$ is a model of $F$ but not vice versa. 

\begin{remark}
It is easy to see that the definition of a stable model above is a proper generalization of the one from~\cite{ferraris11stable}, also reviewed in the previous section: the definition of $\sm[F; {\bf c}]$ in this section reduces to the one in the previous section when all intensional constants in ${\bf c}$ are predicate constants only. 

When all intensional constants are function constants only, the definition of~$\sm[F; {\bf c}]$ is similar to the first-order nonmonotonic causal theories defined in~\cite{lifschitz97onthelogic}. The only difference is that, instead of $F^*(\wh{c})$, a different expression is used there. A more detailed comparison is given in Section~\ref{ssec:relation-nmct}.
\end{remark}

We will often write $F\rar G$ as $G\ar F$ and identify a finite set of formulas with the conjunction of the universal closures of each formula in that set.



For any formula $F$, expression $\{F\}^{\rm ch}$ denotes the ``choice'' formula $(F\lor\neg F)$.

The following two lemmas are often useful in simplifying $F^*(\vbC)$, as we demonstrate in Example~\ref{ex:amount} below. They are natural extensions of Lemmas~5 and 6 from~\cite{ferraris11stable}.

\begin{lemma}\label{lem:monotone}
Formula
$$
(\vbC < \bC) \land F^*(\vbC) \rar F
$$
is logically valid.
\end{lemma}

\proof By induction on the structure of~$F$.
\qed

\begin{lemma}\label{lem:neg}
Formula
$$\vbC <  \bC\rar((\neg F)^*(\vbC)\lrar\neg F)$$
is logically valid.
\end{lemma}
\proof Immediate from Lemma~\ref{lem:monotone}.
\qed

\begin{example}\label{ex:amount}
The following program $F_1$ describes the level of an unlimited water tank that is filled up unless it is flushed. 
\beq
\ba {rcl}
  \{\i{Amt}_1\mvis x\!+\!1\}^{\rm ch} & \ar &  \i{Amt}_0\mvis x, \\
  \i{Amt}_1\mvis 0 & \ar & \i{Flush}\ .
\ea
\eeq{eq:amount}
Here $\i{Amt}_1$ is an intensional function constant, and $x$ is a variable ranging over nonnegative integers. Intuitively, the first rule asserts that the amount increases by one by default.\footnote{Section~\ref{ssec:choice} explains why choice formulas are read as specifying default values.}  However, if $\i{Flush}$ action is executed (e.g., if we add the fact $\i{Flush}$ to \eqref{eq:amount}), this behavior is overridden, and the amount is set to $0$. 

Using Lemmas \ref{lem:monotone} and \ref{lem:neg}, under the assumption $\wh{\i{Amt}_1}<\i{Amt}_1$, one can check that formula $F_1^*(\wh{\i{Amt}_1})$ is equivalent to the conjunction consisting of \eqref{eq:amount} and 
\beq
\ba {rcl}
   (\wh{\i{Amt}_1}=x\!+\!1\land\i{Amt}_1=x\!+\!1) \lor\neg(\i{Amt}_1=x\!+\!1) 
   &\ \ar\ & \i{Amt}_0\!=\!x,  \\
   \wh{\i{Amt}_1}=0\land \i{Amt}_1=0 &\ \ar\ & \i{Flush},
\ea
\eeq{eq:f1star}
so that 
\begin{align*}
 \sm[F_1 ; \i{Amt}_1] & = F_1\land\neg\exists\wh{\i{Amt}_1} (\wh{\i{Amt}_1}\ne\i{Amt}_1\land F_1^*(\wh{\i{Amt}_1})) \\
      & \Lrar F_1\land\neg\exists\wh{\i{Amt}_1} (\wh{\i{Amt}_1}\ne\i{Amt}_1\land \\ 
     & \hspace{3cm} \forall x (\i{Amt}_0 \!=\! x\rar \neg(\i{Amt}_1=x\!+\!1))\land      
          (\i{Flush}\rar\bot)).
\end{align*}


Consider the first-order interpretations that have the set of nonnegative integers as the universe, interprets integers, arithmetic functions, and comparison operators in the standard way, and maps the other constants in the following way.

\begin{center}
\begin{tabular}{|c||c|c|c|c|c|c|c}\hline
           & $\i{Amt}_0$ & $\i{Flush}$ & $\i{Amt}_1$ \\ \hline\hline
 $I_1$ &  $5$  &  $\false$ & $6$  \\ \hline
 $I_2$ &  $5$  &  $\false$ & $8$  \\ \hline
 $I_3$ &  $5$  &  $\true$  & $0$   \\ \hline
\end{tabular}
\end{center}

\bi
\ii Interpretation $I_1$ is in accordance with the intuitive reading of the rules above, and it is indeed a model of $\sm[F_1; \i{Amt}_1]$.

\ii 
   Interpretation $I_2$ is not intuitive (the amount suddenly jumps up with no reason). 
   It is not a model of $\sm[F_1; \i{Amt}_1]$ though it is a model of $F_1$.

\ii
Interpretation $I_3$ is in accordance with the intuitive reading of the rules above. It is a model of~$\sm[F_1;\i{Amt}_1]$. 
\ei
\end{example}

\subsection{Reduct-Based Characterization of the Stable Model Semantics}\label{ssec:reduct-sm}

The second-order logic based definition of a stable model in the previous section is succinct, and is a natural extension of the first-order stable model semantics that is defined in~\cite{ferraris11stable}, but it may look distant from the usual definition of a stable model in the literature that is given in terms of grounding and fixpoints. 

In~\cite{bartholomew13onthestable}, an equivalent definition of the functional stable model semantics in terms of {\em infinitary ground formulas}  and {\em reduct} is given. 
Appendix~A of this article contains a review of the definition.

\BOCCC
In this section, we present an equivalent definition of a stable model in that way. 

Since the definition of a stable model we consider applies to any first-order models, not necessarily Herbrand ones, we refer to the concept of grounding {\sl relative to an interpretation} introduced in~\cite{truszczynski12connecting}, according to which, grounding a quantified sentence may introduce infinite conjunctions and disjunctions over the elements in the universe.


\subsubsection{Infinitary Ground Formulas}\label{sssec:igf}

We define {\em infinitary ground formulas}, which are slightly adapted
from infinitary propositional formulas
from~\cite{truszczynski12connecting}. Unlike
in~\cite{truszczynski12connecting}, 
we do not replace ground terms with their corresponding
object names, keeping them intact during grounding. This difference
is required in defining a reduct for the functional stable model
semantics.\footnote{Another difference is that grounding
  in~\cite{truszczynski12connecting} refers to ``infinitary
  propositional formulas,'' which can be defined on any propositional
  signature. This generality is not essential for our purpose in this paper.}

More specifically, we assume that a signature and an interpretation are defined the same as in the standard first-order logic. 
For each element $\xi$ in the universe $|I|$ of $I$, we introduce a
new symbol $\xi^\dia$, called an {\sl object name}. By $\sigma^I$ we
denote the signature obtained from~$\sigma$ by adding all object names
$\xi^\dia$ as additional object constants. We will identify an
interpretation $I$ of signature $\sigma$ with its extension
to~$\sigma^I$ defined by $I(\xi^\dia)=\xi$.\footnote{%
For details, see \cite{lif08b}.}

We assume the primary connectives of infinitary ground formulas to be $\bot$, $\{\}^\land$, $\{\}^\lor$, and $\rar$. The usual propositional connectives $\land,\lor$ are considered as shorthands: $F\land G$ as $\{F,G\}^\land$, and $F\lor G$ as~$\{F,G\}^\lor$.

Let $A$ be the set of all ground atomic formulas of signature $\sigma^I$. The
sets ${\cal F}_0, {\cal F}_1, \dots$ are defined recursively as follows:
\begin{itemize}
\item ${\cal F}_0=A\cup\{\bot\}$;
\item ${\cal F}_{i+1} (i\ge 0)$ consists of expressions
  ${\cal H}^\lor$ and ${\cal H}^\land$, for all subsets
  ${\cal H}$ of ${\cal F}_0\cup\ldots\cup{\cal F}_i$, and of
  the expressions $F\rar G$, where
  $F$ and $G$ belong to ${\cal F}_0\cup\dots\cup{\cal F}_i$.
\end{itemize}
We define ${\cal L}_A^{inf}=\bigcup_{i=0}^{\infty}{\cal F}_i$, and call
elements of ${\cal L}_A^{inf}$ {\em infinitary ground formulas}
of~$\sigma$ w.r.t.~$I$.

For any interpretation $I$ of $\sigma$ and any infinitary ground
formula $F$ w.r.t.~$I$, the definition of satisfaction, $I\models F$, is as
follows:
\begin{itemize}
\item For atomic formulas, the definition of satisfaction 
  is the same as in the standard first-order logic;

\item $I\models{\cal H}^\lor$ if there is a formula $G\in {\cal H}$ 
   such that $I\models G$;

\item $I\models{\cal H}^\land$ if, for every formula $G\in {\cal H}$, 
   $I \models G$;

\item $I \models G\rar H$ if $I \not\models G$ or $I \models H$.
\end{itemize}

Given an interpretation, we identify any first-order sentence with an infinitary ground formula via the process of grounding relative to that interpretation. 
Let $F$ be any first-order sentence of a signature $\sigma$, and let
$I$ be an interpretation of~$\sigma$. By $gr_I[F]$ we denote the
infinitary ground formula w.r.t.~$I$ that is obtained from $F$ by the
following process:
\begin{itemize}
\item  If $F$ is an atomic formula, $gr_I[F]$ is $F$; 
\item  $gr_I[G\odot H]= gr_I[G]\odot gr_I[H]\ \ \ \
 (\odot\in\{\land,\lor,\rar\})$;

\item  $gr_I[\exists x G(x)] = 
      \{gr_I[G(\xi^\diamond)] \mid \xi\in |I|\}^\lor$; $\qquad$
\item  $gr_I[\forall x G(x)] = 
      \{gr_I[G(\xi^\diamond)] \mid \xi\in |I|\}^\land$.
\end{itemize}

\medskip
{\bf Example~\ref{ex:amount} Continued}\ \ \
{\sl 
For the formula $F_1$ and the interpretation $I$ considered in Example~\ref{ex:amount}, $gr_I[F_1]$ is the conjunction:
\[
\ba {lrclr}
\{ \ \ \ 
  & \i{Amt}_1\mvis 0\!+\!1 \lor \neg(\i{Amt}_1\mvis 0\!+\!1) &\ \ar\  & \i{Amt}_0\mvis 0, \\
  & \i{Amt}_1\mvis 1\!+\!1 \lor \neg(\i{Amt}_1\mvis 1\!+\!1) &\ \ar\ & \i{Amt}_0\mvis 1,  \\
  &   \cdots, \\
  &  \i{Amt}_1\mvis 0 &\ \ar\ & \i{Flush} &  \ \ \ \}^{\land}.
\ea
\]
}

\subsubsection{Stable Models in terms of Grounding and Reduct}\label{sssec:bl-reduct}

For any two interpretations $I$, $J$ of the same signature and any list $\bC$ of distinct predicate and function constants, we write $J<^\bC I$ if 
\begin{itemize}
\item  $J$ and $I$ have the same universe and agree on all constants
  not in $\bC$,

\item  $p^J\subseteq p^I$ for all predicate constants $p$ in $\bC$,\footnote{For any symbol $c$ in a signature, $c^I$ denotes the evaluation of $I$ on $c$.}
and 

\item  $J$ and $I$ do not agree on $\bC$. 
\end{itemize}
%




The {\em reduct} $F^\mu{I}$ of an infinitary ground formula $F$ relative to an interpretation $I$ is defined as follows: 
\begin{itemize}
\item For any atomic formula $F$, $F^\mu{I}= \left\{
  \ba {ll} 
     \bot & \text{ if $I\not\models F$ } \\
     F & \text{ otherwise. }
  \ea
  \right. 
$

\item $({\cal H}^\land)^\mu{I}= \{G^\mu{I} \mid G\in{\cal H}\}^\land$

\item $({\cal H}^\lor)^\mu{I}=  \{G^\mu{I} \mid G\in{\cal H}\}^\lor$

\item 
$
   (G\rar H)^\mu{I}= \left\{
   \ba {ll}
      \bot & \text{ if $I\not\models G\rar H$ } \\
       G^\mu{I} \rar H^\mu{I} & \text{ otherwise. }
   \ea
   \right.
$


\end{itemize}








The following theorem presents an alternative definition of a stable model that is equivalent to the one in the previous section. 

\begin{thm}\label{thm:fsm-reduct}\optional{thm:fsm-reduct}
Let $F$ be a sentence and let $\bC$ be a list of intensional constants. An interpretation $I$ satisfies $\sm[F; \bC]$ iff 
\begin{itemize}
\item  $I$ satisfies $F$, and 
\item  no interpretation $J$ such that $J<^\bC I$ satisfies $(gr_I[F])^\mu{I}$. 
\end{itemize}  
\end{thm}

\medskip
{\bf Example~\ref{ex:amount} Continued}\ \ \
{\sl 
The reduct $(gr_{I_1}[F_1])^\mu{I_1}$ is equivalent to
\beq
 \i{Amt}_1\mvis 5\!+\!1\ \lor\ \bot\ \ar\ \i{Amt}_0\mvis 5.
\eeq{ex1-reduct}

\bi
\ii
There is no interpretation that is different from $I_1$ only on $\i{Amt}_1$ and satisfies the reduct. In accordance with Theorem~\ref{thm:fsm-reduct}, interpretation $I_1$ is a stable model of $F_1$ relative to $\i{Amt}_1$. 

\ii 
On the other hand, consider the interpretation $I_2$. The reduct $(gr_{I_2}[F_1])^\mu{I_2}$ is equivalent to
\[
\ba c
 \bot\lor\neg\bot\ \ar\ \i{Amt}_0\mvis 5,  \\
  \i{Amt}_1\mvis 7\!+\!1\ \lor\ \bot\ \ar\ \bot 
\ea 
\]
and there are other interpretations that are different from $I$ only on
$\i{Amt}_1$ and still satisfy the reduct. In accordance with Theorem~\ref{thm:fsm-reduct}, interpretation $I_2$ is not a stable model of $F_1$ relative to $\i{Amt}_1$. 

\ii Now consider the interpretation $I_3$. The reduct $(gr_{I_3}[F_1])^\mu{I_3}$ is equivalent to
\[
  \i{Amt}_1\mvis 0\ \ar\ \i{Flush} 
\]
and there is no interpretation that is different from $I_3$ only on $\i{Amt}_1$ and satisfies the reduct. In accordance with Theorem~\ref{thm:fsm-reduct}, interpretation $I_3$ is a stable model of $F_1$ relative to $\i{Amt}_1$. 
\ei
}
\EOCCC

\section{Properties of Functional Stable Models} \label{sec:properties}

Many properties known for the stable model semantics can be naturally extended to the functional stable model semantics, which is a desirable feature of the proposed formalism. 

\subsection{Constraints} \label{ssec:constraints}

Following \citeauthor{ferraris09symmetric} [\citeyear{ferraris09symmetric}], we say that an occurrence of a constant or any other subexpression in a formula $F$ is {\em positive} if the number of implications containing that occurrence in the antecedent is even, and {\em negative} otherwise. 
We say that the occurrence is {\em strictly positive} if the number of implications in~$F$ containing that occurrence in the antecedent is~$0$. For example, in $\neg (f=1)\rar g=1$, the occurrences of $f$ and $g$ are both positive, but only the occurrence of $g$ is strictly positive.\footnote{%
Recall that we understand $\neg F$ as shorthand for $F\rar\bot$.}

About a formula~$F$ we say that it is {\em negative} on a list~$\bC$ of predicate and function constants if $F$ has no strictly positive occurrence of a constant from $\bC$.
Since any formula of the form $\neg H$ is shorthand for $H\rar\bot$, such a formula is negative on any list of constants. The formulas of the form $\neg H$  are called {\em constraints} in the literature of ASP:
adding a constraint to a program affects the set of its stable models in a particularly simple way by eliminating the stable models that ``violate" the constraint.\footnote{Note that the term ``constraint'' here is different from the one used in CSP.}



The following theorem is a generalization of Theorem~3 from~\cite{ferraris11stable} for the functional stable model semantics. 

\begin{thm}\label{thm:constraint}\optional{thm:constraint}
For any first-order formulas~$F$ and~$G$, if $G$ is negative on $\bC$, then 
\hbox{$\fsm[F\land G; \bC]$} is equivalent to~$\fsm[F; \bC]\land G$.
\end{thm}

\begin{example}\label{ex:3}
Consider $\fsm[F_2\land\neg(f\mvis 1); fg]$ where $F_2$ is 
\hbox{$(f\mvis 1 \lor g\mvis 1) \land (f\mvis 2 \lor g\mvis 2)$}. Since $\neg (f\mvis 1)$ is negative on $\{f,g\}$,
according to Theorem~\ref{thm:constraint}, 
\hbox{$\fsm[F_2\land\neg(f\mvis 1);f g]$} is equivalent to $\fsm[F_2; f g]\land \neg(f\mvis 1)$, which is
equivalent to $f\mvis 2\land g\mvis 1$.
\end{example}

\subsection{Choice and Defaults}\label{ssec:choice}

Similar to Theorem~2 from~\cite{ferraris11stable}, Theorem~\ref{thm:bi-un-sm} below shows that making the set of intensional constants smaller can only make the result of applying $\sm$ weaker, and that this can be compensated by adding choice formulas.
For any predicate constant~$p$, by $\i{Choice}(p)$ we denote the formula 
$\forall{\bf x} \{p({\bf x})\}^{\rm ch}$ (recall that $\{F\}^{\rm ch}$ is shorthand for $F\lor\neg F$), 
where ${\bf x}$ is a list of distinct object variables.  
For any function constant~$f$, by $\i{Choice}(f)$ we denote the formula 
\hbox{$\forall {\bf x}y \{f({\bf x})=y\}^{\rm ch}$}, where $y$ is an object variable that is distinct from ${\bf x}$.
For any finite list of predicate and function constants $\bC$, the expression $\i{Choice}(\bC)$ stands for the conjunction of the formulas $\i{Choice}(c)$ for all members~$c$ of~$\bC$.
We sometimes identify a list with the corresponding set when there is no
confusion.

The following theorem is a generalization of Theorem~7 from~\cite{ferraris11stable} for the functional stable model semantics. 

\begin{thm}\label{thm:bi-un-sm}\optional{thm:bi-un-sm}
For any first-order formula $F$ and any disjoint lists $\bC$, $\bD$ of distinct constants, the following formulas are logically valid: 
\[
\ba c
   \fsm[F; \bC\bD]\rar\fsm[F; \bC], \\
   \fsm[F\land\i{Choice}(\bD); \bC\bD]\lrar\fsm[F;\bC].
\ea
\]
\end{thm}

For example, 
\[ 
  \fsm[(g\mvis 1\rar f\mvis 1)\land \forall y (g\mvis y\lor\neg (g\mvis
y));\ fg]
\] 
is equivalent to 
\[ 
   \fsm[g\mvis 1\rar f\mvis 1;\ f].
\]



A formula $\{f({\bf t}) = {\bf t}'\}^{\rm ch}$, where $f$ is an intensional function constant and ${\bf t}$, ${\bf t}'$ contain no intensional function constants, intuitively represents that {\em $f({\bf t})$ takes the value ${\bf t}'$ by default}.
For example, the stable models of $\{g\mvis 1\}^{\rm ch}$ relative to $g$ map $g$ to $1$.
On the other hand, the default behavior is overridden when we conjoin
the formula with $g\mvis 2$: the stable models of 
\[
  \{g\mvis 1\}^{\rm ch}\land g\mvis 2
\]
relative to $g$ map $g$ to $2$, and no longer to $1$. 

The treatment of $\{g=1\}^{\rm ch}$ as $(g=1)\lor\neg (g=1)$ is similar to the choice rule $\{p\}^{\rm ch}$ in ASP for propositional constant $p$, which stands for $p\lor \neg p$, with an exception that $g$ has to satisfy a functional requirement, i.e., it is mapped to a unique value. Under that requirement, an interpretation that maps $g$ to $1$ is a stable model but another assignment to $g$ is not a stable model because the choice rule itself does not force one to believe that $g$ is mapped to that other value. This makes the choice rule for the function work as assigning a default value to the function.


With this understanding, the commonsense law of inertia can be succinctly represented using choice formulas for functions. For instance, the formula 
\beq
   \i{Loc}(b,t)\mvis l\ \rar\  \{\i{Loc}(b,t\!+\!1)\mvis l\}^{\rm ch}, 
\eeq{inertia}
where $\i{Loc}$ is an intensional function constant, represents that the location of a block~$b$ at next step retains its value by default. The default behavior can be overridden if some action moves the block. In contrast, the standard ASP representation of the commonsense law of inertia, such as \eqref{water-asp}, uses both default negation and strong negation, and requires the user to be aware of the subtle difference between them.

\subsection{Strong Equivalence}  \label{ssec:se}


Strong equivalence~\cite{lif01} is an important notion that allows us to replace a subformula with another subformula without affecting the stable models. The theorem on strong equivalence can be extended to formulas with intensional functions as follows. 

For first-order formulas~$F$ and~$G$, we say that~$F$ is {\sl strongly equivalent} to~$G$ if, for any formula~$H$, any occurrence of~$F$ in~$H$, and any list $\bC$ of distinct predicate and function constants, $\fsm[H;\bC]$ is equivalent to $\fsm[H';\bC]$, where~$H'$ is obtained from~$H$ by replacing the occurrence of~$F$ by~$G$.

The following theorem tells us that strong equivalence can be  characterized in terms of equivalence in classical logic.

\begin{thm}\label{thm:strong}\optional{thm:strong}
Let $F$ and $G$ be first-order formulas, let $\bC$ be the
list of all predicate and function constants occurring in $F$ or $G$, and let $\wh{\bC}$ be a
list of distinct predicate and function variables corresponding
to~$\bC$. The following conditions are equivalent to each other.
\begin{itemize}
\item  $F$ and $G$ are strongly equivalent to each other;
\item  Formula
\beq
  (F\lrar G)\land(\wh{\bC} < \bC\rar (F^*(\wh{\bC})\lrar G^*(\wh{\bC})))
\eeq{eq:strong}
is logically valid.
\end{itemize}
\end{thm}

For instance, choice formula $\{F\}^{\rm ch}$ is strongly equivalent to \hbox{$\neg\neg F\rar F$}. This can be shown, in accordance with Theorem~\ref{thm:strong}, by checking that not only they are classically equivalent but also
\[
  (F\lor\neg F)^*({\wh{\bC}})
\]
and 
\[
  (\neg\neg F\rar F)^*(\wh{\bC})
\] 
are classically equivalent under $\wh{\bC}<\bC$. Indeed, in view of Lemma~\ref{lem:neg}, $(F\lor\neg F)^*({\wh{\bC}})$ is equivalent to $(F^*({\wh{\bC}})\lor \neg F)$ and  $(\neg\neg F\rar F)^*(\wh{\bC})$ is equivalent to $F\rar F^*(\wh{\bC})$.
This fact allows us to rewrite formula $(\ref{inertia})$ as an implication in which the consequent is an atomic formula: 
\[
   \i{Loc}(b,t)\mvis l\land\neg\neg(\i{Loc}(b,t+1)\mvis l)
   \ \rar\  \i{Loc}(b,t\!+\!1)\mvis l.  
\]

For another example, $(G\rar F)\land (H\rar F)$ is strongly equivalent to $(G\lor H)\rar F$. This is useful for rewriting a theory into ``Clark normal form,'' to which we can apply completion as presented in the next section.

\subsection{Completion} \label{sec:completion}

Completion \cite{cla78} is a process that turns formulas under the stable model semantics to formulas under the standard first-order logic.

We say that a formula $F$ is in {\em Clark normal form} (relative to a list ${\bf c}$ of intensional constants) if it is a conjunction of sentences of the form 
\beq
  \forall {\bf x} (G\rar p({\bf x}))
\eeq{cnf-p}
and 
\beq
  \forall {\bf x}y (G \rar f({\bf x}) \mvis y)
\eeq{cnf-f}
one for each intensional predicate constant $p$ in ${\bf c}$ and each intensional function constant~$f$ in ${\bf c}$, where ${\bf x}$ is a list of distinct object variables, $y$ is another object variable, and $G$ is a formula that has no free variables other than those in ${\bf x}$ and $y$.

The {\em completion} of a formula $F$ in Clark normal form relative to ${\bf c}$, denoted by $\comp[F; {\bf c}]$, is obtained from $F$ by replacing each conjunctive term~(\ref{cnf-p}) with
\beq
  \forall {\bf x} (p({\bf x})\lrar G)
\eeq{comp-p}
and each conjunctive term (\ref{cnf-f}) with 
\beq
  \forall {\bf x}y (f({\bf x})\mvis y\lrar G).
\eeq{comp-f}

The {\em dependency graph} of $F$ (relative to ${\bf c}$), denoted by $\dg_\bC[F]$, is the directed graph that 
\begin{itemize}
\item  has all members of ${\bf c}$ as its vertices, and
\item  has an edge from $c$ to $d$ if, for some strictly positive occurrence of $G\rar H$ in~$F$,
  \begin{itemize}
  \item  $c$ has a strictly positive occurrence in~$H$, and
  \item  $d$ has a strictly positive occurrence in $G$. 
  \end{itemize}
\end{itemize}

We say that $F$ is {\em tight} (on {\bf c}) if the dependency graph of $F$ (relative to {\bf c}) is acyclic. The following theorem, which generalizes Theorem~11 from~\cite{ferraris11stable} for the functional stable model semantics, tells us that, for a tight formula, completion is a process that allows us to reclassify intensional constants as non-intensional ones. It is similar to the main theorem of~\cite{lifschitz13functional}, which describes functional completion in the context of nonmonotonic causal logic.

\begin{thm}\label{thm:completion}\optional{thm:completion}
For any formula $F$ in Clark normal form relative to~${\bf c}$ that is tight on~${\bf c}$, an interpretation $I$ that satisfies $\exists xy(x\ne y)$ is a model of $\fsm[F;{\bf c}]$ iff $I$ is a model of $\comp[F; {\bf c}]$.
\end{thm}

\medskip\noindent
{\bf Example~\ref{ex:amount} Continued}\ \ {\sl 
Formula $F_1$  is not in Clark normal Form relative to $\i{Amt}_1$, but it is strongly equivalent to 
\[
\ba {rcl}
  \i{Amt}_1\mvis y  &\ \ar\ &  y\mvis x\!+\!1 \land \i{Amt}_0\mvis x \land \neg\neg (\i{Amt}_1=y), \\
  \i{Amt}_1\mvis y &\  \ar\  & y\mvis 0\land \i{Flush}\ .
\ea
\]
and further to
\[
\ba {rcl}
    \i{Amt}_1\mvis y &\ \ar\ & \big( y\mvis x\!+\!1\land \i{Amt}_0\mvis x\land \neg\neg (\i{Amt}_1\mvis y) \big) \lor  \big(y\mvis 0\land \i{Flush}\big), 
\ea
\]
which is in Clark normal form relative to $\i{Amt}_1$ and is tight on $\i{Amt}_1$. In accordance with Theorem~\ref{thm:completion}, the stable models of $F_1$ relative to $\i{Amt}_1$ coincide with the classical models of 
\[
\ba {rcl}
    \i{Amt}_1\mvis y &\ \lrar\ & \big( y\mvis x\!+\!1\land \i{Amt}_0\mvis x\land \neg\neg (\i{Amt}_1\mvis y) \big) \lor  \big(y\mvis 0\land \i{Flush}\big).
\ea
\]
}

The assumption $\exists xy(x\ne y)$ in the statement of Theorem~\ref{thm:completion} is essential to avoid the mismatch between ``trivial" stable models and models of completion when the universe is a singleton.
Recall that in order to dispute the stability of a model $I$ in the presence of intensional function constants, one needs another interpretation that is different from $I$ on intensional function constants. If the universe contains only one element, the stability of a model is trivial. 
For example, take $F$ to be $\top$ and ${\bf c}$ to be an intensional function constant~$f$. If the universe $|I|$ of an interpretation $I$ is a singleton, then $I$ satisfies $\fsm[F]$ because there is only one way to interpret $\bC$, but $I$ does not satisfy the completion formula $\forall {\bf x}y (f({\bf x})=y\lrar \bot)$.

\section{Eliminating Intensional Predicates in Favor of Intensional Functions}\label{sec:elim-p}

In first-order logic, it is known that predicate constants can be replaced by function constants and vice versa. This section and the next section show similar transformations under the functional stable model semantics.  

\subsection{Eliminating Intensional Predicates} \label{ssec:elim-p}

Intensional predicate constants can be eliminated in favor of intensional function constants as follows. 

Given a formula $F$ and an intensional predicate constant $p$, formula $F^p_f$
is obtained from $F$ as follows: 
\begin{itemize}
\item  in the signature of $F$, replace $p$ with a new intensional function constant $f$ of arity~$n$, where $n$ is the arity of~$p$, 
  and add two new non-intensional object constants $0$ and $1$ (rename if necessary);


\item  replace each subformula $p({\bf t})$ in $F$ with $f({\bf t})=1$. 
\end{itemize}

By $\i{FC}_f$ (``Functional Constraint on $f$'') we denote the
conjunction of the following formulas, which enforces $f$ to be two-valued: 
\beq
   0\ne 1, 
\eeq{f1}
\beq
   \neg\neg\forall {\bf x}(f({\bf x})=0 \lor f({\bf x})=1), 
 \eeq{f2}
where ${\bf x}$ is a list of distinct object variables.
By $\i{DF}_f$ (``Default False on $f$'') we denote the formula
\beq
  \forall {\bf x} \{f({\bf x})=0\}^{\rm ch}.
\eeq{f3}


\begin{example}\label{ex:monkey}
Let $F$ be the conjunction of the universal closures of the following formulas: 
\[
\ba c
  \i{Loc}(b,t)\mvis l\rar \{\i{Loc}(b,t+1)\mvis l\}^{\rm ch}, \\
  \i{Move}(b,l,t)\rar \i{Loc}(b,t+1)=l 
\ea
\]
(lower case symbols are variables).
We eliminate the intensional predicate constant $\i{Move}$ in favor of an intensional function constant $\i{Move}_f$ to obtain 
$F_{\textit{Move}_f}^{\textit{Move}}\land\i{FC}_{\textit{Move}_f}\land\i{DF}_{\textit{Move}_f}$, which is the conjunction of the universal closures of the following formulas:  
\[
\ba c
  \i{Loc}(b,t)\mvis l\rar \{\i{Loc}(b,t\!+\!1)\mvis l\}^{\rm ch}, \\
  \i{Move}_f(b,l,t)=1\rar \i{Loc}(b,t\!+\!1)=l, \\
  0\ne 1, \\
  \neg\neg(\i{Move}_f(b,l,t) = 0 \lor \i{Move}_f(b,l,t) = 1), \\
  \{\i{Move}_f(b,l,t) = 0\}^{\rm ch}. 
\ea
\]
\end{example}

The following theorem asserts the correctness of the elimination method. 

\begin{thm}\label{thm:elim-p}\optional{thm:elim-p}
The set of formulas 
$$
  \{\forall {\bf x} (f({\bf x})=1\lrar p({\bf x})), \ \ \i{FC}_f\}
$$
entails $$\sm[F; p{\bf c}]\lrar\sm[F^p_f\land\i{DF}_f; f{\bf c}].$$
\end{thm}


The following corollary to Theorem~\ref{thm:elim-p} tells us that there is a 1--1 correspondence between the stable models of $F$ and the stable models of its ``functional image'' $F^p_f\land\i{DF}_f\land\i{FC}_f$. 
For any interpretation $I$ of the signature of $F$, by $I^p_f$ we denote the interpretation of the signature of $F^p_f$ obtained from $I$ by replacing the set $p^I$ with the function $f^{I^p_f}$ such that, for all $\xi_1,\dots,\xi_n$ in the universe of $I$, 
\[
\ba l
  f^{I^p_f}(\xi_1,\dots,\xi_n)= 1^I \text{   if } p^I(\xi_1,\dots,\xi_n)= \true \\
  f^{I^p_f}(\xi_1,\dots,\xi_n)= 0^I \text{   otherwise }.
\ea
\] 
Furthermore, we assume that $I^p_f$ satisfies (\ref{f1}). Consequently, $I^p_f$ satisfies $\i{FC}_f$.

\begin{cor}\label{cor:elim-p2}\optional{cor:elim-p2}
Let $F$ be a first-order sentence.
\begin{itemize}
\item[(a)] An interpretation $I$ of the signature of $F$ is a model  of~$\sm[F; p{\bf c}]$ iff $I^p_f$ is a model of~$\sm[F^p_f\land\i{DF}_f\land\i{FC}_f; f{\bf c}]$.
\item[(b)] An interpretation $J$ of the signature of $F^p_f$ is a model of~$\sm[F^p_f\land\i{DF}_f\land\i{FC}_f ; f{\bf c}]$ iff $J=I^p_f$ for some model $I$ of~$\sm[F; p{\bf c}]$.
\end{itemize}
\end{cor}

In Corollary~\ref{cor:elim-p2}~(b), it is clear by the construction of $I^p_f$ that, for each $J$, there is exactly one $I$ that satisfies the statement.

Repeated applications of Corollary~\ref{cor:elim-p2} allow us to completely eliminate intensional predicate constants in favor of intensional function constants, thereby turning formulas under the stable model semantics from~\cite{ferraris11stable} into formulas under FSM whose intensional constants are function constants only. 

Note that $\neg\neg$ in \eqref{f2} cannot be dropped in general. 
The formula $\neg\neg F$ is not strongly equivalent to $F$. The former is a weaker assertion than the latter under the stable model semantics. 
Indeed, if it is dropped, in Corollary~\ref{cor:elim-p2}, when $F$ is $\top$, the empty set is the only model of $\sm[F; p]$ whereas $\sm[F^p_f\land\i{DF}_f\land\i{FC}_f; f]$ has two models where $f$ is mapped to $0$ or $1$.



\section{Eliminating Intensional Functions in favor of Intensional Predicates}\label{sec:elim-f}

We show how to eliminate intensional function constants in favor of intensional predicate constants. Unlike in the previous section, the result is established for ``$f$-plain'' formulas only. It turns out that there is no elimination method for arbitrary formulas that is both modular and signature-preserving.

\subsection{Eliminating Intensional Functions from $\bC$-Plain Formulas in favor of Intensional Predicates} \label{ssec:elim-f}


Let $f$ be a function constant. A first-order formula is called {\em
  $f$-plain} \cite{lifschitz11eliminating} if each atomic formula in it
\begin{itemize}
\item  does not contain $f$, or
\item  is of the form $f({\bf t}) = t_1$ where ${\bf t}$ is a tuple of
  terms not containing $f$, and $t_1$ is a term not containing $f$.
\end{itemize}
For example, $f\mvis 1$ is $f$-plain, but each of $p(f)$, $g(f)=1$,
and $1\mvis f$ is not $f$-plain.

For any list $\bC$ of predicate and function constants, we say that
$F$ is {\em $\bC$-plain} if $F$ is $f$-plain for each function
constant $f$ in~$\bC$.



Let $F$ be an $f$-plain formula, where $f$ is an intensional function
constant. Formula $F^f_p$ is obtained from $F$ as follows:
\begin{itemize}
\item  in the signature of $F$, replace $f$ with a new intensional
  predicate constant $p$ of arity $n+1$, where $n$ is the arity
  of~$f$;


\item  replace each subformula $f({\bf t}) = t_1$ in $F$ with $p({\bf t},t_1)$.
\end{itemize}

The following theorem asserts the correctness of the elimination.

\begin{thm}\label{thm:elim-f}\optional{thm:elim-f}
For any $f$-plain formula $F$, the set of formulas 
$$\{\forall {\bf x}y (p({\bf x},y)\lrar f({\bf x})=y), \ \ 
\exists xy (x\ne y)\}$$ entails 
$$\sm[F; f{\bf c}]\lrar\sm[F^f_p; p{\bf c}].$$
\end{thm}

The theorem tells us how to eliminate an intensional function constant $f$ from an $f$-plain formula in favor of an intensional predicate constant.
By $\i{UEC}_p$ we denote the following formulas that enforce the
``functional image'' on the predicate $p$, 
\beq
\ba c
\forall {\bf x}yz(p({\bf x},y) \land p({\bf x},z)\land y\ne z\rar\bot) ,\\
\neg\neg \forall {\bf x}\exists y\, p({\bf x},y),
\ea
\eeq{uec}
where ${\bf x}$ is an $n$-tuple of variables, and all variables in
${\bf x}$, $y$, and $z$ are pairwise distinct. Note that each formula is negative on any list of constants, so they work as constraints (Section~\ref{ssec:constraints}) to eliminate the stable models that violate them.


\begin{example}
Consider the same formula $F$ in Example~\ref{ex:monkey}.
We eliminate the function constant $\i{Loc}$ in favor of the intensional predicate constant $\i{Loc}_p$ to obtain $F_{\textit{Loc}_p}^{\textit{Loc}}\land\i{UEC}_{\textit{Loc}_p}$, which is the conjunction of the universal closures of the following formulas: 
\beq
\ba c
  \i{Loc}_p(b,t,l)\rar \{\i{Loc}_p(b,t\!+\!1,l)\}^{\rm ch}, \\
  \i{Move}(b,l,t)\rar \i{Loc}_p(b,t\!+\!1,l), \\
  \i{Loc}_p(b,t,l)\land \i{Loc}_p(b,t,l')\land l\ne l' \rar \bot, \\
  \neg\neg\forall b\, t\, \exists l(\i{Loc}_p(b,t,l)).
\ea
\eeq{ex:elim}

\end{example}

The following corollary shows that there is a simple 1--1
correspondence between the stable models of $F$ and the stable models
of $F^f_p\land\i{UEC}_p$. Recall that the signature of $F^f_p$ is
obtained from the
signature of $F$ by replacing $f$ with $p$. For any interpretation $I$
of the signature of $F$, by $I^f_p$ we denote the interpretation of the
signature of $F^f_p$ obtained from $I$ by replacing the function $f^I$
with the predicate $p^I$ that consists of the tuples
\[
   \langle\xi_1,\dots,\xi_n,f^I(\xi_1,\dots,\xi_n)\rangle
\]
for all $\xi_1,\dots,\xi_n$ from the universe of $I$.

\begin{cor}\label{cor:elim-f2}\optional{cor:elim-f2}
Let $F$ be an $f$-plain sentence. 
\begin{itemize}
\item[(a)] An interpretation $I$ of the signature of $F$ that satisfies
$\exists xy (x\ne y)$ is a model of $\sm[F;f{\bf c}]$ iff $I^f_p$ is a
model of $\sm[F^f_p\land\i{UEC}_p;\ p{\bf c}]$.
\item[(b)] An interpretation $J$ of the signature of $F^f_p$ that satisfies
$\exists xy (x\ne y)$ is a model of
$\sm[F^f_p\land\i{UEC}_p;\ p{\bf c}]$ iff $J = I^f_p$ for some model $I$ of 
$\sm[F;f{\bf c}]$.
\end{itemize}



\end{cor}

In Corollary~\ref{cor:elim-f2} (b), it is clear by the construction of $I^f_p$ that, for each $J$, there is exactly one $I$ that satisfies the statement.

Theorem~\ref{thm:elim-f} and Corollary~\ref{cor:elim-f2} are similar to Theorem 3 and Corollary 5 from~\cite{lifschitz11eliminating}, which are about eliminating ``explainable'' functions in nonmonotonic causal logic in favor of ``explainable'' predicates.


Similar to Theorem~\ref{thm:completion}, the condition $\exists xy (x\ne y)$ is necessary in Theorem~\ref{thm:elim-f} and Corollary~\ref{cor:elim-f2} because in order to dispute the stability of a model $I$ in the presence of intensional function constants, one needs another interpretation that is different from~$I$ on intensional function constants. Such an interpretation simply does not exist if the condition is missing, so $I$ becomes trivially stable. 
For example, 
consider the formula $\top$ with signature $\sigma = \{f\}$ and the universe $\{1\}$. There is only one interpretation, which maps $f$ to $1$. This is a stable model of~$\top$. 
On the other hand, 
the formula $\top\land\i{UEC}_p$, which is $\top\land\neg\neg \exists y\, p(y)$, has no stable models.

The method above eliminates only one intensional function constant at a time, but repeated applications can eliminate all intensional function constants from a given $\bC$-plain formula in favor of intensional predicate constants. In other words, it tells us that the stable model semantics for $\bC$-plain formulas can be reduced to the stable model semantics from~\cite{ferraris11stable} by adding uniqueness and existence of value constraints.
 
The elimination method described in Corollary~\ref{cor:elim-f2} has shown to be useful in a special class of FSM, known as {\sl multi-valued propositional formulas} \cite{giu04}.\footnote{We discuss the relationship in Section~\ref{ssec:relation-mvp}.}
In \cite{lee13action}, the method allows us to relate the two different translations of action language ${\cal BC}$ into multi-valued propositional formulas and into the usual ASP programs.  Also, it led to the design of {\sc mvsm},\footnote{\url{http://reasoning.eas.asu.edu/mvsm/}} which computes stable models of multi-valued propositional formulas using {\sc f2lp} and {\sc clingo}, and the design of {\sc cplus2asp} \cite{babb13cplus2asp},\footnote{\url{http://reasoning.eas.asu.edu/cplus2asp/}} which computes action languages using ASP solvers.







Interestingly, the elimination method results in a new way of formalizing the commonsense law of inertia using choice rules instead of using strong negation, e.g., \eqref{water-asp}. The formulas \eqref{ex:elim}  can be more succinctly represented in the language of ASP as follows. 
\[
\ba l
  \{\i{Loc}_p(b,t\!+\!1,l)\}^{\rm ch} \ar \i{Loc}_p(b,t,l)  \\
   \i{Loc}_p(b,t\!+\!1,l) \ar \i{Move}(b,l,t)  \\
   \ar \no\ 1\{\i{Loc}_p(b,t,l): \i{Location}(l) \}1, \i{Block}(b), \i{Time}(t) 
\ea
\]
where $\i{Location}$, $\i{Block}$, and $\i{Time}$ are domain predicates.
The first rule says that if the location of $b$ at time $t$ is $l$, then decide arbitrarily whether to assert $\i{Loc}_p(b,t\!+\!1,l)$  at time $t+1$.  In the absence of additional information about the location of $b$ at time $t+1$, asserting $\i{Loc}_p(b,t\!+\!1,l)$ will be the only option, as the third rule requires one of the location $l$ to be associated with the block $b$ at time $t+1$.   But if we are given conflicting information about the location at time $t+1$ due to the $\i{Move}$ action,  then not asserting  $\i{Loc}_p(b,t\!+\!1,l)$ will be the only option, and the second rule will tell us the new location of $b$ at time $t+1$.


\subsection{Non-$\bC$-plain formulas vs. $\bC$-plain formulas}\label{ssec:non-bc}

One may wonder if the method of eliminating intensional function constants in the previous section can be extended to non-$\bC$-plain formulas, possibly by first rewriting the formulas into $\bC$-plain formulas. In classical logic, this is easily done by ``unfolding'' nested functions by introducing existential quantifiers, but this is not the case under the stable model semantics because nested functions in general express weaker assertions than unfolded ones.


\begin{example} \label{ex:aplusb}
Consider $F$ to be $a+b=5$, where $a$ and $b$ are object constants. The formula $F$ is equivalent to $\exists xy (a\mvis x\ \land\ b\mvis y\ \land\ x+y\mvis 5)$ under classical logic, but this is not the case under FSM. The former has no stable models, and the latter has many stable models, including $I$ such that $a^I=1, b^I=4$.
\end{example}	

Gelfond and Kahl [\citeyear{gelfond14knowledge}] describe the intuitive meaning of stable models in terms of {\sl rationality principle}: ``believe nothing you are not forced to believe." In the example above, it is natural to understand that $a+b=5$ does not force one to believe $a=1$ and $b=4$. 

The weaker assertion expressed by function nesting is useful for specifying the range of a function using a domain predicate, or expressing the concept of synonymity between the two functions without forcing the functions to have specific values.

\begin{example} \label{ex:pf}
Consider $F$ to be $\i{Dom}(a)$ where $\i{Dom}$ is a predicate constant and $a$ is an object constant. The formula $F$ can be viewed as applying the sort predicate (i.e., domain predicate) $\i{Dom}$ to specify the value range of $a$, but it does not force one to believe that $a$ has a particular value.
In classical logic, $F$ is equivalent to \hbox{$\exists x (\i{Dom}(x)\land x=a)$}, but their stable models are different. The former has no stable models, and the latter has many stable models, including $I$ such that $\i{Dom}^I=\{1,2,3\}$ and $a^I = 1$. 


\end{example}

\begin{example} \label{ex:synonym}
A ``synonymity'' rule~\cite{lifschitz11eliminating} has the form  
\beq
  B\ \rar\ f_1({\bf t}_1) = f_2({\bf t}_2),
\eeq{syn}
where $f_1$, $f_2$ are intensional
function constants in $\bF$, and ${\bf t}_1$, ${\bf t}_2$ are tuples
of terms not containing members of $\bF$. 
This rule expresses that we believe $f_1({\bf t}_1)$ to be
``synonymous'' to $f_2({\bf t}_2)$ under condition $B$, but it does not force one to assign particular values to $f_1({\bf t_1})$ and $f_2({\bf t_2})$. 
As a special case, consider $f_1=f_2$ vs. $\exists x (f_1=x\land f_2=x)$. The latter forces one to assign some values to $f_1$ and $f_2$, and does not express the intended weaker assertion that they are synonymous.
\end{example}

To sum up, in Examples \ref{ex:aplusb}, \ref{ex:pf}, and \ref{ex:synonym}, the classically equivalent transformations do not preserve strong equivalence. They affect the beliefs, forcing one to believe more than what the original formulas assert.

On the other hand, there is some special class of formulas for which the process of ``unfolding'' preserves stable models.
We first define precisely the process. 

\begin{definition}
The process of {\sl unfolding $F$ w.r.t. a list $\bC$ of constants}, denoted by
$\i{UF}_\bC(F)$, is recursively defined as follows. 
%
%
\begin{itemize}
\item If $F$ is an atomic formula that is $\bC$-plain, $\i{UF}_\bC(F)$ is $F$; 

\item If $F$ is an atomic formula of the form $p(t_1,\dots, t_n)$ ($n\ge 0$) such that $t_{k_1},\dots, t_{k_j}$ are all the terms in $t_1,\dots, t_n$ 
   that contain some members of $\bC$, then 
   $\i{UF}_\bC(p(t_1,\dots, t_n))$ is 
\[
 \exists x_1\dots x_j \Big(p(t_1,\dots, t_n)''\land
   \bigwedge_{1\le i\le j} 
  \i{UF}_\bC(t_{k_i} = x_i)\Big),
\]
where $p(t_1,\dots, t_n)''$ is obtained from $p(t_1,\dots, t_n)$ by
replacing each $t_{k_i}$ with a new variable $x_i$.

\item If $F$ is an atomic formula of the form $f(t_1,\dots,t_n)=t_0$ ($n\ge 0$) such that 
   $t_{k_1},\dots, t_{k_j}$ are all the terms in $t_0,\dots, t_n$ 
   that contain some members of $\bC$, then   
   $\i{UF}_\bC(f(t_1,\dots,t_n)=t_0)$ is 
\[
  \exists x_1\dots x_j \Big((f(t_1,\dots,t_n) = t_0)'' \land 
   \bigwedge_{1\le i\le j} 
   \i{UF}_\bC(t_{k_i} = x_i)\Big),
\] 
where $(f(t_1,\dots, t_n)=t_0)''$ is obtained from $f(t_1,\dots,t_n)=t_0$ by 
replacing each $t_{k_i}$ with a new variable $x_i$.

\item $\i{UF}_\bC(F\odot G)$ is 
    $\i{UF}_\bC(F)\odot\i{UF}_\bC(G)$, 
    where \hbox{$\odot\in\{\land, \lor, \rar\}$}.

\item $\i{UF}_\bC(Qx F)$ is $Qx\ \i{UF}_\bC(F(x))$, where 
  $Q\in\{\forall,\exists\}$.
\end{itemize}
\end{definition} 


In Example~\ref{ex:pf}, $\i{UF}_{Dom}(F)$ is $\exists x (\i{Dom}(x)\land a=x)$, and in Example~\ref{ex:aplusb}, $\i{UF}_{(a,b)}(F)$ is $\exists xy (a=x\land b=y\land x+y=5)$. In Example~\ref{ex:synonym}, $\i{UF}_{(f_1,f_2)}(f_1=f_2)$ is $\exists x
(f_1=x\land f_2=x)$. We already observed that the process of unfolding does not preserve the stable models of the formulas.



%





Theorem~\ref{thm:head-cplain} below presents a special class of formulas, for which the process of unfolding does preserve stable models, or in other words, unfolding does not affect the beliefs.

\begin{definition}\label{def:head-c-plain}
We say that a formula is {\em head-$\bC$-plain} if every strictly positively occurrence of an atomic formula in it is $\bC$-plain.
\end{definition}

For instance, $f(g)\mvis 1\rar h\mvis 1$ is head-$(f,g,h)$-plain,
though it is not $(f,g,h)$-plain. 



\begin{thm}\label{thm:head-cplain}\optional{thm:head-cplain}
For any head-$\bC$-plain sentence $F$ that is tight on $\bC$ and any
interpretation $I$ satisfying $\exists xy(x \neq y)$, 
we have $I\models \sm[F;\bC]$ iff $I\models\sm[\i{UF}_\bC(F);\bC]$.
\end{thm}

One may wonder if there is any other translation that would work to unfold nested functions. However, it turns out that there is no modular, signature-preserving translation from arbitrary formulas to $\bC$-plain formulas while preserving stable models.
\BOC
By a {\em modular} translation $\i{Tr}$, we mean $\i{Tr}(F\land G)$ has the same stable models as $\i{Tr}(F)\land\i{Tr}(G)$ for any sentences $F$ and $G$. We say that $\i{Tr}(F)$ is a {\em signature-preserving} translation if it does not introduce new constants not in $F$.
\EOC

\begin{thm}\label{thm:nomodular}\optional{thm:nomodular}
For any set $\bC$ of constants, there is no strongly equivalent transformation that turns an arbitrary formula into a $\bC$-plain formula.
\end{thm}

The proof follows from the following lemma. 

\begin{lemma} \label{lem:nonmodular}
There is no ${f}$-plain formula that is strongly equivalent to 
$p(f)\land p(1)\land p(2)\land\neg p(3)$. 
\end{lemma}

Theorem~\ref{thm:nomodular} tells us that the set of arbitrary formulas is strictly more expressive than the set of $\bC$-plain formulas of the same signature. One application of this greater expressivity is in reducing many-sorted FSM to unsorted FSM in Section \ref{ssec:many-sorted-reduction} later.





\BOC

\cgre
In fact, the stable models defined in the previous section can be extended to any first-order formulas without any modification. Indeed, this is how it was defined in [KR 2012]. However, this is sometimes unintuitive. 

For another instance, for the formula 
\beq
  f(1) = 1 \land f(2) = 1 \land (f(g) = 1 \rightarrow g = 1),
\eeq{eq:cyclic}
and an interpretation $I$ such that the universe $|I|$ is $\{1,2\}$, and 
$1^I = 1$, $2^I=2$, $f^I(1) = 1$, $f^I(2) = 1$, $g^I=1$,  
one can check that $I$ is a stable model of \eqref{eq:cyclic} relative
to $(f,g)$ according to~\cite{bartholomew12stable}.
One may argue that this involves a cyclic dependency because the fact
that the stable model maps $g$ to $1$ is ``supported'' by the fact
that $g$ is mapped to $1$. 
On the other hand, in this paper we identify \eqref{eq:cyclic} with
$\i{UF}_{(f,g)}(\eqref{eq:cyclic})$,
\[ 
  f(1) = 1 \land f(2) = 1 \land (\exists x (f(x) = 1\land
  g=x)\rightarrow g = 1),
\]
which has no stable models. Indeed, this formula, when viewed as an
abbreviation of a formula under the General Theory of Stable
Models as in Theorem~\ref{thm:fsm-psm}, obviously has a cyclic
dependency involving $g=1$.

\cbla
\EOC



\section{Comparing FSM with Other Approaches to Intensional Functions}
\label{sec:comparison-if}
\subsection{Relation to Nonmonotonic Causal Logic}
\label{ssec:relation-nmct}


A {\sl (nonmonotonic) causal theory} is a finite list of rules of the
form
\[
   F\nec G
\]
where $F$ and $G$ are formulas as in first-order logic. We identify a rule with the
universal closure of the implication $G\rar F$. 
A {\sl causal model} of a causal theory $T$ is defined as the models
of the second-order sentence
\[
  \cm[T; \bF] = T\land\neg
                \exists\wh{\bF}(\wh{\bF}\ne\bF\land T^\dagger(\wh{\bF}))
\]
where $\bF$ is a list of {\em explainable} function
constants,
and $T^\dagger(\wh{\bF})$ denotes the conjunction of the formulas \footnote{%
$\widetilde{\forall} F$ represents the universal closure of $F$.}
\beq
\widetilde{\forall}(G\rar F(\wh{\bF})) 
\eeq{cr1}
for all rules $F\nec G$ of~$T$. 
By a {\sl definite} causal theory, we mean the causal theory whose
rules have the form either
\beq
 f({\bf t})=t_1 \nec B
\eeq{definite2}
or
\beq
 \bot\nec B,
\eeq{definite3}
where $f$ is an explainable function constant,
${\bf t}$ is a list of terms that does not contain explainable
function constants, and $t_1$ is a term that does not contain
explainable function constants. 
By $\i{Tr}(T)$ we denote the theory consisting of the
following formulas:
\[
  \widetilde{\forall} (\neg\neg B\rar f({\bf t})=t_1)
\]
for each rule (\ref{definite2}) in $T$, and
\[
  \widetilde{\forall} \neg B
\]
for each rule (\ref{definite3}) in $T$.
The causal models of such $T$ coincide with the stable models of
$\i{Tr}(T)$.

\begin{thm}\label{thm:cm-fsm}\optional{thm:cm-fsm}
For any definite causal theory $T$, 
$I\models\cm[T; \bF]$ iff $I\models\sm[\i{Tr}(T); \bF]$.
\end{thm}

For non-definite theories, they do not coincide as shown by the following example.

\begin{example}
Consider the following non-definite causal theory $T$:
\[ 
\ba l
 \neg (f = 1) \nec \top  \\
 \neg (f = 2) \nec \top 
\ea
\]

An interpretation $I$ where $|I| = \{1, 2, 3 \}$, and $f^I = 3$ is a causal model of $T$. However, the corresponding formula $\i{Tr}(T)$ is equivalent to
\[
  \neg (f = 1) \land\neg (f = 2), 
\]
which has no stable models.
\end{example}

The following example, a variant of Lin's suitcase example \cite{lin95}, demonstrates some unintuitive behavior of definite causal theories in representing indirect effects of actions, which is not present in the functional stable model semantics. 

\begin{example}
Consider the two switches which can be flipped but cannot be both up or down at the same time. If one of them is down and the other is up, the direct effect of flipping only one switch is changing the status of that switch, and the indirect effect is changing the status of the other switch.
Let $\i{Up}(s,t)$, where $s$ is switch $A$ or $B$, and $t$ is a time stamp $0$ or $1$, be object constants whose values are Boolean, let $\i{Flip}(s)$, where $s$ is switch $A$ or $B$, be function constants whose values are Boolean, and  let $x,y$ be variables ranging over Boolean values. The domain can be formalized in a causal theory as 
\[
\ba {rclr}
    \i{Up}(s,1)\mvis x & \nec & \i{Up}(s,0)\mvis y \land \i{Flip}(s)\mvis \true\ \ \  & (x\ne y)\\
    \i{Up}(s,1)\mvis x & \nec & \i{Up}(s',1)\mvis y   & (s\ne s',  x\ne y) \\ 
    \i{Up}(s,1)\mvis x & \nec & \i{Up}(s,1)\mvis x\land \i{Up}(s,0)\mvis x \\
    \i{Flip}(s)\mvis x & \nec&  \i{Flip}(s)\mvis x  \\
    \i{Up}(A,0)\mvis  \false & \nec & \top \\
    \i{Up}(B,0)\mvis  \true & \nec & \top
\ea
\]
There are five causal models as shown in the following table. 

\begin{center}
\begin{tabular}{|c||c|c|c|c|c|c|c}\hline
  & \i{Up}(A,0) & \i{Up}(B,0) & \i{Flip}(A) & \i{Flip}(B) & \i{Up}(A,1) & \i{Up}(B,1) \\ \hline\hline
 $I_1$ &  \false  & \true  & \false & \false & \false & \true \\ \hline
 $I_2$ &  \false  & \true  & \false & \true & \true  & \false \\ \hline
 $I_3$ &  \false  & \true  & \true & \false & \true  & \false \\ \hline
 $I_4$ &  \false  & \true  & \true & \true & \true  & \false \\    \hline
  $I_5$ &  \false  & \true  & \false & \false & \true & \false \\ \hline
\end{tabular}
\end{center}

$I_2$ and $I_3$ exhibit the indirect effect of the action $\i{Flip}$.
Only $I_5$ is not intuitive because the fluent $\i{Up}$ changes its value for no reason. 

In the functional stable model semantics, the domain can be represented as 
\[
\ba {rclr}
    \i{Up}(s,1)\mvis x &\ \ar\  & \i{Up}(s,0)\mvis y \land \i{Flip}(s)\mvis \true\ \ \  & (x\ne y)\\
    \i{Up}(s,1)\mvis x & \ar & \i{Up}(s',1)\mvis y   & (s\ne s',  x\ne y) \\ 
    \{\i{Up}(s,1)\mvis x\}^{\rm ch} & \ar & \i{Up}(s,0)\mvis x \\
    \{Flip(s)\mvis x\}^{\rm ch} & \ar&  \top  \\
    \i{Up}(A,0)\mvis  \false & \ar & \top \\
    \i{Up}(B,0)\mvis  \true & \ar & \top
\ea
\]
The program has four stable models $I_1,I_2,I_3,I_4$; The unintuitive causal model $I_5$ is not its stable model.

\end{example} 

\subsection{Relation to Cabalar Semantics} 
\label{ssec:relation-cbl}

As mentioned earlier, the stable model semantics by Cabalar [\citeyear{cabalar11functional}] is defined in terms of partial satisfaction, which deviates from classical satisfaction.   \citeauthor{bartholomew13onthestable} [\citeyear{bartholomew13onthestable}] show its relationship to FSM.
%
There, it is shown that when we consider stable models to be total interpretations only, both semantics coincide on $\bC$-plain formulas. Also, $F$ and $\i{UF}_{\bf c}(F)$ have the same stable models under the Cabalar semantics, so any complex formula under the Cabalar semantics can be reduced to a ${\bf c}$-plain formula by preserving stable models. Furthermore, partial stable models under the Cabalar semantics can be embedded into FSM by introducing an auxiliary object constant {\sc none} to denote that the function is undefined.
Consequently, the Cabalar semantics can be fully embedded into FSM by unfolding using an auxiliary constant. 
We refer the reader to \cite[Section~4]{bartholomew13onthestable} for the details.

On the other hand, Theorem~\ref{thm:nomodular} of this paper shows that the reverse direction is not possible because the class of ${\bf c}$-plain formulas is a restricted subset in the functional stable model semantics, which is not the case with the Cabalar semantics. In other words, non-${\bf c}$-plain formulas are weaker than ${\bf c}$-plain formulas under FSM whereas the Cabalar semantics does not distinguish them.
For instance, under the Cabalar semantics, the formula $a+b=5$ in Example~\ref{ex:aplusb} has many stable models $I$ as long as $a^I + b^I = 5$; in Example~\ref{ex:pf}, $\i{Dom}(a)$ has many stable models rather than simply restricting the value of $a$ to the extent of $\i{Dom}$; in Example~\ref{ex:synonym}, $f_1=f_2$ has stable models as long as the functions are assigned the same values instead of merely stating that the functions are synonymous.

We observe that the weaker assertions by non-${\bf c}$-plain formulas are often useful but they are not allowed in the Cabalar semantics. 
%
In particular, the use of ``sort predicates" as in Example~\ref{ex:pf} is important in specifying the range of an intensional function, rather than a particular value. 
\footnote{
In Section~\ref{ssec:many-sorted-reduction} below, we formally show how to reduce many-sorted FSM into unsorted FSM and notes that the axioms used there is not expressible in the Cabalar semantics. }
The synonymity rule like Example~\ref{ex:synonym} is useful for the design of modular action languages as described in \cite{lifschitz11eliminating}.

\subsection{Relation to IF-Programs} \label{ssec:if}

The functional stable model semantics presented here is inspired by IF-programs from~\cite{lifschitz12logic}, where intensional functions were defined without requiring the complex notion of partial functions and partial satisfaction but instead by imposing the uniqueness of values on {\em total functions}. It turns out that neither semantics is stronger than the other while they coincide on a certain syntactically restricted class of programs. However, the semantics of IF-programs exhibits an unintuitive behavior.

\subsubsection{Review of IF-Programs}  \label{sssec:if}

\NBB{No implications in the body and head}

We consider rules of the form
\beq
  H\ar B,
\eeq{if-rule}
where $H$ and $B$ are formulas that do not contain $\rar$. As before,
we identify a rule with the universal closure of the implication
$B\rar H$. An IF-program is a finite conjunction of those rules. 

An occurrence of a symbol in a formula is {\sl negated} if it
belongs to a subformula that begins with negation, and is {\sl 
non-negated} otherwise. Let $F$ be a formula, let ${\bf f}$ be a
list of distinct function constants, and let ${\bf \wh{f}}$ be a list
of distinct function variables similar to ${\bf f}$.
By $F^\diamond(\wh{\bF})$ we denote the formula obtained
from $F$ by replacing each non-negated occurrence of a member of 
${\bf f}$ with the corresponding function variable in $\wh{{\bf f}}$. 
By $\lif[F; \bF]$ we denote the second-order sentence
\[
    F\land\neg\exists\wh{\bF}
       (\wh{\bF}\ne\bF\land F^\diamond(\wh{\bF})).
\]
According to~\cite{lifschitz12logic}, the {\em $\bF$-stable models} of
an IF-program $\Pi$ are defined as the models of $\lif[F; \bF]$, where
$F$ is the FOL-representation of $\Pi$. 
\NBB{
[[ What is the relationship between $F^\diamond$ and $F^*$. ]]
}

\subsubsection{Comparison} \label{sssec:comparison}

The definition of the $\lif$ operator above looks close to our
definition of the $\fsm$ operator. However, they often behave quite differently. 

\begin{example}\label{ex:9} 
Let $F$ be the following program 
\[
\ba l
  d=2 \ar c=1, \\
  d=1 
\ea
\]
and let $I$ be an interpretation such that
$|I|=\{1,2\}$, $c^I=2$ and $d^I=1$. 
$I$ is a model of $\lif[F; cd]$, but not a model of $\fsm[F; cd]$. The former is not intuitive from the rationality principle because $c$ does not even appear in the head of a rule.
\end{example}

\begin{example}\label{ex:10}
Let $F$ be the following program 
\[
   (c=1\lor d=1)\land (c=2\lor d=2)
\]
and let $I_1$ and $I_2$ be interpretations such that 
$|I_1| = |I_2| = \{1,2,3\}$ and $I_1(c) = 1$, $I_1(d) = 2$, $I_2(c) =
2$, $I_2(d) = 1$. The interpretations $I_1$ and $I_2$ are models of $\fsm[F; cd]$. On the
other hand,  $\lif[F; cd]$ has no models.
\end{example}

\begin{example}\label{ex:11}
Let $F_1$ be $\neg (c\mvis 1)\ar\top$ and let $F_2$ be $\bot\ar c\mvis 1$. Under the functional stable model semantics, they are strongly equivalent to each other, and neither of them has a stable model. However, this is not the case with IF-programs. For instance,  let $I$ be an interpretation such that $|I|=\{1,2\}$ and $I(c)=2$. $I$ satisfies $\lif[F_2; c]$ but not $\lif[F_1; c]$.
\end{example}

While $\bot\ar F$ is a constraint in our formalism, in view of Theorem~\ref{thm:constraint}, the last example illustrates that $\bot\ar F$
is not considered a constraint in the semantics of IF-programs. This behavior deviates from the standard stable model semantics.  Unlike the functional stable model semantics, in general, it is not obvious how various mathematical results established for the first-order stable model semantics, such as the theorem on strong equivalence~\cite{lif01}, the theorem on completion~\cite{ferraris11stable}, and the splitting theorem \cite{ferraris09symmetric}, can be extended to the above formalisms on intensional functions.


The following theorem gives a specific form of formulas on which the two semantics agree. 

\begin{thm}\label{thm:if-fsm}\optional{thm:if-fsm}
Let $T$ be an IF-program whose rules have the form 
\beq
  f({\bf t})=t_1 \ar\neg\neg B
\eeq{if-fsm-r}
where $f$ is an intensional function constant, ${\bf t}$ and $t_1$ do
not contain intensional function constants, and $B$ is an arbitrary
formula. We identify $T$ with the corresponding first-order formula. 
Then we have 
$I\models\fsm[T; \bF]$ iff $I\models\lif[T; \bF]$.
\end{thm} 


\section{Many-Sorted FSM} \label{sec:many-sorted}

The following is the standard definition of many-sorted first-order logic.
%
A signature~$\sigma$ is comprised of a set of function and predicate constants and a set of sorts. To every function and predicate constant of arity $n$, we assign argument sorts $s_1,\dots,s_n$ and to every function constant of arity $n$, we assign also its value sort $s_{n+1}$. We assume that there are infinitely many variables for each sort. Atomic formulas are built similar to the standard unsorted logic with the restriction that in a term $f(t_1,\dots,t_n)$ (an atom $p(t_1,\dots,t_n)$, respectively), the sort of $t_i$ must be a subsort of the $i$-th argument of $f$ ($p$, respectively). In addition $t_1 = t_2$ is an atomic formula if the sorts and $t_1$ and $t_2$ have a common supersort. 

A many-sorted interpretation $I$ has a non-empty universe $|I|^s$ for each sort $s$. When $s_1$ is a subsort of $s_2$,  an interpretation must satisfy $|I|^{s_1} \subseteq |I|^{s_2}$. The notion of satisfaction is similar to the unsorted case with the restriction that an interpretation maps a term to an element in its associated sort.

\BOCC
Let $\sigma^{\cal T}$ be the (many-sorted) signature of the background theory~${\cal T}$. An interpretation of $\sigma^{\cal T}$ is called the {\em background interpretation} if it satisfies the background theory. For instance, in the theory of reals, we assume that $\sigma^{\cal T}$ contains the set $\mathcal{R}$ of symbols for all real numbers, the set of arithmetic functions over real numbers, and the set $\{<, >, \le, \ge\}$ of binary predicates over real numbers. A background interpretation interprets these symbols in the standard way.

Let $\sigma$ be a signature that contains~$\sigma^{\cal T}$. An interpretation of $\sigma$ is called a {\em ${\cal T}$-interpretation} if it agrees with the fixed background interpretation of $\sigma^{\cal T}$ on the symbols in $\sigma^{\cal T}$. 

A ${\cal T}$-interpretation is a {\em ${\cal T}$-model} of $F$ if it satisfies $F$. 
\EOCC

The definition of many-sorted FSM is a straightforward extension of unsorted FSM. For any list ${\bf c}$ of constants in $\sigma$, an interpretation $I$ is a {\em stable model} of $F$ relative to~${\bf c}$ if $I$ satisfies $\sm[F; {\bf c}]$, where $\sm[F; {\bf c}]$ is syntactically the same as in Section~\ref{sec:stable} but formulas are understood as in many-sorted logic.


\subsection{Reducing Many-sorted FSM to unsorted FSM}\label{ssec:many-sorted-reduction}

We can turn many-sorted FSM into unsorted FSM as follows. Given a many-sorted signature $\sigma$, we define the signature $\sigma^{ns}$ to contain every function and predicate constant from $\sigma$. In addition, for each sort $s \in \sigma$, we add a unary predicate $\sort{s}$ to $\sigma^{ns}$. 

Given a formula $F$ of many-sorted signature $\sigma$, we obtain the formula $F^{ns}$ of the unsorted signature $\sigma^{ns}$ as follows. 



We replace every formula $\exists x F(x)$, where $x$ is a variable of sort $s$, with the formula
$$\exists y(\sort{s}(y)\land F(y))$$
where $y$ is an unsorted variable and $\sort{s}$ is a predicate constant in $\sigma^{ns}$ corresponding to $s$ in $\sigma$. Similarly, we replace every $\forall x\ F(x)$, where $x$ is a  variable of sort $s$, with the formula
$$\forall y(\sort{s}(y) \rightarrow F(y)). $$

By $SF_\sigma$ we denote the conjunction of
\begin{itemize}
\item the formulas $\forall y({\tt s}_i(y) \rightarrow {\tt s}_j(y))$ for every two sorts $s_i$ and $s_j$ in $\sigma$ such that $s_i$ is a subsort of $s_j$ ($s_i\ne s_j$),
\item the formulas $\exists y\ \sort{ s}(y)$ for every sort $s$ in $\sigma$

\item the formulas 
$$
\forall y_1\dots y_{k} (\sort{ args}_1(y_1) \land \dots \land \sort{args}_k(y_k) \rightarrow \sort{vals}(f(y_1,\dots, y_{k})))
$$ 
for each function constant $f$ in $\sigma$, where the arity of $f$ is $k$, and the $i$-th argument sort of $f$ is $args_i$ and the value sort of $f$ is $vals$.

\item the formulas 
$$\forall y_1\dots y_{k+1} (\neg \sort{args}_1(y_1) \lor \dots \lor \neg \sort{args}_k(y_k) \rightarrow \{f(y_1,\dots, y_{k}) = y_{k+1}\}^{\rm ch})$$
for each function constant $f$ in $\sigma$, where the arity of $f$ is $k$ and the $i$-th argument sort of $f$ is $args_i$. 

\item the formulas 
$$
\forall y_1\dots y_{k} (\neg \sort{args}_1(y_1) \lor \dots \lor \neg \sort{args}_k(y_k) \rightarrow \{p(y_1,\dots, y_{k})\}^{\rm ch})
$$ 
for each predicate constant $p$ in $\sigma$, where the arity of $p$ is $k$, and the $i$-th argument sort of $p$ is $args_i$. 
\end{itemize}

Note that only the first three items are necessary for classical logic but we need to add the fourth and fifth items for the FSM semantics so that the witness $J$ to dispute the stability of $I$ can only disagree with $I$ on the atomic formulas that actually correspond to atomic formulas in the many-sorted setting (which has arguments adhering to the argument sorts). Also note that the formulas in item 3 are not ${\bf c}$-plain, which 
illustrates the usefulness of non-${\bf c}$-plain formulas.

We map an interpretation $I$ of a many-sorted signature $\sigma$ to an interpretation $I^{ns}$ of an unsorted signature $\sigma^{ns}$ as follows. 
First, the universe $|I^{ns}|$ of $\sigma^{ns}$ is $\bigcup \limits_{s \text{ is a sort in } \sigma} |I|^{s}$. We specify that the sort predicates and sorts correspond by defining the extent of sort predicate $\sort{ s}$ for every sort $s \in \sigma$ as 
$$\sort{ s}^{I^{ns}} = |I|^s.$$
For every function constant $f$ in $\sigma$ and every tuple ${\bfxi}$ comprised of elements from $|I^{ns}|$, we take	

\begin{displaymath}
   f^{I^{ns}}({\bfxi}) = \left\{
     \begin{array}{lr}
       f^I({\bfxi}) & \text{ if each } \xi_i \in |I|^{args_i} \text { where $args_i$ is the $i$-th argument sort of $f$}\\
       |I^{ns}|_0 & \text{ otherwise }
     \end{array}
   \right.
\end{displaymath} 
where $|I^{ns}|_0$ is an arbitrarily chosen element in the universe $|I^{ns}|$
(we use the same element for every situation this case holds).

For every predicate constant $p$ in $\sigma$ and every ${\bfxi}$, we take
\begin{displaymath}
   p^{I^{ns}}({\bfxi})= \left\{
     \begin{array}{lr}
       p^I({\bfxi}) & \text{ if each } \xi_i \in |I|^{args_i} \text { where $args_i$ is the $i$-th argument sort of $p$}\\
       \false & \text{otherwise. }
     \end{array}
   \right.
\end{displaymath} 
Note that $\false$ was arbitrarily chosen.

The choice of $I^{ns}$ mapping a function whose arguments are not of the intended sort to the value $|I^{ns}|_0$ is arbitrary and so there are many unsorted interpretations that correspond to the 
many-sorted interpretation. To characterize this one-to-many relationship, we say two 
unsorted interpretations $I$ and $J$ are {\em related} with relation $R$, denoted $R(I,J)$, 
if for every predicate or function constant $c$, we have 
$c^I(\xi_1,\dots,\xi_k) = c^J(\xi_1,\dots,\xi_k)$
whenever each $\xi_i \in args_i$ where $args_i$ is the $i$-th argument sort of $c$.

\begin{thm}\label{thm:ms2us}\optional{thm:ms2us}
Let $F$ be a formula of a many-sorted signature $\sigma$, and let ${\bf c}$ be a set of function and predicate constants.
\begin{itemize}
\item[(a)] If an interpretation $I$ of signature $\sigma$ is a model of $\fsm[F;{\bf c}]$, then $I^{ns}$ is a model of $\fsm[F^{ns} \land SF_\sigma;{\bf c}]$.
\item[(b)] If an interpretation $L$ of signature $\sigma^{ns}$ is a model of $\fsm[F^{ns} \land SF_\sigma;{\bf c}]$ then there is some interpretation $I$ of signature $\sigma$ such that $I$ is a model of $\fsm[F;{\bf c}]$ and $R(L,I^{ns})$.
\end{itemize}
\end{thm}


\begin{example}
Consider $\sigma = \{s_1, s_2,  f/1, 1, 2\}$ where both the argument and the value sort of function constant $f$ are $s_1$. Take $F$ to be
$f(1) = 1 \land f(2) = 2$. The many-sorted interpretation $I$ such that $|I|^{s_1} = \{1,2\}$, $|I|^{s_2} = \{3,4\}$,  $n^I = f^I(n) = n$ for $n\in\{1,2\}$ is clearly a stable model of $F$. However, if we drop the last two items of $SF_\sigma$, formula~$F^{ns}\land SF_\sigma$ is
$$
\ba l
f(1) = 1 \land f(2) = 2\ \land \\
\exists y\, \sort{s}_1(y) \land \exists y\, \sort{s_2}(y) \ \land\  \\
\forall y_1(\sort{s}_1(y_1) \rightarrow \sort{s}_1(f(y_1)))
\ea 
$$
and $K$ is an unsorted interpretation such that $|K| = \{1,2,3,4\}$, $(\sort{s}_1)^{K} = \{1,2\}$, $(\sort{s}_2)^{K} = \{3,4\}$, $n^K = n$ for $n\in\{1,2,3,4\}$,  $f^{K}(n) = n$ for $n\in\{1,2,3,4\}$, which  is not a stable model of $F^{ns}$ since we can take $J$ that is different from $K$ only on $f(4)$, i.e., $f^J(4) = 3$, to dispute the stability of $K$.
\end{example}


\subsection{Relation to Multi-Valued Propositional Formulas Under the Stable Model Semantics} 
\label{ssec:relation-mvp}

Multi-valued propositional formulas \cite{giu04} are an extension of the standard propositional formulas where  atomic parts of a formula are equalities of the kind found in constraint satisfaction problems.
Action languages such as ${\cal C}$+ \cite{giu04} and $\cal{BC}$ \cite{lee13action}  are defined based on multi-valued propositional formulas. In particular, the latter two languages are defined as shorthand for multi-valued propositional formulas under the stable model semantics, which is a special case of the functional stable model semantics as we show in this section.

A {\sl multi-valued propositional signature} is a set $\sigma$~of symbols called {\sl multi-valued propositional constants (mvp-constants)}, along with a nonempty finite set~$\i{Dom}(c)$ of symbols, disjoint from $\sigma$, assigned to each mvp-constant~$c$.  We call $\i{Dom}(c)$ the {\sl domain} of~$c$.
A {\sl multi-valued propositional atom (mvp-atom)} of a signature~$\sigma$ is an expression of the form ${c\mvis v}$ (``the value of~$c$ is~$v$'') where $c \in \sigma$ and $v
\in \i{Dom}(c)$. 
A {\sl multi-valued propositional formula (mvp-formula)} of~$\sigma$ is a propositional combination of mvp-atoms.

A {\sl multi-valued propositional interpretation (mvp-interpretation)} of~$\sigma$ is a function that maps every element of~$\sigma$ to an element of its domain.  An mvp-interpretation~$I$ {\sl satisfies} an mvp-atom ${c\mvis v}$ (symbolically, ${I\models c\mvis v}$) if $I(c)=v$.
The satisfaction relation is extended from mvp-atoms to arbitrary mvp-formulas according to the usual truth tables for the propositional connectives.

The reduct $F^I$ of an mvp-formula $F$ relative to an mvp-interpretation $I$ is the mvp-formula obtained from $F$ by replacing each maximal subformula that is not satisfied by $I$ with $\bot$. 
$I$ is called a {\em stable model} of $F$ if $I$ is the only mvp-interpretation satisfying $F^I$.


Multi-valued propositional formulas can be viewed as a special class of ground first-order formulas of many-sorted signatures. We identify a multi-valued propositional signature with a many-sorted signature that consists of mvp-constants and their values understood as object constants. Each mvp-constant $c$ is identified with an intensional object constant whose sort is $\i{Dom}(c)$. Each value in $\i{Dom}(c)$ is identified with a non-intensional object constant of the same sort $\i{Dom}(c)$, except that if the same value $v$  belongs to multiple domains, the sort of $v$ is the union of the domains.\footnote{This is because in many-sorted logic with ordered sorts, the equality is defined when both terms have the same common supersort.} 
For instance, if $\i{Dom}(c_1)=\{1,2\}$ and $\i{Dom}(c_2)=\{2,3\}$, then the sort of $2$ is $\i{Dom}(c_1)\cup\i{Dom}(c_2)$, while the sort of $1$ is $\i{Dom}(c_1)$ and the sort of $3$ is $\i{Dom}(c_2)$.
An mvp-atom $c\mvis v$ is identified with an equality between an intensional object constant $c$ and a non-intensional object constant $v$. 

We identify an mvp-interpretation with the many-sorted interpretation in which each non-intensional object constant is mapped to itself, and is identified with an element in $\i{Dom}(c)$ for some intensional object constant $c$.

It is easy to check that an mvp-interpretation $I$ is a stable model of $F$ in the sense of multi-valued propositional formulas iff $I$ is a stable model of $F$ in the sense of the functional stable model semantics. Under this view, every mvp-formula is identified with a $\bC$-plain formula, where $\bC$ is the set of all mvp-constants. The elimination of intensional functions in favor of intensional predicates in Section~\ref{ssec:elim-f} essentially turns mvp-formulas into the usual propositional formulas.

\section{Answer Set Programming Modulo Theories}\label{sec:aspmt}

Sections~\ref{sec:elim-p} and \ref{sec:elim-f} show that intensional predicate constants and intensional function constants are interchangeable in many cases.
On the other hand, this section shows that considering intensional functions has the computational advantage of making use of efficient computation methods available in the work on satisfiability modulo theories.


We define ASPMT as a special case of many-sorted FSM by restricting attention to interpretations that conform to the background theory.


\subsection{ASPMT as a Special Case of the Functional Stable Model
  Semantics}  \label{ssec:aspmt}


Formally, an SMT instance is a formula in many-sorted first-order
logic, where some designated function and predicate constants are
constrained by some fixed background interpretation. SMT is the
problem of determining whether such a formula has a model that expands
the background interpretation~\cite{barrett09satisfiability}.

Let $\sigma^{\cal T}$ be the many-sorted signature of the background theory~${\cal T}$. An interpretation of $\sigma^{\cal T}$ is called the {\em background interpretation} if it satisfies the background theory. For instance, in the theory of reals, we assume that $\sigma^{\cal T}$ contains the set $\mathcal{R}$ of symbols for all real numbers, the set of arithmetic functions over real numbers, and the set $\{<, >, \le, \ge\}$ of binary predicates over real numbers. A background interpretation interprets these symbols in the standard way.

Let $\sigma$ be a signature that contains~$\sigma^{\cal T}$. An interpretation of $\sigma$ is called a {\em ${\cal T}$-interpretation} if it agrees with the fixed background interpretation of $\sigma^{\cal T}$ on the symbols in $\sigma^{\cal T}$. 

A ${\cal T}$-interpretation is a {\em ${\cal T}$-model} of $F$ if it satisfies $F$. 

For any list ${\bf c}$ of constants in $\sigma\setminus \sigma^{\cal T}$, a ${\cal T}$-interpretation $I$ is a {\em ${\cal T}$-stable model} of $F$ relative to~${\bf c}$ if $I$ satisfies $\sm[F; {\bf c}]$.

%
\cbla

\subsection{Describing Actions in ASPMT} 


The following example demonstrates how ASPMT can be applied to solve an instance of planning problem with the continuous time that requires real number computation. The encoding extends the standard ASP representation for transition systems \cite{lif99b}.

\begin{example}\label{ex:car}
Consider the following running example from a Texas Action Group
discussion posted by Vladimir Lifschitz.\footnote{http://www.cs.utexas.edu/users/vl/tag/continuous\_problem}

\begin{quote}
A car is on a road of length $L$.  If the accelerator is activated, the
car will speed up with constant acceleration ${\rm A}$ until the accelerator is
released or the car reaches its maximum speed ${\rm MS}$, whichever comes first.
If the brake is activated, the car will slow down with acceleration\ 
$-{\rm A}$
until the brake is released or the car stops, whichever comes first.
Otherwise, the speed of the car remains constant.
Give a formal representation of this domain, and write a program that
uses your representation to generate a plan satisfying the following
conditions:  at duration 0, the car is at rest at one end of the road; at
duration $T$, it should be at rest at the other end.
\end{quote}

This example can be represented in ASPMT as follows.
Below $s$ ranges over time steps, $b$ is a Boolean variable,
$x,y,a,c,d$ are variables over nonnegative reals, and ${\rm A}$ and ${\rm MS}$ are some specific real numbers.

We represent that the actions $\i{Accel}$ and $\i{Decel}$ are
exogenous and the duration of each time step is to be arbitrarily
selected as
\[
\ba{l}
\{\i{Accel}(s) = b\}^{\rm ch}, \\ 
\{\i{Decel}(s) = b\}^{\rm ch}, \\ 
\{\i{Duration}(s) = x\}^{\rm ch}.
\ea
\]
Both $\i{Accel}$ and $\i{Decel}$ cannot be performed at the
same time: 
\[
\ba{l}

\bot \leftarrow \i{Accel}(s) = \true \land \i{Decel}(s) = \true.\\
\ea
\]
The effects of $\i{Accel}$ and $\i{Decel}$ on $\i{Speed}$ are described as
\[
\ba{rl}
  \i{Speed}(s+1) = y \ar& \i{Accel}(s)\mvis\true\ \land\ 
                          \i{Speed}(s)\mvis x\ \land\ 
                          \i{Duration}(s)\mvis d  \\
   & \land\ (y = x+{\rm A}\times d),  \\
  \i{Speed}(s+1) = y \ar& \i{Decel}(s)\mvis\true\ \land\ 
                          \i{Speed}(s)\mvis x\ \land\ 
                          \i{Duration}(s)\mvis d\  \\
   & \land\ (y = x-{\rm A}\times d). 
\ea
\]
The preconditions of $\i{Accel}$ and $\i{Decel}$ are described as
\[
\ba{rl}
\bot \ar & \i{Accel}(s)\mvis\true\ \land\ 
           \i{Speed}(s)\mvis x\ \land\
           \i{Duration}(s)\mvis d \\
         & \land\ (y = x+{\rm A}\times d)\ \land\ (y > {\rm MS}), \\
\bot \ar & \i{Decel}(s)\mvis\true\ \land\ 
           \i{Speed}(s)\mvis x\ \land\
           \i{Duration}(s)\mvis d \\
         & \land\ (y = x-{\rm A}\times d)\ \land\ (y < 0).
\ea
\]
$\i{Speed}$ is inertial:
\[
  \{\i{Speed}(s+1) = x\}^{\rm ch} \ar \i{Speed}(s) = x.
\]
$\i{Speed}$ at any moment does not exceed the maximum speed {\rm MS}:
\[
  \bot\ar \i{Speed}(s)>{\rm MS}. 
\]
$\i{Location}$ is defined in terms of $\i{Speed}$ and
$\i{Duration}$ as 
\[
\ba{rl}
\i{Location}(s+1) = y \ar & \i{Location}(s) = x \land 
     \i{Speed}(s)=a  \land 
     \i{Speed}(s+1) = c\  \\
     & \land\ \i{Duration}(s)\mvis d\ \land\ 
     y = x + ((a + c) / 2)\times d. 
\ea
\]
\end{example}

Theorem~\ref{thm:completion} tells us that a tight ASPMT theory in Clark normal form can be turned into an SMT instance.

{\bf Example~\ref{ex:car} Continued}\ \ {\sl 
Since the formalization above can be written in Clark Normal Form that is tight,  its stable models coincide with the models of the completion formulas. 
For instance, to form the completion of $\i{Speed}(1)$, consider the rules that have $\i{Speed}(1)$ in the head: 
{
\[
\ba {rl}
  \i{Speed}(1)\mvis y\ \ar\ & \i{Accel}(0)\mvis\true\ \land\ 
                          \i{Speed}(0)\mvis x\ \land\ 
                          \i{Duration}(0)\mvis d\  \\
    & \land\ (y = x+{\rm A}\times d) \land\ (y\leq {\rm MS}), \\
  \i{Speed}(1)\mvis y\ \ar\ & \i{Decel}(0)\mvis\true\ \land\ 
                          \i{Speed}(0)\mvis x\ \land\ 
                          \i{Duration}(0)\mvis d\  \\
   & \land\ (y = x-{\rm A}\times d)\land\ (y\ge 0),  \\
   \i{Speed}(1)\mvis y\ \ar\ & \i{Speed}(0)\mvis y\ \land\
   \neg\neg(\i{Speed}(1)\mvis y)
\ea 
\]
}
($\{c\mvis v\}^{\rm ch}\ar G$ is strongly equivalent to $c\mvis v\ar G\land
\neg\neg (c\mvis v)$).
The completion turns them into the following equivalence:
{
\beq
\ba{l}
\i{Speed}(1)  = y\  \lrar\ \\
~~\exists xd(~~ 
      (\i{Accel}(0)\mvis\true\ \land\ 
      \i{Speed}(0)\mvis x\ \land\ 
      \i{Duration}(0)\mvis d\   \\
\hspace{7.3cm}      \land\ (y = x+{\rm A}\times d)\land  (y\leq {\rm MS}))\ \\ 
~~~~~~~\lor (\i{Decel}(0)\mvis\true\ \land\ 
      \i{Speed}(0)\mvis x\ \land\ 
      \i{Duration}(0)\mvis d\   \\
\hspace{7.3cm}     \land\ (y = x-{\rm A}\times d)\land  (y\ge 0))   \\
~~~~~~~\lor \i{Speed}(0)=y ~~).
\ea
\eeq{ex:comp}
}
}


It is worth noting that most action descriptions can be represented by tight ASPMT theories due to the associated time stamps. In~\cite{lee13answer}, ASPMT was used as the basis of extending action language $\cal C$+ \cite{giu04} to represent the durative action model of PDDL 2.1 \cite{fox03pddl} and the start-process-stop model of representing continuous changes in PDDL+ \cite{fox06modelling}. In~\cite{lee17representing}, language ${\cal C}$+ was further extended to allow ordinary differential equations (ODE), the concept borrowed from SAT modulo ODE. As our action language is based on ASPMT, which in turn is founded on the basis of ASP and SMT, it enjoys the development in SMT solving techniques as well as the expressivity of ASP language.

\subsection{Implementations of ASPMT}

A few implementations of ASPMT emerged based on the idea that reduces tight ASPMT theories to the input language of SMT solvers. System {\sc aspmt2smt} \\~\cite{bartholomew14system} is a proof-of-concept implementation of ASPMT by reducing ASPMT programs into the input language of SMT solver {\sc z3}, and is shown to effectively handle real number computation for reasoning about continuous changes. The system allows a fragment of ASPMT in the input language, whose syntax resembles ASP rules and which can be effectively translated into the input language of SMT solvers. In particular, the language imposes a syntactic condition that quantified variables can be eliminated by equivalent rewriting. 

\citeauthor{walega15aspmt} [\citeyear{walega15aspmt}] extended the system {\sc aspmt2smt} to handle nonmonotonic spatial reasoning that uses both qualitative and quantitative information, where spatial relations are encoded in theory of nonlinear real arithmetic.

%
%
%


In~\cite{lee17representing}, based on the recent development in SMT called ``Satisfiability Modulo Ordinary Differential Equations (ODE)'' \cite{gao13satisfiability} and its implementation  {\sc dReal} \cite{gao13dreal}, the system {\sc cplus2aspmt} was built on top of {\sc aspmt2smt}. The paper showed that a general class of hybrid automata with non-linear flow conditions and non-convex invariants can be turned into first-order action language ${\cal C}$+, and {\sc cplus2aspmt} can be used to compute the action language modulo ODE by translating ${\cal C}$+ into ASPMT. For example, the effect of $\i{Accel}$ in Example~\ref{ex:car} can be represented using ODE as 
\begin{align*}
  \i{Speed}(s\!+\!1) = x+y\ar\ \ & 
       \i{Accel}(s)\mvis\true\ \land\ 
       \i{Speed}(s)\mvis x\ \land\ 
       \i{Duration}(s)\mvis \delta\ \land\  \\
       & y\mvis \int_0^\delta {\rm A}\ dt\ \land\ y\!\le\! {\rm MS}.
\end{align*}

The theory of reals is decidable as shown by Tarski, and some SMT solvers do not always approximate reals with floating point numbers. Even for undecidable theories, such as formulas with trigonometric functions and differential equations, SMT solving techniques ensure certain error-bounds: A $\delta$-complete decision procedure \cite{gao13satisfiability} for such an SMT formula $F$ returns false if $F$ is unsatisfiable, and returns true if its syntactic ``numerical perturbation'' of $F$ by bound $\delta$ is satisfiable, where $\delta>0$ is number provided by the user to bound on numerical errors. This is practically useful since it is not possible to sample exact values of physical parameters in reality. ASPMT is able to take the advantage of the SMT solving techniques whereas it is shown that the ASPMT description of action domains is much more compact than the SMT counterpart. 

In~\cite{asuncion15ordered}, the authors presented the ``ordered completion," that compiles logic programs with convex aggregates into the input language of SMT solvers. The focus there was to compute the standard ASP language using SMT solvers. So unlike the other systems mentioned above, neither intensional functions nor various background theories in SMT were considered there. On the other hand, the input programs are not restricted to tight programs. 

%
%
%
%
%


\section{Comparing ASPMT with Other Approaches to Combining ASP with CSP/SMT} \label{sec:comparison-casp}


We compare ASPMT with other approaches to combining ASP with CSP/SMT. These approaches can be related to a special case of ASPMT in which all functions are non-intensional.


\subsection{Relation to Clingcon Programs}


A {\sl constraint satisfaction problem} (CSP) is a tuple $(V,D,C)$, where $V$ is a set of {\em constraint variables} with their respective {\em domains} in $D$, and $C$ is a set of {\em constraints} that specify some legal assignments of values in the domains to the constraint variables.

A {\sl clingcon program $\Pi$} \cite{gebser09constraint} with a constraint satisfaction
problem $(V,D,C)$ is a set of rules of the form
\beq
   a\ar B, N, \i{Cn},
\eeq{clingcon-rule}
where $a$ is a propositional atom or $\bot$, $B$ is a set of
positive propositional literals, $N$ is a set of negative
propositional literals, and $\i{Cn}$ is a set of constraints from $C$,
possibly preceded by $\no$.

Clingcon programs can be viewed as ASPMT instances.
 Below is a reformulation of the semantics using the terminologies in ASPMT.
We assume that constraints are expressed by ASPMT sentences of signature $V\cup\sigma^{\cal T}$, where $V$ is a set of object constants, which is identified with the set of constraint variables $V$ in $(V,D,C)$, whose value sorts are identified with the domains in $D$; we assume that $\sigma^{\cal T}$ is disjoint from $V$ and contains all values in~$D$ as object constants, and other symbols to represent constraints, such as $+$, $\times$, and $\ge$. 
In other words, we represent a constraint as a formula $F(v_1,\dots,v_n)$ over $V\cup\sigma^{\cal T}$ where $F(x_1,\dots,x_n)$ is a formula of the signature $\sigma^{\cal T}$ and $F(v_1,\dots,v_n)$ is obtained from $F(x_1,\dots,x_n)$ by substituting the object constants $(v_1,\dots,v_n)$ in $V$ for $(x_1,\dots,x_n)$. We say this background theory ${\cal T}$ {\em conforms} to $(V, D,C)$.

For any signature $\sigma$ that consists of object constants and
propositional constants, we identify an interpretation $I$ of $\sigma$
as the tuple $\langle I^f,X\rangle$, where $I^f$ is the restriction of
$I$ onto the object constants in $\sigma$, and $X$ is a set of
propositional constants in $\sigma$ that are true under $I$.

Given a clingcon program $\Pi$ with $(V,D,C)$, and a ${\cal T}$-interpretation 
$I=\langle I^f,X\rangle$, we define the {\sl constraint reduct of
  $\Pi$ relative to~$X$ and $I^f$} (denoted by $\Pi^X_{I^f}$) as the
set of rules
$
   a \ar B
$ 
for each rule~\eqref{clingcon-rule} in $\Pi$ such that
 $I^f\models\i{Cn}$, and
$X\models N$.
We say that a set $X$ of propositional atoms is a {\sl constraint
  answer set} of $\Pi$ relative to $I^f$ if $X$ is a minimal model
of~$\Pi^X_{I^f}$.

\medskip\noindent{\bf Example~\ref{ex:amount} continued}\ \ 
{\sl 
The rules
\[
\ba l
  \i{Amt}_1=^\$ \i{Amt}_0\!+\!1 \ar\no\ \i{Flush}, \\
  \i{Amt}_1 =^\$ 0 \ar \i{Flush}
\ea
\]
are identified with 
\[
\ba {l}
  \bot \ar \no\ \i{Flush}, \no (\i{Amt}_1 =^\$ \i{Amt}_0\!+\!1) \\
  \bot \ar \i{Flush}, \no (\i{Amt}_1 =^\$ 0) 
\ea
\]
under the semantics of {clingcon} programs with the theory of integers as the background theory; $\i{Amt}_0$, $\i{Amt}_1$ are object constants and $\i{Flush}$ is a propositional constant. Consider $I_1$ in Example~\ref{ex:amount}, which can be represented as $\langle (I_1)^f, X\rangle$ where $(I_1)^f$ maps $\i{Amt}_0$ to $5$, and $\i{Amt}_1$ to $6$, and $X=\emptyset$.
The set $X$ is the constraint answer set relative to $(I_1)^f$ because $X$ is the minimal model of the constraint reduct relative to $X$ and $(I_1)^f$, which is the empty set. 

}\smallskip

Similar to the way that rules are identified as a special
case of formulas~\cite{ferraris11stable}, we identify a clingcon
program $\Pi$ with the conjunction of implications $B\land N\land
\i{Cn}\rar a$ for all rules~\eqref{clingcon-rule} in $\Pi$.
The following theorem tells us that clingcon programs are a special
case of ASPMT in which the background theory ${\cal T}$ conforms to 
$(V,D,C)$, and intensional constants are limited to
propositional constants only, and do not allow function constants, so the language cannot express the default assignment of values to a function.


\NBB{[[ change bg notation to ${\cal T}$-models ]] }

\begin{thm}\label{thm:clingcon}\optional{thm:clingcon}
Let $\Pi$ be a clingcon program with CSP $(V,D,C)$,
let ${\bf p}$ be the set of all propositional constants occurring in~$\Pi$,
let ${\cal T}$ be the background theory conforming to $(V,D,C)$,  and let $\langle I^f,X\rangle$ be a ${\cal T}$-interpretation. Set $X$ is a constraint answer set of~$\Pi$ relative
to~$I^f$ iff $\langle I^f, X\rangle$ is a ${\cal T}$-stable model of $\Pi$ relative to~${\bf p}$.
\end{thm}


Note that a clingcon program does not allow an atom that consists of elements from both $V$ and ${\bf p}$. Thus the truth value of an atom is determined by either $I^f$ or $X$, but not by involving both of them. 

In~\cite{lierler16constraint}, the authors compared Constraint ASP and SMT by relating the different terminologies and concepts used in each of them. This is related to the relationship shown in Theorem~\ref{thm:clingcon} since ${\cal T}$-stable models of an ASPMT program $\Pi$ relative to $\emptyset$ are precisely SMT models of $\Pi$ with background theory ${\cal T}$. One main difference between the two comparisons is that an {\em answer set} in \cite{lierler16constraint} is a set containing ordinary atoms and theory/constraint atoms, while a {\em stable model} in this paper is a classical model.
\BOCC
While this allows for loose coupling of an ASP solver and a constraint solver, it leads to some unintuitive cases as we noted with the example in the introduction. The main problem there has to do with the fact that the symbol $a$ in the {\sl propositional} atom $p(a)$ has nothing to do with the constraint variable $a$ in $a =^\$ b$. The anomaly can be avoided in ASPMT when we view the symbol $a$ in both places as an object constant in ASPMT.
\EOCC



\subsection{Relation to ASP(LC) Programs} \label{ssec:asplc}

\citeauthor{liu12answer} [\citeyear{liu12answer}] consider logic programs with linear
constraints, or {\em ASP(LC)} programs, comprised of rules of the form 
\beq
   a\ar B, N, LC
\eeq{nlplc}
where $a$ is a propositional atom or $\bot$, $B$ is a set of positive
propositional literals,  and $N$ is a set of negative propositional
literals, and $LC$ is a set of {\em theory atoms}---linear constraints
of the form 
$ 
  \displaystyle\sum\limits_{i=1}^n (c_i\times x_i)\ \bowtie\ k
$ 
where $\bowtie\in\{\leq, \geq, =\}$, each $x_i$ is an object
constant whose value sort is integers (or reals), and each $c_i$, $k$ is an
integer (or real).

An ASP(LC) program $\Pi$ can be viewed as an ASPMT formula whose background theory ${\cal T}$ is the theory of integers or the theory of reals. 
We identify an ASP(LC) program $\Pi$ with the conjunction of ASPMT formulas 
$B\land N\land LC\rar a$ for all rules~\eqref{nlplc} in $\Pi$. 

An {\em LJN-intepretation} is a pair $(X,T)$ where $X$ is a set of propositional atoms and $T$ is a subset of theory atoms occurring in~$\Pi$ such that there is some ${\cal T}$-interpretation~$I$ that satisfies $T\cup\overline{T}$, 
where $\overline{T}$ is the set of negations of each theory atom
occurring in $\Pi$ but not in $T$. 
An LJN-interpretation $(X,T)$ satisfies an atom $b$ if
$b\in X$, the negation of an atom $not\ c$ if $c \notin X$, and a
theory atom $t$ if $t \in T$. The notion of satisfaction is extended
to other propositional connectives as usual. 

The {\em LJN-reduct} of a program $\Pi$ with respect to an
LJN-interpretation $(X,T)$, denoted by $\Pi^{(X,T)}$,
consists of rules 
$
  a \ar B
$
for each rule~\eqref{nlplc} such that $(X,T)$ satisfies 
$
N\land LC$.
$(X,T)$ is an {\em LJN-answer set} of $\Pi$ if $(X,T)$ satisfies
$\Pi$, and $X$ is the smallest set of atoms satisfying $\Pi^{(X,T)}$.

The following theorem 
tells us that there is a one-to-many relationship between LJN-answer sets 
and the stable models in the sense of ASPMT. Essentially, the set of theory atoms in an LJN-answer set encodes all valid mappings for functions in the stable model semantics.


\begin{thm}\label{thm:niemelafsm}\optional{thm:niemelafsm}
Let $\Pi$ be an ASP(LC) program of signature $\langle \sigma^p, \sigma^f\rangle$ where $\sigma^p$ is a set of propositional constants, and let $\sigma^f$ be a set of object constants, and let $I^f$ be an interpretation of $\sigma^f$.
\begin{itemize}
\item[(a)] If $(X,T)$ is an LJN-answer set of $\Pi$, then for any ${\cal T}$-interpretation $I$ such that $I^f\models T\cup\overline{T}$, we have $\langle I^f, X\rangle\models\fsm[\Pi;\sigma^p]$.

\item[(b)] For any ${\cal T}$-interpretation $I=\langle I^f, X\rangle$, 
  if  $\langle I^f, X\rangle\models\fsm[\Pi;\sigma^p]$, then an LJN-interpretation 
   $(X, T)$ where $$T=\{t\mid \text{$t$ is a theory atom in $\Pi$ such that $I^f\models t$}\}$$ is an LJN-answer set of $\Pi$.
\end{itemize}
\end{thm}

\begin{example}
Let $F$ be
\[
\ba{ll}
a \ar x\!-\!z\!>\!0. \hspace{1cm} & b\ar x\!-\!y\!\le\!0. \\ 
c \ar b,\  y\!-\!z\!\le\!0. & \ar \no\ a. \\
b \ar c.
\ea
\]
The LJN-interpretation 
$L = \langle \{a\},\{x\!-\!z\!>\!0\}\rangle$ is an answer set of~$F$
since 
$\{(x\!-\!z\!>\!0, \neg(x\!-\!y\!\le\!0), \neg(y\!-\!z\!\le\!0)\}$ 
is satisfiable (e.g., take $x^I\mvis 2, y^I\mvis 1, z^I\mvis 0$) and
the set $\{a\}$ is the minimal model satisfying the reduct $F^L$, which is equivalent to $(\top \rightarrow a)\land (c\rightarrow b)$. 
In accordance with Theorem~\ref{thm:niemelafsm}, the interpretation $I$ such that 
$x^I\mvis 2, y^I\mvis 1, z^I\mvis 0, a^I\mvis\true, b^I\mvis\false,
c^I\mvis\false$ satisfies $I\models\fsm[F;abc]$.
\end{example}

As with {clingcon} programs, ASP(LC) programs do not allow intensional functions. 






\BOC
The concept of ASPT programs from~\cite{lierler16constraint} is essentially a generalization of ASP(LC) programs to allow arbitrary background theories besides linear constraints (LC). That paper showed that tight ASPT programs can be computed by SMT solvers by the process of completion. Unlike ASPMT programs, ASPT programs do not allow intensional functions, and have disjoint vocabularies for ordinary atoms (which they call ``regular'' atoms) vs. constraints or theory atoms (which they call ``irregular'' atoms): constraint variables and uninterpreted functions in the latter may not be a part of ordinary atoms. Theorem~\ref{thm:niemelafsm} remains valid when $\Pi$ is an ASPT program, the result which provides a link between ASPT and ASPMT programs.
\EOC

\subsection{Relation to Lin-Wang Programs} \label{ssec:lw}
\citeauthor{lin08answer} (\citeyear{lin08answer}) extended answer set semantics
with functions by extending the definition of a reduct, and also
provided loop formulas for such programs.  We can provide an
alternative account of their results by considering the notions there
as special cases of the definitions presented in this paper. 
Essentially, they restricted attention to a special case of non-Herbrand
interpretations such that object constants form the universe, and
ground terms other than object constants are mapped to the object
constants.
More precisely, according to~\cite{lin08answer}, an {\em LW-program} $P$ consists of {\sl
type definitions} and a set of rules of the form
\beq
   A\ar B_1,\dots, B_m, \no\ C_1, \dots, \no\ C_n
\eeq{lw-rule}
where $A$ is $\bot$ or an atom, and $B_i$ ($1\le i\le m$) and $C_j$ ($1\le j\le n$) are atomic formulas possibly containing equality. 
Type definitions are essentially a special case of many-sorted signature declarations, 
where each sort is a set of object constants.
%
%
For such many-sorted signature, we say that a many-sorted interpretation $I$ is a {\em $P$-interpretation} if it evaluates each object constant to itself, and each ground term other than object constants to an object constant conforming to the type definitions of $P$.
The {\em functional reduct} of $P$ under $I$ is a normal logic program without functions obtained from $P$ by 
\begin{enumerate}
\item  replacing each functional term $f(t_1, \dots, t_n)$ with $c$ where $f^I(t_1 , \dots, t_n) = c$; 
\item  removing any rule containing $c\ne c$ or $c=d$ where $c$, $d$ are distinct constants; 
\item  removing any remaining equalities from the remaining rules; 
\item  removing any rule containing $\no\ A$ in the body of the rule where $A^I = \true$;
\item  removing any remaining $\no\ A$ from the bodies of the remaining rules.
\end{enumerate}

A $P$-interpretation is an answer set of $P$ in the sense of \cite{lin08answer} if $I$ is the minimal model of $P^I$. 

The following theorem tells us that $LW$ programs are a special case of FSM formulas whose function constants are non-intensional. 

\begin{thm}\label{thm:lw-fsm}\optional{thm:lw-fsm}
Let $P$ be an LW-program and let $F$ be the FOL-representation of the
set of rules in $P$. The following conditions are equivalent to each
other:
\begin{itemize}
\item[(a)]  $I$ is an answer set of $P$ in the sense of~\cite{lin08answer};
\item[(b)]  $I$ is a $P$-interpretation that satisfies $\fsm[F; {\bf
    p}]$ where ${\bf p}$ is the list of all predicate constants
  occurring in $F$. 
\end{itemize}
\end{thm}



In other words, like clingcon programs, Lin-Wang programs can be identified with a special case of the first-order stable model semantics from~\cite{ferraris11stable}, which do not allow intensional functions.




\section{Conclusion}\label{sec:conclusion}


In this paper, we presented the functional stable model semantics, which properly extends the first-order stable model semantics to distinguish between intensional and non-intensional functions. We observe that many properties known for the first-order stable model semantics naturally extend to the functional stable model semantics. 

The presented semantics turns out to be useful for overcoming the limitations of the stable model semantics originating from the propositional setting, and enables us to combine with other related formalisms where general functions play a central role in efficient computation. 
ASPMT benefits from the expressiveness of ASP modeling language while leveraging efficient constraint/theory solving methods originating from SMT. For instance, it provides a viable approach to nonmonotonic reasoning about hybrid transitions where discrete and continuous changes co-exist. 
%

The relationship between ASPMT and SMT is similar to the relationship between ASP and SAT. We expect that, in addition to completion and the results shown in this paper, many other results known between ASP and SAT can be carried over to the relationship between ASPMT and SMT, thereby contributing to  efficient first-order reasoning in answer set programming.
A future work is to lift the limitation of the current ASPMT implementation limited to tight programs by designing and implementing a native computation algorithm which borrows the techniques from SMT, similar to the way that ASP solvers adapted SAT solving computation.

{\bf Acknowledgements}\ \ 
We are grateful to Yi Wang and Nikhil Loney for many useful discussions and to the anonymous referees for their constructive comments.
This work was partially supported by the National Science Foundation under Grants IIS-1319794, IIS-1526301, and IIS-1815337.

\bibliographystyle{named}


\begin{thebibliography}{}

\bibitem[\protect\citeauthoryear{Asuncion \bgroup \em et al.\egroup
  }{2015}]{asuncion15ordered}
Vernon Asuncion, Yin Chen, Yan Zhang, and Yi~Zhou.
\newblock Ordered completion for logic programs with aggregates.
\newblock {\em Artificial Intelligence}, 224:72--102, 2015.

\bibitem[\protect\citeauthoryear{Babb and Lee}{2013}]{babb13cplus2asp}
Joseph Babb and Joohyung Lee.
\newblock {C}plus2{A}{S}{P}: Computing action language {$\cal C$+} in answer
  set programming.
\newblock In {\em Proceedings of International Conference on Logic Programming
  and Nonmonotonic Reasoning (LPNMR)}, pages 122--134, 2013.

\bibitem[\protect\citeauthoryear{Balduccini}{2009}]{balduccini09representing}
Marcello Balduccini.
\newblock Representing constraint satisfaction problems in answer set
  programming.
\newblock In {\em Working Notes of the Workshop on Answer Set Programming and
  Other Computing Paradigms (ASPOCP)}, 2009.

\bibitem[\protect\citeauthoryear{Balduccini}{2012}]{balduccini12aconservative}
Marcello Balduccini.
\newblock A ``conservative'' approach to extending answer set programming with
  non-{H}erbrand functions.
\newblock In {\em Correct Reasoning - Essays on Logic-Based AI in Honour of
  Vladimir Lifschitz}, pages 24--39, 2012.

\bibitem[\protect\citeauthoryear{Barrett \bgroup \em et al.\egroup
  }{2009}]{barrett09satisfiability}
Clark~W. Barrett, Roberto Sebastiani, Sanjit~A. Seshia, and Cesare Tinelli.
\newblock Satisfiability modulo theories.
\newblock In Armin Biere, Marijn Heule, Hans van Maaren, and Toby Walsh,
  editors, {\em Handbook of Satisfiability}, volume 185 of {\em Frontiers in
  Artificial Intelligence and Applications}, pages 825--885. IOS Press, 2009.

\bibitem[\protect\citeauthoryear{Bartholomew and
  Lee}{2012}]{bartholomew12stable}
Michael Bartholomew and Joohyung Lee.
\newblock Stable models of formulas with intensional functions.
\newblock In {\em Proceedings of International Conference on Principles of
  Knowledge Representation and Reasoning (KR)}, pages 2--12, 2012.

\bibitem[\protect\citeauthoryear{Bartholomew and
  Lee}{2013a}]{bartholomew13functional}
Michael Bartholomew and Joohyung Lee.
\newblock Functional stable model semantics and answer set programming modulo
  theories.
\newblock In {\em Proceedings of International Joint Conference on Artificial
  Intelligence (IJCAI)}, 2013.

\bibitem[\protect\citeauthoryear{Bartholomew and
  Lee}{2013b}]{bartholomew13afunctional}
Michael Bartholomew and Joohyung Lee.
\newblock A functional view of strong negation.
\newblock In {\em Working Notes of the 5th Workshop on Answer Set Programming
  and Other Computing Paradigms (ASPOCP)}, 2013.

\bibitem[\protect\citeauthoryear{Bartholomew and
  Lee}{2013c}]{bartholomew13onthestable}
Michael Bartholomew and Joohyung Lee.
\newblock On the stable model semantics for intensional functions.
\newblock {\em Theory and Practice of Logic Programming}, 13(4-5):863--876,
  2013.

\bibitem[\protect\citeauthoryear{Bartholomew and
  Lee}{2014}]{bartholomew14system}
Michael Bartholomew and Joohyung Lee.
\newblock System {ASPMT2SMT}: Computing aspmt theories by smt solvers.
\newblock In {\em Proceedings of {E}uropean Conference on Logics in Artificial
  Intelligence ({JELIA})}, pages 529--542, 2014.

\bibitem[\protect\citeauthoryear{Brewka \bgroup \em et al.\egroup
  }{2011}]{bre11}
Gerhard Brewka, Ilkka Niemel\"{a}, and Miroslaw Truszczynski.
\newblock Answer set programming at a glance.
\newblock {\em Communications of the ACM}, 54(12):92--103, 2011.

\bibitem[\protect\citeauthoryear{Cabalar}{2011}]{cabalar11functional}
Pedro Cabalar.
\newblock Functional answer set programming.
\newblock {\em Theory and Practice of Logic Programming}, 11(2-3):203--233,
  2011.

\bibitem[\protect\citeauthoryear{Clark}{1978}]{cla78}
Keith Clark.
\newblock Negation as failure.
\newblock In Herve Gallaire and Jack Minker, editors, {\em Logic and Data
  Bases}, pages 293--322. Plenum Press, New York, 1978.

\bibitem[\protect\citeauthoryear{Ferraris \bgroup \em et al.\egroup
  }{2007}]{fer07a}
Paolo Ferraris, Joohyung Lee, and Vladimir Lifschitz.
\newblock A new perspective on stable models.
\newblock In {\em Proceedings of International Joint Conference on Artificial
  Intelligence ({IJCAI})}, pages 372--379, 2007.

\bibitem[\protect\citeauthoryear{Ferraris \bgroup \em et al.\egroup
  }{2009}]{ferraris09symmetric}
Paolo Ferraris, Joohyung Lee, Vladimir Lifschitz, and Ravi Palla.
\newblock Symmetric splitting in the general theory of stable models.
\newblock In {\em Proceedings of International Joint Conference on Artificial
  Intelligence (IJCAI)}, pages 797--803. AAAI Press, 2009.

\bibitem[\protect\citeauthoryear{Ferraris \bgroup \em et al.\egroup
  }{2011}]{ferraris11stable}
Paolo Ferraris, Joohyung Lee, and Vladimir Lifschitz.
\newblock Stable models and circumscription.
\newblock {\em Artificial Intelligence}, 175:236--263, 2011.

\bibitem[\protect\citeauthoryear{Fox and Long}{2003}]{fox03pddl}
Maria Fox and Derek Long.
\newblock {P}{D}{D}{L}2.1: An extension to {PDDL} for expressing temporal
  planning domains.
\newblock {\em J. Artif. Intell. Res. (JAIR)}, 20:61--124, 2003.

\bibitem[\protect\citeauthoryear{Fox and Long}{2006}]{fox06modelling}
Maria Fox and Derek Long.
\newblock Modelling mixed discrete-continuous domains for planning.
\newblock {\em J. Artif. Intell. Res. (JAIR)}, 27:235--297, 2006.

\bibitem[\protect\citeauthoryear{Gao \bgroup \em et al.\egroup
  }{2013a}]{gao13satisfiability}
Sicun Gao, Soonho Kong, and Edmund Clarke.
\newblock Satisfiability modulo {O}{D}{E}s.
\newblock {\em arXiv preprint arXiv:1310.8278}, 2013.

\bibitem[\protect\citeauthoryear{Gao \bgroup \em et al.\egroup
  }{2013b}]{gao13dreal}
Sicun Gao, Soonho Kong, and Edmund~M Clarke.
\newblock d{R}eal: An {S}{M}{T} solver for nonlinear theories over the reals.
\newblock In {\em International Conference on Automated Deduction}, pages
  208--214. Springer Berlin Heidelberg, 2013.

\bibitem[\protect\citeauthoryear{Gebser \bgroup \em et al.\egroup
  }{2009}]{gebser09constraint}
M.~Gebser, M.~Ostrowski, and T.~Schaub.
\newblock Constraint answer set solving.
\newblock In {\em Proceedings of International Conference on Logic Programming
  (ICLP)}, pages 235--249, 2009.

\bibitem[\protect\citeauthoryear{Gelfond and Kahl}{2014}]{gelfond14knowledge}
Michael Gelfond and Yulia Kahl.
\newblock {\em Knowledge Representation, Reasoning, and the Design of
  Intelligent Agents}.
\newblock Cambridge University Press, 2014.

\bibitem[\protect\citeauthoryear{Gelfond and Lifschitz}{1988}]{gel88}
Michael Gelfond and Vladimir Lifschitz.
\newblock The stable model semantics for logic programming.
\newblock In Robert Kowalski and Kenneth Bowen, editors, {\em Proceedings of
  International Logic Programming Conference and Symposium}, pages 1070--1080.
  MIT Press, 1988.

\bibitem[\protect\citeauthoryear{Giunchiglia \bgroup \em et al.\egroup
  }{2004}]{giu04}
Enrico Giunchiglia, Joohyung Lee, Vladimir Lifschitz, Norman McCain, and Hudson
  Turner.
\newblock Nonmonotonic causal theories.
\newblock {\em Artificial Intelligence}, 153(1--2):49--104, 2004.

\bibitem[\protect\citeauthoryear{Janhunen \bgroup \em et al.\egroup
  }{2011}]{janhunen11tight}
Tomi Janhunen, Guohua Liu, and Ilkka Niemel{\"a}.
\newblock Tight integration of non-ground answer set programming and
  satisfiability modulo theories.
\newblock In {\em Working notes of the 1st Workshop on Grounding and
  Transformations for Theories with Variables}, 2011.

\bibitem[\protect\citeauthoryear{Lee and Meng}{2013}]{lee13answer}
Joohyung Lee and Yunsong Meng.
\newblock Answer set programming modulo theories and reasoning about continuous
  changes.
\newblock In {\em Proceedings of International Joint Conference on Artificial
  Intelligence (IJCAI)}, 2013.

\bibitem[\protect\citeauthoryear{Lee \bgroup \em et al.\egroup
  }{2013}]{lee13action}
Joohyung Lee, Vladimir Lifschitz, and Fangkai Yang.
\newblock Action language {$\cal BC$}: Preliminary report.
\newblock In {\em Proceedings of International Joint Conference on Artificial
  Intelligence (IJCAI)}, 2013.

\bibitem[\protect\citeauthoryear{Lee \bgroup \em et al.\egroup
  }{2017}]{lee17representing}
Joohyung Lee, Nikhil Loney, and Yunsong Meng.
\newblock Representing hybrid automata by action language modulo theories.
\newblock {\em Theory and Practice of Logic Programming}, 2017.

\bibitem[\protect\citeauthoryear{Lierler and
  Susman}{2016}]{lierler16constraint}
Yuliya Lierler and Benjamin Susman.
\newblock Constraint answer set programming versus satisfiability modulo
  theories.
\newblock In {\em IJCAI}, pages 1181--1187, 2016.

\bibitem[\protect\citeauthoryear{Lifschitz and Turner}{1999}]{lif99b}
Vladimir Lifschitz and Hudson Turner.
\newblock Representing transition systems by logic programs.
\newblock In {\em Proceedings of International Conference on Logic Programming
  and Nonmonotonic Reasoning ({LPNMR})}, pages 92--106, 1999.

\bibitem[\protect\citeauthoryear{Lifschitz and
  Yang}{2011}]{lifschitz11eliminating}
Vladimir Lifschitz and Fangkai Yang.
\newblock Eliminating function symbols from a nonmonotonic causal theory.
\newblock In Gerhard Lakemeyer and Sheila~A. McIlraith, editors, {\em Knowing,
  Reasoning, and Acting: Essays in Honour of Hector J. Levesque}. College
  Publications, 2011.

\bibitem[\protect\citeauthoryear{Lifschitz and
  Yang}{2013}]{lifschitz13functional}
Vladimir Lifschitz and Fangkai Yang.
\newblock Functional completion.
\newblock {\em Journal of Applied Non-Classical Logics}, 23(1-2):121--130,
  2013.

\bibitem[\protect\citeauthoryear{Lifschitz \bgroup \em et al.\egroup
  }{2001}]{lif01}
Vladimir Lifschitz, David Pearce, and Agustin Valverde.
\newblock Strongly equivalent logic programs.
\newblock {\em ACM Transactions on Computational Logic}, 2:526--541, 2001.

\bibitem[\protect\citeauthoryear{Lifschitz}{1988}]{lif88}
Vladimir Lifschitz.
\newblock On the declarative semantics of logic programs with negation.
\newblock In Jack Minker, editor, {\em Foundations of Deductive Databases and
  Logic Programming}, pages 177--192. Morgan Kaufmann, San Mateo, CA, 1988.

\bibitem[\protect\citeauthoryear{Lifschitz}{1994}]{lif93e}
Vladimir Lifschitz.
\newblock Circumscription.
\newblock In D.M. Gabbay, C.J. Hogger, and J.A. Robinson, editors, {\em
  Handbook of Logic in AI and Logic Programming}, volume~3, pages 298--352.
  Oxford University Press, 1994.

\bibitem[\protect\citeauthoryear{Lifschitz}{1997}]{lifschitz97onthelogic}
Vladimir Lifschitz.
\newblock On the logic of causal explanation.
\newblock {\em Artificial Intelligence}, 96:451--465, 1997.

\bibitem[\protect\citeauthoryear{Lifschitz}{2008}]{lif08}
Vladimir Lifschitz.
\newblock What is answer set programming?
\newblock In {\em Proceedings of the AAAI Conference on Artificial
  Intelligence}, pages 1594--1597. MIT Press, 2008.

\bibitem[\protect\citeauthoryear{Lifschitz}{2011}]{lifschitz11datalog}
Vladimir Lifschitz.
\newblock Datalog programs and their stable models.
\newblock In O.~de~Moor, G.~Gottlob, T.~Furche, and A.~Sellers, editors, {\em
  Datalog Reloaded: First International Workshop, Datalog 2010, Oxford, UK,
  March 16-19, 2010. Revised Selected Papers}. Springer, 2011.

\bibitem[\protect\citeauthoryear{Lifschitz}{2012}]{lifschitz12logic}
Vladimir Lifschitz.
\newblock Logic programs with intensional functions.
\newblock In {\em Proceedings of International Conference on Principles of
  Knowledge Representation and Reasoning (KR)}, pages 24--31, 2012.

\bibitem[\protect\citeauthoryear{Lin and Wang}{2008}]{lin08answer}
Fangzhen Lin and Yisong Wang.
\newblock Answer set programming with functions.
\newblock In {\em Proceedings of International Conference on Principles of
  Knowledge Representation and Reasoning (KR)}, pages 454--465, 2008.

\bibitem[\protect\citeauthoryear{Lin}{1995}]{lin95}
Fangzhen Lin.
\newblock Embracing causality in specifying the indirect effects of actions.
\newblock In {\em Proceedings of International Joint Conference on Artificial
  Intelligence ({IJCAI})}, pages 1985--1991, 1995.

\bibitem[\protect\citeauthoryear{Liu \bgroup \em et al.\egroup
  }{2012}]{liu12answer}
Guohua Liu, Tomi Janhunen, and Ilkka Niemel{\"a}.
\newblock Answer set programming via mixed integer programming.
\newblock In {\em Proceedings of International Conference on Principles of
  Knowledge Representation and Reasoning (KR)}, pages 32--42, 2012.

\bibitem[\protect\citeauthoryear{McCarthy}{1980}]{mcc80}
John McCarthy.
\newblock Circumscription---a form of non-mono\-tonic reasoning.
\newblock {\em Artificial Intelligence}, 13:27--39,171--172, 1980.

\bibitem[\protect\citeauthoryear{Mellarkod \bgroup \em et al.\egroup
  }{2008}]{mellarkod08integrating}
Veena~S Mellarkod, Michael Gelfond, and Yuanlin Zhang.
\newblock Integrating answer set programming and constraint logic programming.
\newblock {\em Annals of Mathematics and Artificial Intelligence},
  53(1-4):251--287, 2008.

\bibitem[\protect\citeauthoryear{Wa{\l}{{e}}ga \bgroup \em et al.\egroup
  }{2015}]{walega15aspmt}
Przemys{\l}aw~Andrzej Wa{\l}{{e}}ga, Mehul Bhatt, and Carl Schultz.
\newblock {ASPMT(QS)}: non-monotonic spatial reasoning with answer set
  programming modulo theories.
\newblock In {\em Logic Programming and Nonmonotonic Reasoning}, pages
  488--501. Springer, 2015.

\end{thebibliography}

\appendix 
\section{Review of Reduct-Based Definition of Stable Models}

Some of the proofs below use the definition of functional stable models based on the notions of an infinitary ground formula and a reduct from~\cite{bartholomew13onthestable}. We review the semantics below.

\subsection{Infinitary Ground Formulas}\label{sssec:igf}

\BOCC
We define {\em infinitary ground formulas}, which are slightly adapted
from infinitary propositional formulas
from~\cite{truszczynski12connecting}. Unlike
in~\cite{truszczynski12connecting}, 
we do not replace ground terms with their corresponding
object names, keeping them intact during grounding. This difference
is required in defining a reduct for the functional stable model
semantics.\footnote{Another difference is that grounding
  in~\cite{truszczynski12connecting} refers to ``infinitary
  propositional formulas,'' which can be defined on any propositional
  signature. This generality is not essential for our purpose in this paper.}

More specifically, 
\EOCC
We assume that a signature and an interpretation are defined the same as in the standard first-order logic. 
For each element $\xi$ in the universe $|I|$ of $I$, we introduce a
new symbol $\xi^\dia$, called an {\sl object name}. By $\sigma^I$ we
denote the signature obtained from~$\sigma$ by adding all object names
$\xi^\dia$ as additional object constants. We will identify an
interpretation $I$ of signature $\sigma$ with its extension
to~$\sigma^I$ defined by $I(\xi^\dia)=\xi$.

We assume the primary connectives of infinitary ground formulas to be $\bot$, $\{\}^\land$, $\{\}^\lor$, and $\rar$. The usual propositional connectives $\land,\lor$ are considered as shorthands: $F\land G$ as $\{F,G\}^\land$, and $F\lor G$ as~$\{F,G\}^\lor$.

Let $A$ be the set of all ground atomic formulas of signature $\sigma^I$. The
sets ${\cal F}_0, {\cal F}_1, \dots$ are defined recursively as follows:
\begin{itemize}
\item ${\cal F}_0=A\cup\{\bot\}$;
\item ${\cal F}_{i+1} (i\ge 0)$ consists of expressions
  ${\cal H}^\lor$ and ${\cal H}^\land$, for all subsets
  ${\cal H}$ of ${\cal F}_0\cup\ldots\cup{\cal F}_i$, and of
  the expressions $F\rar G$, where
  $F$ and $G$ belong to ${\cal F}_0\cup\dots\cup{\cal F}_i$.
\end{itemize}
We define ${\cal L}_A^{inf}=\bigcup_{i=0}^{\infty}{\cal F}_i$, and call
elements of ${\cal L}_A^{inf}$ {\em infinitary ground formulas}
of~$\sigma$ w.r.t.~$I$.

For any interpretation $I$ of $\sigma$ and any infinitary ground
formula $F$ w.r.t.~$I$, the definition of satisfaction, $I\models F$, is as
follows:
\begin{itemize}
\item For atomic formulas, the definition of satisfaction 
  is the same as in the standard first-order logic;

\item $I\models{\cal H}^\lor$ if there is a formula $G\in {\cal H}$ 
   such that $I\models G$;

\item $I\models{\cal H}^\land$ if, for every formula $G\in {\cal H}$, 
   $I \models G$;

\item $I \models G\rar H$ if $I \not\models G$ or $I \models H$.
\end{itemize}

Given an interpretation, we identify any first-order sentence with an infinitary ground formula via the process of grounding relative to that interpretation. 
Let $F$ be any first-order sentence of a signature $\sigma$, and let
$I$ be an interpretation of~$\sigma$. By $gr_I[F]$ we denote the
infinitary ground formula w.r.t.~$I$ that is obtained from $F$ by the
following process:
\begin{itemize}
\item  If $F$ is an atomic formula, $gr_I[F]$ is $F$; 
\item  $gr_I[G\odot H]= gr_I[G]\odot gr_I[H]\ \ \ \
 (\odot\in\{\land,\lor,\rar\})$;

\item  $gr_I[\exists x G(x)] = 
      \{gr_I[G(\xi^\diamond)] \mid \xi\in |I|\}^\lor$; $\qquad$
\item  $gr_I[\forall x G(x)] = 
      \{gr_I[G(\xi^\diamond)] \mid \xi\in |I|\}^\land$.
\end{itemize}

\subsection{Stable Models in terms of Grounding and Reduct}\label{sssec:bl-reduct}

For any two interpretations $I$, $J$ of the same signature and any list $\bC$ of distinct predicate and function constants, we write $J<^\bC I$ if 
\begin{itemize}
\item  $J$ and $I$ have the same universe and agree on all constants
  not in $\bC$,

\item  $p^J\subseteq p^I$ for all predicate constants $p$ in $\bC$,\footnote{For any symbol $c$ in a signature, $c^I$ denotes the evaluation of $I$ on $c$.}
and 

\item  $J$ and $I$ do not agree on $\bC$. 
\end{itemize}
%




The {\em reduct} $F^\mu{I}$ of an infinitary ground formula $F$ relative to an interpretation $I$ is defined as follows: 
\begin{itemize}
\item For any atomic formula $F$, $F^\mu{I}= \left\{
  \ba {ll} 
     \bot & \text{ if $I\not\models F$ } \\
     F & \text{ otherwise. }
  \ea
  \right. 
$

\item $({\cal H}^\land)^\mu{I}= \{G^\mu{I} \mid G\in{\cal H}\}^\land$

\item $({\cal H}^\lor)^\mu{I}=  \{G^\mu{I} \mid G\in{\cal H}\}^\lor$

\item 
$
   (G\rar H)^\mu{I}= \left\{
   \ba {ll}
      \bot & \text{ if $I\not\models G\rar H$ } \\
       G^\mu{I} \rar H^\mu{I} & \text{ otherwise. }
   \ea
   \right.
$


\end{itemize}








The following theorem presents an alternative definition of a stable model that is equivalent to the one in the previous section. 

\begin{thm}[Theorem~1 from \cite{bartholomew13onthestable}]\label{thm:fsm-reduct}\optional{thm:fsm-reduct}
Let $F$ be a sentence and let $\bC$ be a list of intensional constants. An interpretation $I$ satisfies $\sm[F; \bC]$ iff 
\begin{itemize}
\item  $I$ satisfies $F$, and 
\item  no interpretation $J$ such that $J<^\bC I$ satisfies $(gr_I[F])^\mu{I}$. 
\end{itemize}  
\end{thm}


\BOCCC
\subsection{Proof of Theorem~\ref{thm:fsm-reduct}}

We will often use the following notation throughout this section. 
Let $\sigma$ be a first-order signature, let $\bC$ be a set of
constants that is a subset of~$\sigma$, and let $\bD$ be a set of
constants not belonging to $\sigma$ and is similar to
$\bC$.\footnote{That is to say, $\bD$ and $\bC$ have the same length
  and the corresponding members are either predicate constants of the
  same arity or function constants of the same arity.} 
For any interpretation $J$ of signature $\sigma$,
$J^\bC_\bD$ denotes the interpretation of
signature~\hbox{$(\sigma\setminus \bC)\cup \bD$} obtained from
$J$ by replacing every constant from $\bC$ with the corresponding
constant from~$\bD$.
For two interpretations $I$ and $J$ of~$\sigma$ that agree on
all constants in $\sigma\setminus\bC$, we define $J^\bC_\bD\cup I$ to
be the interpretation from the extended signature $\sigma\cup\bD$ such
that
\begin{itemize}
\item $J^\bC_\bD\cup I$ agrees with $I$ on all constants in~$\bC$; 
\item $J^\bC_\bD\cup I$ agrees with $J^\bC_\bD$ on all constants
  in~$\bD$;
\item $J^\bC_\bD\cup I$ agrees with both $I$ and $J$ on all constants
  in~$\sigma\setminus \bC$.
\end{itemize}

\begin{lemma}\label{lem:reduct}\optional{lem:reduct}
Let $F$ be a sentence of signature~$\sigma$, and let $I$ and $J$ be
interpretations of $\sigma$ such that $J<^\bC I$. We have 
$J^\bC_\bD\cup I\models F^*(\bD)$ 
iff $J\models gr_I[F]^\mu{I}$.
\end{lemma}

\proof
By induction on the structure of $F$. 

\noindent
{\sl Case 1:} $F$ is an atomic sentence. Then $F^*(\bD)$ is
$F(\bD)\land F$, where $F(\bD)$ is obtained from $F$ by replacing the
members of $\bC$ with the corresponding members of~$\bD$. Consider the
following subcases:

\begin{itemize}
\item {\sl Subcase 1:} $I\not\models F$. 
Then $J^\bC_\bD\cup I\not\models F^*(\bD)$. 
Further, $gr_I[F]^\mu{I} = \bot$, so $J\not\models gr_I[F]^\mu{I}$.

\item {\sl Subcase 2:} $I\models F$. 
Then $J^\bC_\bD\cup I\models F^*(\bD)$ 
iff $J^\bC_\bD\models F(\bD)$ iff $J\models F$. 
Further, $gr_I[F]^\mu{I} = F$, so 
$J\models gr_I[F]^\mu{I}$ iff $J\models F$.
\end{itemize}

\smallskip\noindent
{\sl Case 2:} $F$ is $G\land H$ or $G\lor H$. 
The claim follows immediately from I.H. on $G$ and~$H$.

\smallskip\noindent
{\sl Case 3:} $F$ is $G\rar H$. Then 
$F^*(\bD) = (G^*(\bD)\rar H^*(\bD))\land (G\rar H)$. 
Consider the following subcases:

\begin{itemize}
\item {\sl Subcase 1:} $I\not\models G\rar H$. 
  Then $J^\bC_\bD\cup I\not\models F^*(\bD)$. 
  Further, $gr_I[F]^\mu{I}=\bot$, which $J$ does not satisfy. 

\item {\sl Subcase 2:} $I\models G\rar H$. 
  Then $J^\bC_\bD\cup I\models F^*(\bD)$ iff 
  $J^\bC_\bD\cup I\models G^*(\bD)\rar H^*(\bD)$.    
  On the other hand, $gr_I[F]^\mu{I} = gr_I[G]^\mu{I}\rar gr_I[H]^\mu{I}$ 
  so this case holds by I.H. on $G$ and $H$.
\end{itemize}

\smallskip\noindent
{\sl Case 4:} $F$ is $\exists x G(x)$. By I.H., 
$J^\bC_\bD\cup I\models G(\xi^\dia)^*(\bD)$ iff 
$J\models gr_I[G(\xi^\dia)]^\mu{I}$ for each $\xi\in |I|$. 
The claim follows immediately. 

\smallskip\noindent
{\sl Case 5:} $F$ is $\forall x G(x)$. Similar to Case 4. 
\qed

\begin{lemma}\label{lem:fsm-lt}\optional{lem:fsm-lt}
For any interpretations $I$ and $J$ of signature~$\sigma$, 
we have $J^\bC_\bD\cup I\models \bD < \bC$ iff $J <^\bC I$.
\end{lemma}

\proof
Recall that by definition, $\bD<\bC$ is 
\[ 
  (\bD^{pred}\le \bC^{pred})\land\neg (\bD = \bC), 
\]
and by definition, $J <^\bC I$ is
\begin{itemize}
\item  $J$ and $I$ have the same universe and agree on all constants
  not in $\bC$;
\item  $p^J\subseteq p^I$ for all predicate constants $p$ in $\bC$; and
\item  $J$ and $I$ do not agree on $\bC$. 
\end{itemize}

First, by definition of $J^\bC_\bD\cup I$, $J$ and $I$ have the same universe
and agree on all constants in $\sigma\setminus\bC$.

Second, by definition, 
$J^\bC_\bD\cup I\models \bD^{pred} \leq \bC^{pred}$ 
iff, for every predicate constant $p$ in~$\bC$, 
\[ 
  J^\bC_\bD\cup I\models
    \forall {\bf x} (p({\bf x})^\bC_\bD\rar p({\bf x})),
    \footnote{$p({\bf x})^\bC_\bD$ denotes the atom that is obtained
      from~$p({\bf x})$ by replacing $p$ with the
     corresponding member of $\bD$ if $p\in \bC$, and no change otherwise.}
\]
which is equivalent to saying that 
$(p^\bC_\bD)^{J^\bC_\bD\cup I}\subseteq p^{J^\bC_\bD\cup I}$. Since
$I$ does not interpret any constant from $\bD$, and $J^\bC_\bD$ does
not interpret any constant from $\bC$, this is equivalent to
$(p^\bC_\bD)^{J^\bC_\bD}\subseteq p^I$ and further to 
$p^J\subseteq p^I$.

Third, since $I$ does not interpret any constant from $\bD$ and 
$J^\bC_\bD$ does not interpret any constant from $\bC$,
$J^\bC_\bD\cup I\models \neg (\bD = \bC)$ is equivalent
to saying $J$ and $I$ do not agree on $\bC$. 
\qed

\bigskip\noindent
{\bf Theorem~\ref{thm:fsm-reduct} \optional{thm:fsm-reduct}}\   
{\sl Let $F$ be a sentence and let $\bC$ be a list of intensional constants. An interpretation $I$ satisfies $\sm[F; \bC]$ iff 
\begin{itemize}
\item  $I$ satisfies $F$, and 
\item  no interpretation $J$ such that $J<^\bC I$ satisfies $(gr_I[F])^\mu{I}$. 
\end{itemize}  
}\medskip

\proof 
$I \models \sm[F;\bC]$ is by definition 
\beq
  I\models F\land\neg\exists\vbC (\vbC<\bC\land F^*(\vbC)).
\eeq{reductsm}
The first item, ``$I$ satisfies $F$,'' is equivalent to the first conjunctive term of (\ref{reductsm}). 

By Lemma~\ref{lem:reduct} and Lemma~\ref{lem:fsm-lt}, the second item,
``no interpretation $J$ of~$\sigma$ such that $J <^\bC I$ satisfies
$gr_I[F]^\mu{I}$'', is equivalent to the second conjunctive term in
(\ref{reductsm}). 
\qed
\EOCCC

\section{Proofs}

\subsection{Proof of Theorem~\ref{thm:constraint}}

\BOCC
\begin{lemma}\label{lem:mono}\optional{lem:mono}
The formula
$$
(\wh{\bC} < \bC) \land F^*(\wh{\bC}) \rar F
$$
is logically valid.
\end{lemma}
\proof By induction on the structure of $F$.
\qed

\noindent{\bf Lemma~\ref{lem:neg}}\optional{lem:neg}\ \  
{\sl 
Formula
$$\wh{\bC}< \bC\rar((\neg F)^*(\wh{\bC})\lrar\neg F)$$
is logically valid.
}\medskip

\proof Immediate from Lemma~\ref{lem:mono}.
\qed
\EOCC

\noindent{\textbf {Theorem~\ref{thm:constraint} \optional{thm:constraint}}}
\ 
{\sl For any first-order formulas~$F$ and~$G$, if $G$ is negative on $\bC$, 
\hbox{$\sm[F\land G; \bC]$} is equivalent to~$\sm[F; \bC]\land G$.}

\proof By Lemma~\ref{lem:neg},
\begin{align*}
\sm[F\land\neg G;\ \bC]\  
  &\ = F\land\neg G \land \neg \exists \vbC ((\vbC<\bC)\land (F\land\neg G)^*(\vbC))\\
  &\ \Leftrightarrow F\land\neg G\land\neg\exists \vbC ((\vbC<\bC) \land F^*(\vbC)\land\neg G)\\
  &\ \Leftrightarrow F\land\neg\exists \vbC ((\vbC<\bC) \land F^*(\vbC))\land\neg G\\
  &\ = \sm[F;\ \bC]\land\neg G.
\end{align*}
\qed

\subsection{Proof of Theorem~\ref{thm:bi-un-sm}}

\begin{lemma}\label{lem:choice-star}\optional{lem:choice-star}
$\i{Choice}(\bC)^*(\vbC)$ is equivalent to 
\[ 
   (\bC^{pred}\le \vbC^{pred}) \land(\bC^{func}=\vbC^{func}).
\]
\end{lemma}
\proof
$\i{Choice}(\bC)$ is the conjunction for each predicate $p$ in $\bC^{pred}$ of 
\hbox{$\forall {\bf x}(p({\bf x}) \lor \neg p({\bf x}))$} and for each function $f$ in $\bC^{func}$ of
$\forall {\bf x}y(f({\bf x}) = y \lor \neg f({\bf x}) = y)$. 

First, 
$$[\forall {\bf x}(p({\bf x}) \lor \neg p({\bf x}))]^*(\vbC)$$ 
is equivalent to 
$$\forall {\bf x}(\wh{p}({\bf x}) \lor (\neg \wh{p}({\bf x})\land \neg p({\bf x}))),$$
which is further equivalent to
$$\forall {\bf x}(p({\bf x}) \rightarrow \wh{p}({\bf x})),$$
or simply $p \leq \wh{p}$.

Next,
$$[\forall {\bf x}y(f({\bf x}) = y \lor \neg (f({\bf x}) = y))]^*(\vbC)$$
is equivalent to
$$\forall {\bf x}y((\wh{f}({\bf x}) = y \land f({\bf x}) = y) \lor 
   (\neg (\wh{f}({\bf x}) = y)\land\neg (f({\bf x})=y))),$$
which is further equivalent to 
$$\forall {\bf x}y(f({\bf x}) = y \lrar \wh{f}({\bf x}) = y),$$
or simply $f = \wh{f}$. 

Thus, $\i{Choice}(\bC)^*(\vbC)$ is the conjunction for each predicate $p$ in $\bC^{pred}$ of 
$p \leq \wh{p}$ and for each function $f$ in $\bC^{func}$ of
$f = \wh{f}$, or simply $\i{Choice}(\bC)^*(\wh{\bC})$ is
\[ 
   (\bC^{pred}\le \vbC^{pred}) \land(\bC^{func}=\vbC^{func}). \qed
\]

\noindent{\textbf {Theorem~\ref{thm:bi-un-sm} \optional{thm:bi-un-sm}}}\
\ 
{\sl For any first-order formula $F$ and any disjoint lists $\bC$, $\bD$ of
distinct constants, the following formulas are logically valid: 
\[
\ba {l}
(i) \hspace{5 mm}   \sm[F; \bC\bD]\rar\sm[F; \bC] \\
(ii) \hspace{4 mm}  \sm[F\land\i{Choice}(\bD); \bC\bD]\lrar\sm[F;\bC].
\ea
\]
}
\bigskip

\proof
The proof is not long, but there is a notational difficulty that we
need to overcome before we can present it.  
The notation $F^*(\vbC)$ 
does not take into account the fact that the construction of this
formula depends on the choice of the list~$\bC$ of intensional
constants.  Since the dependence on~$\bC$ is essential in
the proof of Theorem~\ref{thm:bi-un-sm}, we use here the more
elaborate notation $F^{*[\bC]}(\vbC)$.  For instance, if~$F$
is~$p(x)\land q(x)$ then
$$
\ba {rcl}
   F^{*[p]}(\wh{p})&\hbox{ is }&\wh{p}(x)\land q(x),\\
F^{*[pq]}(\wh{p},\wh{q})&\hbox{ is }&\wh{p}(x)\land \wh{q}(x).
\ea
$$

It is easy to verify by induction on~$F$ that for any disjoint lists
$\bC$, $\bD$ of distinct predicate constants,
\beq
F^{*[\bC]}(\vbC)\,=\, F^{*[\bC\bD]}(\vbC,\bD).
\eeq{starequality}

(i) In the notation introduced above, $\sm[F;\bC]$ is
\[ 
  F\land\neg\exists\vbC((\vbC<\bC)\land F^{*[\bC]}(\vbC)).
\]
By~(\ref{starequality}), this formula can be written also as
\[ 
  F\land\neg\exists\vbC((\vbC<\bC)\land F^{*[\bC\bD]}(\vbC,\bD)),
\]
which is equivalent to
\beq 
  F\land\neg\exists\vbC(((\vbC,\bD)<(\bC,\bD))\land F^{*[\bC\bD]}(\vbC,\bD)).
\eeq{c}
On the other hand, $\sm[F;\bC\bD]$ is
\beq
  F\land\neg\exists\vbC\vbD(((\vbC,\vbD)<(\bC,\bD))
                                \land F^{*[\bC\bD]}(\vbC,\vbD)).
\eeq{cd}
It is clear that \eqref{cd} entails \eqref{c}.

\medskip
(ii) Note that, by~(\ref{starequality}) and
Lemma~\ref{lem:choice-star}, the formula
\[
  \exists \vbC\vbD ( ({\bf (\vbC,\vbD)}<{\bf (c,d)})
  \land F^{*[{\bf cd}]}(\vbC,\vbD)\land\i{Choice}({\bf
    d})^{*[{\bf cd}]}(\vbC,\vbD))
\]
is equivalent to
\[
  \exists \vbC\vbD ( ((\vbC,\vbD)<{\bf (c,d)})
  \land F^{*[{\bf cd}]}(\vbC,\vbD)\land({\bf d}=\vbD)).
\]
It follows that it can be also equivalently rewritten as
\[
  \exists \vbC((\vbC<{\bf c})\land F^{*[{\bf cd}]}(\vbC,{\bf d})).
\]
By~(\ref{starequality}), the last formula can be represented as
\[
  \exists \vbC((\vbC<{\bf c})\land F^{*[{\bf c}]}(\vbC)).
\]
\qed


\subsection{Proof of Theorem~\ref{thm:strong}}

Recall that about first-order formulas~$F$ and~$G$ we say that~$F$ is {\sl
  strongly equivalent} to~$G$ if, for any formula~$H$, any occurrence
of~$F$ in~$H$,
and any list $\bC$ of distinct predicate and function constants,
$\sm[H;\bC]$ is equivalent to $\sm[H';\bC]$, where~$H'$
is obtained from~$H$ by replacing the occurrence of~$F$ by~$G$.

\begin{lemma}\label{lemmareplace}\optional{lemmareplace}
Formula
$$
(F\lrar G)\land ((F^*(\vbC)\lrar G^*(\vbC)) 
\rar (H^*(\vbC) \lrar (H')^*(\vbC)))
$$
is logically valid.
\end{lemma}
\proof By induction on the structure of $H$.
\qed

The following lemma is equivalent to the ``only if'' part of Theorem~\ref{thm:strong}.

\begin{lemma}\label{lem:strong-onlyif}\optional{lem:strong-onlyif}
If the formula (\ref{eq:strong}) is logically valid, then $F$ is
strongly equivalent to $G$. 
\end{lemma} 

\proof 
Assume that (\ref{eq:strong}) is logically valid. We need to show that
\beq
H\land\neg\exists \vbC ((\vbC<\bC)\land H^*(\vbC))
\eeq{strf}
is equivalent to
\beq
H'\land\neg\exists \vbC ((\vbC<\bC)\land (H')^*(\vbC)).
\eeq{strg}
Since~(\ref{eq:strong}) is logically valid, 
the first conjunctive term of~(\ref{strf}) is equivalent to the first
conjunctive term of~(\ref{strg}).  By Lemma~\ref{lemmareplace}, it also
follows that the same relationship holds between
the two second conjunctive terms of the same formulas.\qed

The following lemma is equivalent to the ``if'' part of Theorem~\ref{thm:strong}.

\begin{lemma}\label{lem:strong-if}\optional{lem:strong-if}
If~$F$ is strongly equivalent to~$G$, then~(\ref{eq:strong}) is logically valid.
\end{lemma}

\proof 
Let $C$ be the formula $\i{Choice}(\bC)$.
Let $E$ stand for $F\lrar G$, and~$E'$ be $F\lrar F$.
Since $F$ is strongly equivalent to~$G$, the formula
$\sm[E \lrar C]$ is equivalent to $\sm[E' \lrar C]$.

Recall that by Lemma~\ref{lem:choice-star}, $\i{Choice}(\bC)^*(\vbC)$, which we abbreviate as $C^*$, is equivalent to 
\[ 
   (\bC^{pred}\le \vbC^{pred}) \land(\bC^{func}=\vbC^{func}).
\]
On the other hand, $\vbC<\bC$ can be equivalently rewritten as 
\[
  (\vbC^{pred}<\bC^{pred}) \lor 
       ((\vbC^{pred}=\bC^{pred})\land (\vbC^{func}\ne\bC^{func})).%
\]
It follows that 
\[
  \vbC<\bC \rar (C^*\lrar\bot)
\]
is logically valid. 

It is easy to see that $(E\lrar C)^*$ can be rewritten as
$$E\land (E^*(\vbC) \lrar C^*),$$ 
and that $E^*(\vbC)$ is equivalent to
$$E \land (F^*(\vbC)\lrar G^*(\vbC)).$$
Using these two facts and Lemma~\ref{lem:monotone}, we can simplify
$\sm[E\lrar C]$ as follows:
\begin{align*}
 \sm[E\lrar C]\ 
  & \Lrar (E\lrar C)\land\neg\exists \vbC((\vbC<\bC)\land
          E\land (E^*(\vbC)\lrar C^*))\\
  & \Lrar E\land\neg\exists \vbC ((\vbC<\bC)\land (E^*(\vbC)\lrar\bot))\\
  & \Lrar E\land\neg\exists \vbC ((\vbC<\bC)\land\neg E^*(\vbC))\\
  & \Lrar E\land\neg\exists \vbC ((\vbC<\bC)\land\neg (F^*(\vbC)\lrar G^*(\vbC)))\\
  &= (F\lrar G)\land\forall \vbC ((\vbC<\bC)\rar (F^*(\vbC)\lrar G^*(\vbC))). 
\end{align*}

Similarly, $\sm[E'\lrar C]$ is equivalent to
\[
  (F\lrar F)\land\forall \vbC
    ((\vbC<\bC)\rar (F^*(\vbC) \lrar F^*(\vbC))),
\]
which is logically valid.  Consequently,~(\ref{eq:strong}) is logically
valid also.\qed

\noindent{\textbf {Theorem~\ref{thm:strong} \optional{thm:strong}}}\
{\sl
Let $F$ and $G$ be first-order formulas, let $\bC$ be the
list of all constants occurring in $F$ or $G$, and let $\wh{\bC}$ be a
list of distinct predicate and function variables corresponding
to~$\bC$. The following conditions are equivalent to each other.
\begin{itemize}
\item  $F$ and $G$ are strongly equivalent to each other;
\item  Formula
\[
  (F\lrar G)\land(\wh{\bC} < \bC\rar (F^*(\wh{\bC})\lrar G^*(\wh{\bC})))
\] 
is logically valid.
\end{itemize}

}\bigskip

\proof
Immediate from Lemma~\ref{lem:strong-onlyif} and Lemma~\ref{lem:strong-if}.\qed


\subsection{Proof of Theorem~\ref{thm:completion}}

\NBB{refer to onfsm} 
\BOCC
Notation: the universe of an interpretation~$I$  is denoted by~$|I|$; for any
signature~$\sigma$ and any set~$U$,~$\sigma^U$ stands for the extension
of~$\sigma$ obtained by adding distinct new symbols $\xi^*$, called {\sl
names\/}, for all $\xi\in U$ as object constants.  We will identify an
interpretation~$I$ of~$\sigma$ with its extension to~$\sigma^{|I|}$ defined by
$I(\xi^*)=\xi$.  By~$\sigma_f$ we denote the part of~$\sigma$ consisting of
its object and function constants.
\EOCC

\begin{lemma}\label{lem:reductbody}\optional{lem:reductbody}
For any first-order sentence $F$, any list $\bC$ of constants, and any interpretations $I$ and $J$ such that $J <^\bC I$, if $I \models gr_I(F)^\mu{I}$ and $J \not \models gr_I(F)^\mu{I}$, then there is some constant $d$ occurring strictly positively in $F$ such that $d^I\ne d^J$.
\end{lemma}
\proof By induction on the structure of $F$.  
\qed

\BOC
\begin{itemize}
\item Case 1: $F$ is a ground atomic formula that may contain object names. In this case $gr_I(F)^\mu{I} = F$ since $I \models gr_I(F)^\mu{I}$. And since $J \not \models gr_I(F)^\mu{I}$, there must be at least one constant in $gr_I(F)^\mu{I}$ that $I$ and $J$ disagree on and since $gr_I(F)^\mu{I}$ is an atomic formula, this is a strictly positively occurrence.

\item Case 2: $F$ is $\mathcal{H}^\land$. Since $I \models gr_I(F)^\mu{I}$, $gr_I(F)^\mu{I}$ is $\mathcal{H'}^\land$ (and not $\bot$) where $\mathcal{H'} = \{gr_I(G)^\mu{I} | G \in \mathcal{H}\}$. Since $J \not \models gr_I(F)^\mu{I}$, $J \not \models gr_I(G)^\mu{I}$ for at least one $G \in \mathcal{H}$ so the claim follows by induction on whichever subformula $J$ does not satisfy since in any case, the subformula occurs strictly positively.

\item Case 3: $F$ is $\mathcal{H}^\lor$. Since $I \models gr_I(F)^\mu{I}$, $gr_I(F)^\mu{I}$ is $\mathcal{H'}^\lor$ (and not $\bot$) where $\mathcal{H'} = \{gr_I(G)^\mu{I} | G \in \mathcal{H}\}$. Since $J \not \models gr_I(F)^\mu{I}$, $J \not \models gr_I(G)^\mu{I}$ for every $G \in \mathcal{H}$. Now it could be that $I \not \models gr_I(G)^\mu{I}$ for some $G \in \mathcal{H}$ but not all of them. In such a case $gr_I(G)^\mu{I}$ would be $\bot$, which $I$ also does not satisfy. Thus the claim follows by induction on whichever of $G \in \mathcal{H}$ whose reduct $I$ satisfies.

\item Case 4: $F$ is $G\rightarrow H$. $gr_I(F)^\mu{I}$ is $gr_I(G)^\mu{I} \rightarrow gr_I(H)^\mu{I}$ (and not $\bot$). Since $J \not \models gr_I(F)^\mu{I}$, $J \models gr_I(G)^\mu{I}$ and $J \not \models gr_I(H)^\mu{I}$. Note that it must be the case then that $I \models gr_I(G)^\mu{I}$ since if not, it must be that $gr_I(G)^\mu{I}$ is $\bot$ and thus it is impossible for it to be that $J \models gr_I(G)^\mu{I}$. Consequently, it also follows that $I \models gr_I(H)^\mu{I}$ since $I \models gr_I(F)^\mu{I}$ so the claim follows by induction on the structure of $H$ since the subformula occurs strictly positively. \qed
\end{itemize}
\EOC

\begin{lemma} \label{lem:negative}\optional{lem:negative} 
If a ground formula $F$ is negative on a list ${\bf c}$ of predicate and function constants, then for every $J <^{\bf c} I$, 
\[
  J\models F^I \text{ iff } I\models F. 
\]
\end{lemma}

\proof By induction on the structure of $F$.  
\qed

\noindent{\textbf{Theorem~\ref{thm:completion} \optional{thm:completion}}}\  
{\sl For any formula $F$ in Clark normal form relative to ${\bf c}$ that is tight on ${\bf c}$, an interpretation $I$ that satisfies $\exists xy(x \ne y)$ is a model of $\fsm[F;{\bf c}]$ iff $I$ is a model of $\comp[F; {\bf c}]$.
}
\medskip

\proof
In this proof, we use Theorem~\ref{thm:fsm-reduct} and refer to the reduct-based characterization of a stable model.

\medskip\noindent
$(\Leftarrow)$ Take an interpretation $I$ that is a model of $\comp[F; {\bf c}]$. $I$ is clearly a model of $F$. We wish to show that, for any interpretation $J$ such that $J <^{\bf c} I$, we have $J\not\models gr_I[F]^\mu{I}$. 
Let $S$ be a subset of ${\bf c}$ consisting of constants $c$ on which $I$ and $J$ disagree,  that is, $c^I\ne c^J$. 
Let $s_0$ be a constant from $S$ such that there is no edge in the dependency graph from $s_0$ to any constant in $S$. Such an $s_0$ is guaranteed to exist since $F$ is tight on ${\bf c}$. 

If $s_0$ is a predicate, then for some $\xi$, we have $s_0({\bf \xi})^I = \true$ and $s_0({\bf \xi})^J = \false$ by definition of $J <^{\bf c} I$. If $s_0$ is a function, then for some $\xi$, we have $s_0({\bf \xi})^I = v$ and $s_0({\bf \xi})^J \ne v$.

Since $F$ is in Clark normal form, there must be a rule in $gr_I[F]$ of the form $B \rightarrow s_0({\bf \xi}^\diamond)$ if $s_0$ is a predicate ($B \rightarrow s_0({\bf \xi}^\diamond) = v$ if $s_0$ is a function) where $B$ may be $\top$. Further it must be that $I \models B$ since if not, $I$ would not be a model of $\comp[F; {\bf c}]$. Thus, the corresponding rule in $gr_I[F]^\mu{I}$ is $B^\mu{I} \rar s_0({\bf \xi}^\diamond)$ ($B^\mu{I} \rar s_0({\bf \xi}^\diamond) = v$ if $s_0$ is a function). 

Now there are two cases to consider:
\begin{itemize}
\item Case 1: $J \models B^\mu{I}$. In this case, $J \not \models B^\mu{I} \rightarrow s_0({\bf \xi}^\diamond)$ (or $J \not \models B^\mu{I} \rightarrow s_0({\bf \xi}^\diamond) = v$ if $s_0$ is a function) and so $J \not \models gr_I[F]^\mu{I}$.

\item Case 2: $J \not \models B^\mu{I}$. By Lemma~\ref{lem:reductbody}, there is a constant $d$ occurring strictly positively in $B$ that $I$ and $J$ disagree on. However, this means there is an edge from $s_0$ to $d$ and since $I$ and $J$ disagree on $d$, $d$ belongs to $S$ which contradicts the fact that $s_0$ was chosen so that it had no edge to any element in $S$. Thus this case cannot arise.
\end{itemize}

($\Rightarrow$) Assume $I \models \sm[F;{\bf c}]$.  
$F$ can be viewed as the conjunction of 
$\forall {\bf x}(H({\bf x})\ar G({\bf x}))$, where each $H$ is an atomic formula containing each intensional constant $c_i$.
It is sufficient to prove that $I\models\forall {\bf x}(H({\bf x})\rar G({\bf x}))$ for each such formula. Assume for the sake of contradiction that for some formula 
$\forall {\bf x}(H({\bf x})\rar G({\bf x}))$ whose $H$ contains an intensional constant $c$, 
$I\models H(\xi)$ and $I\not\models G(\xi)$ for some list $\xi$ of object names.

Consider an interpretation $J$ that differs from $I$ only in that $J\not\models H(\xi)$. 
($I \models \exists xy (x \neq y)$ means there are at least two elements in the universe so this is possible when $c$ is a function constant.)

\bi
\item Clearly, $J\models  (H(\xi)\ar G(\xi))^\mu{I}$ because $G(\xi)^\mu{I}=\bot$.

\item For other rules $H(\xi')\ar G(\xi')$ where $\xi'$ is a list of object names different from $\xi$, 
clearly, $J\models H(\xi')$ iff $I\models H(\xi')$. 
Since $G$ is negative on $\bC$ and $J<^\bC I$, by Lemma~\ref{lem:negative} we have $I\models G(\xi')$ iff $J\models G(\xi')^\mu{I}$. 
Since $I\models H(\xi')\ar G(\xi')$, it follows that $J\models (H(\xi')\ar G(\xi'))^\mu{I}$.

\item For all other rules $H'(\xi)\ar G'(\xi)$ whose $H'$ has an intensional constant different from $c$, we have $I\models H'(\xi)\ar G'(\xi)$. Since $H'(\xi)\ar G'(\xi)$ is negative on $\bC$ and $J<^\bC I$, by Lemma~\ref{lem:negative}, we have $J\models (H'(\xi)\ar G'(\xi))^\mu{I}$.
\ei

The presence of $J$ contradicts that $I\models\sm[F; {\bf c}]$.\qed

\BOC
\vspace{3cm}

%
Now for every rule $r$ in $F$ of the form $\forall {\bf x}(H({\bf x}) \leftarrow G({\bf x}))$, for each of the ground rules in $gr_I[F]$ corresponding to $r$ of the form $H(\xi^\diamond) \leftarrow G(\xi^\diamond)$ there are two cases:
\begin{itemize}
\item Case 1: $I \models G(\xi^\diamond)$.  Since $I \models F$, it must also be that $I \models H(\xi^\diamond)$. Thus, $I \models H(\xi^\diamond) \lrar G(\xi^\diamond)$.

\item Case 2: $I \not \models G(\xi^\diamond)$.
The corresponding rule in the reduct $gr_I[F]^\mu{I}$ is equivalent to 
\beq
  H(\xi^\diamond) \ar \bot. 
\eeq{h-bot}
However, since $F$ is in Clark normal form, $H(\xi^\diamond)$ appears in the head of no other rule. Thus, if $I \models H(\xi^\diamond)$, $I \not \models \sm[F;{\bf c}]$ since we can take $J < ^{\bf c} I$ ($I \models \exists xy (x \neq y)$ means there are at least two elements in the universe so this is possible) that differs from $I$ only in that $J\not\models H(\xi^\diamond)$ which will satisfy \eqref{h-bot}. Furthermore for other rules $r$ in $gr_I[F]^{\mu I}$, $r$ is negative on the intensional constant in $H$, thus by Lemma~\ref{lem:negative}, $J$ satisfies $r$ as well. Consequently, $J$ satisfies $F^\mu{I}$, which contradicts that $I\models\sm[F; {\bf c}]$. 
It then follows that $I \models H(\xi^\diamond) \leftrightarrow G(\xi^\diamond)$. \qed
\end{itemize}
\EOC

\BOC
\vspace{3cm}

Since $F$ is tight on ${\bf c}$, by the splitting theorem $\sm[F; {\bf c}]$ is equivalent to the conjunction of formulas of the form $\sm[\forall {\bf x}(H({\bf x})\ar G({\bf x})); c_i]$. It is sufficient to show that each $\sm[H({\bf x})\ar G({\bf x}); c_i]$ is equivalent to $\forall {\bf x}
(H({\bf x})\lrar G({\bf x})$.
Now for every rule $r$ in $F$ of the form $\forall {\bf x}(H({\bf x}) \leftarrow G({\bf x}))$, for each of the ground rules in $gr_I[F]$ corresponding to $r$ of the form $H(\xi^\diamond) \leftarrow G(\xi^\diamond)$ there are two cases:
\begin{itemize}
\item Case 1: $I \models G(\xi^\diamond)$.  Since $I \models F$, it must also be that $I \models H(\xi^\diamond)$. Thus, $I \models H(\xi^\diamond) \lrar G(\xi^\diamond)$.

\item Case 2: $I \not \models G(\xi^\diamond)$.
The corresponding rule in the reduct $gr_I[F]^\mu{I}$ is equivalent to 
$$H(\xi^\diamond) \ar \bot. $$
However, since $F$ is in Clark normal form, $H(\xi^\diamond)$ appears in the head of no other rule. Thus, if $I \models H(\xi^\diamond)$, $I \not \models \sm[F;{\bf c}]$ since we can take $J < ^{\bf c} I$ ($I \models \exists xy (x \neq y)$ means there are at least two elements in the universe so this is possible) that differs from $I$ only in that $J \not \models H(\xi^\diamond)$ which will satisfy $F^\mu{I}$. Thus, it must be that $I \not \models H(\xi^\diamond)$. It then follows that $I \models H(\xi^\diamond) \leftrightarrow G(\xi^\diamond)$. \qed

\end{itemize}
\EOC

\subsection{Proof of Theorem~\ref{thm:elim-p}}


\noindent{\textbf {Theorem~\ref{thm:elim-p} \optional{thm:elim-p}}}\ \ 
{\sl 
The set of formulas consisting of
\beq
\forall {\bf x} (f({\bf x})=1\lrar p({\bf x})),
\eeq{pelimassum}
and $\i{FC}_f$ entails 
$$\sm[F; p{\bf c}]\lrar\sm[F^p_f\land\i{DF}_f; f{\bf c}].$$
}

\proof
For any interpretation $I$ of signature $\sigma\supseteq \{f,p,{\bf c}\}$ satisfying (\ref{pelimassum}), it is clear that $I\models F$ iff $I\models F^p_f \land \i{DF}_f$ since $\i{DF}_f$ is a tautology and $F^p_f$ is equivalent to $F$ under (\ref{pelimassum}). Thus it only remains to be shown that 
$$I\models \exists \wh{p} \wh{\bf c} ((\wh{p} \wh{\bf c} < p{\bf c}) \land F^*(\wh{p}\wh{\bf c}))$$
iff 
$$I\models \exists \wh{f}\wh{\bf c} ((\wh{f} \wh{\bf c} < f{\bf c}) \land (F^p_f)^*(\wh{f}\wh{\bf c}) \land \i{DF}_f^*(\wh{f}\wh{\bf c})).$$

Let $\sigma'= \sigma\cup\{g,q,\bD\}$ be an extended signature such that $g,q,\bD$ are similar to $f,p,\bC$ respectively, and do not belong to $\sigma$.

\medskip
($\Rightarrow$) Assume $I\models \exists \wh{p} \wh{\bf c} ((\wh{p} \wh{\bf c} < p{\bf c}) \land F^*(\wh{p}\wh{\bf c}))$. 
This is equivalent to saying that there is an interpretation $J$ of $\sigma$ that agrees with $I$ on all constants other than $p$ and $\bC$ such that 
$\mathcal{I} = J^{p\bC}_{q\bD}\cup I$ of signature $\sigma'$ satisfies 
$(q{\bf d} < p{\bf c} \land F^*(q{\bf d}))$.

It is sufficient to show that there is an interpretation $K$ of $\sigma$ that agrees with $J$ on all constants other than $f$ such that 
$\mathcal{I'} = K^{f\bC}_{g\bD}\cup I$ of signature $\sigma'$ satisfies 
$(g{\bf d} < f{\bf c} \land (F^p_f)^*(g{\bf d}) \land \i{DF}_f^*(g{\bf d}))$.
We define the interpretation of $K$ on $f$ as follows:
\begin{displaymath}
f^K({\vec{\xi}}) = \left\{ 
\begin{array}{ll} 
1 & \text{ if } p^J({\vec{\xi}})=\true \\ 
0 & \text{ otherwise. }
\end{array}
\right.
\end{displaymath}

We now show $\mathcal{I'}\models g{\bf d} < f{\bf c}$:
\bi
\item Case 1: $\mathcal{I} \models (q=p)$.
Since $\mathcal{I} \models q{\bf d} < p{\bf c}$, by definition $\mathcal{I} \models {\bf d}^{pred}\leq {\bf c}^{pred}$ and $\mathcal{I} \models \neg (q{\bf d} = p{\bf c})$ and since in this case, $\mathcal{I} \models (q=p)$, it must be that $\mathcal{I} \models \neg ({\bf d} = {\bf c})$. From this, we conclude $\mathcal{I'} \models \neg (g{\bf d} = f{\bf c})$. Further, since $\mathcal{I'} \models {\bf d}^{pred}\leq {\bf c}^{pred}$, we conclude $\mathcal{I'} \models  g{\bf d} < f{\bf c}$.

\item Case 2: $\mathcal{I} \models \neg (q=p)$.
Since $\mathcal{I} \models q{\bf d} < p{\bf c}$, by definition, $\mathcal{I} \models {\bf d}^{pred}\leq {\bf c}^{pred}$ and $\mathcal{I} \models (q\leq p)$. Thus, since in this case $\mathcal{I} \models \neg (q=p)$, it must be that $\mathcal{I} \models \exists {\bf x} (p({\bf x}) \land \neg q({\bf x}))$. From the definition of $f^K$ and from (\ref{pelimassum}), this is equivalent to $\mathcal{I'} \models \exists {\bf x} (f({\bf x})=1 \land g({\bf x})=0)$. Thus, we conclude $\mathcal{I'} \models \neg (f=g)$ and since $\mathcal{I'} \models {\bf d}^{pred}\leq {\bf c}^{pred}$, we further conclude that $\mathcal{I'} \models g{\bf d} < f{\bf c}$.
\ei

We now show $\mathcal{I'} \models \i{DF}_f^*(g{\bf d})$:
\\Since $\mathcal{I} \models q{\bf d} < p{\bf c}$, by definition, $\mathcal{I} \models (q\leq p)$, or equivalently $\mathcal{I} \models \forall {\bf x}(q({\bf x})\rightarrow p({\bf x}))$ and by contraposition, $\mathcal{I} \models \forall {\bf x}(\neg p({\bf x}) \rightarrow \neg q({\bf x}))$. Finally, by (\ref{pelimassum}),$FC_f$, and the definition of $f^K$, $\mathcal{I'} \models \forall {\bf x}( f({\bf x})=0 \rightarrow g({\bf x})=0)$ or simply $\mathcal{I'} \models \i{DF}_f^*(g{\bf d})$.

We now show $\mathcal{I'} \models (F^p_f)^*(g{\bf d})$ by proving the following: 

\noindent
{\bf Claim:} $\mathcal{I} \models F^*(q{\bf d})$ iff $\mathcal{I'} \models (F^p_f)^*(g{\bf d})$.

The proof of the claim is by induction on the structure of $F$.
\bi
\item Case 1: $F$ is an atomic formula not containing $p$.
$F^p_f$ is exactly $F$ thus $F^*(q{\bf d})$ is exactly $(F^p_f)^*(g{\bf d})$ so certainly the claim holds.
\item Case 2: $F$ is $p({\bf t})$ where ${\bf t}$ contains an intensional function constant from ${\bf c}$.
$F^*(q{\bf d})$ is $p({\bf t}) \land q({\bf t'})$
where ${\bf t'}$ is the result of replacing all intensional functions from ${\bf c}$ occurring in ${\bf t}$ with the corresponding function from ${\bf d}$.
Since $F^p_f$ is $f({\bf t}) = 1$, formula $(F^p_f)^*(g{\bf d})$ is $f({\bf t}) = 1 \land g({\bf t'}) = 1$.
The claim follows from (\ref{pelimassum}) and the definition of $f^K$.

\item Case 3: $F$ is $p({\bf t})$ where ${\bf t}$ does not contain any intensional function constant from~${\bf c}$. 
$F^*(q{\bf d})$ is $q({\bf t})$.
Since $F^p_f$ is $f({\bf t}) = 1$, formula $(F^p_f)^*(g{\bf d})$ is $f({\bf t}) = 1 \land g({\bf t}) = 1$.
Since $\mathcal{I} \models (q \leq p)$, if $\mathcal{I} \models q({\bf t})$, then $\mathcal{I} \models p({\bf t})$. 
The claim follows from (\ref{pelimassum}) and the definition of $f^K$.

\item The other cases are straightforward from I.H. 


\ei


($\Leftarrow$) Assume $I \models \exists \wh{f}\wh{\bf c} ((\wh{f} \wh{\bf c} < f{\bf c}) \land (F^p_f)^*(\wh{f}\wh{\bf c}) \land \i{DF}_f^*(\wh{f}\wh{\bf c}))$. This is equivalent to saying that there is an interpretation $J$ of $\sigma$ that agrees with $I$ on all constants other than $f$ and $\bC$ such that ${\cal I} = J^{f\bC}_{g\bD}\cup I$ of signature $\sigma'$ satisfies 
$(g\bD < f\bC) \land (F^p_f)^*(f \bC) \land \i{DF}_f^*(f \bC)$.

It is sufficient to show that there is an interpretation $K$ of $\sigma$ that agrees with $J$ on all constants other than $p$ such that 
$\mathcal{I'} = K^{p\bC}_{q\bD}\cup I$ of signature $\sigma'$ satisfies 
$(q{\bf d} < p{\bf c} \land F^*(q{\bf d})$.
We define the interpretation of $K$ on $p$ as follows:
\begin{displaymath}
p^K({\vec{\xi}}) = \left\{ 
\begin{array}{lll} 
\true & &\text{ if } f^J({\vec{\xi}}) = 1 \\
\false & &\text{otherwise}.
\end{array}
\right.
\end{displaymath}

We now show $\mathcal{I'} \models q{\bf d} < p{\bf c}$:
\bi
\item Case 1: $\mathcal{I} \models (g=f)$.
By definition of $p^K$ and by (\ref{pelimassum}), in this case, $\mathcal{I} \models q=p$ and in particular, $\mathcal{I} \models q \leq p$. Since $\mathcal{I} \models g{\bf d} < f{\bf c}$, by definition $\mathcal{I} \models {\bf d}^{pred}\leq {\bf c}^{pred}$ and $\mathcal{I} \models \neg (g{\bf d} = f{\bf c})$ and since in this case, $\mathcal{I} \models (g=f)$, it must be that $\mathcal{I} \models \neg ({\bf d} = {\bf c})$. From this, we conclude $\mathcal{I'} \models \neg (q{\bf d} = p{\bf c})$. Further, since $\mathcal{I'}\models \bD^{pred}\le\bC^{pred}$, we conclude $\mathcal{I'} \models  q{\bf d} < p{\bf c}$.
\item Case 2: $\mathcal{I} \models \neg (g=f)$.
Since $\mathcal{I} \models \i{DF}_f^*(g{\bf d})$, it must be that $\mathcal{I} \models \forall {\bf x}(f({\bf x}) = 0 \rightarrow g({\bf x}) = 0)$. From this, we conclude by definition of $p^K$, $\i{FC}_f$ (note that $0 \neq 1$ is essential here) and (\ref{pelimassum}) that $\mathcal{I'} \models \forall {\bf x}(\neg p({\bf x}) \rightarrow \neg q({\bf x}))$. Equivalently, this is $\mathcal{I'} \models \forall {\bf x}(q({\bf x}) \rightarrow p({\bf x}))$ or simply $\mathcal{I'} \models q \leq p$.

Now, since $\mathcal{I} \models FC_f$, then $\mathcal{I} \models \forall {\bf x} (f({\bf x}) = 0 \lor f({\bf x}) = 1)$. Thus, for the assumption in this case that $\mathcal{I} \models \neg (g=f)$ to hold, it must be that $\mathcal{I} \models \exists {\bf x}(f({\bf x}) = 1 \land \neg (g({\bf x}) = 1))$. By defintion of $p^K$ and (\ref{pelimassum}), it follows that $\mathcal{I'} \models \exists {\bf x}(p({\bf x}) \land \neg q({\bf x}))$. Thus, since $\mathcal{I'} \models \neg (q=p)$, then $\mathcal{I'} \models \neg (q{\bf d}=p{\bf c})$. Also, since $\mathcal{I} \models g{\bf d} < f{\bf c}$, by definition $\mathcal{I'} \models {\bf d}^{pred}\leq {\bf c}^{pred}$, and thus we conclude that $\mathcal{I'} \models q{\bf d} < p{\bf c}$.
\ei


The proof of $\mathcal{I'} \models F^*(q{\bf d})$ is by induction similar to the proof of the claim above.
\qed
\BOC
by proving that $\mathcal{I'} \models (F^p_f)^*(g{\bf d})$ iff $\mathcal{I'} \models F^*(q{\bf d})$:
\ \\Case 1: $F$ is an atomic formula not containing $p$.
\\$F^p_f$ is exactly $F$ thus $F^*(q{\bf d})$ is exactly $(F^p_f)^*(g{\bf d})$ so certainly the claim holds.
\ \\Case 2: $F$ is $p({\bf t})$ where ${\bf t}$ contains an intensional function constant from ${\bf c}$.
\\$F^*(q{\bf d})$ is $p({\bf t}) \land q({\bf t'})$
\\where ${\bf t'}$ is the result of replacing all intensional functions from ${\bf c}$ occurring in ${\bf t}$ with the corresponding function from ${\bf d}$
\\$F^p_f$ is $f({\bf t}) = 1$.
\\$(F^p_f)^*(g{\bf d})$ is $f({\bf t}) = 1 \land g({\bf t'}) = 1$.
\\Since $\mathcal{I'} \models f({\bf t}) = 1 \land g({\bf t'}) = 1$, by definition of $q$ and (\ref{pelimassum}), $\mathcal{I'} \models p({\bf t}) \land q({\bf t'})$ and thus $\mathcal{I'} \models F^*(q{\bf d})$.

\ \\Case $3$: $F$ is $p({\bf t})$ where ${\bf t}$ does not contain any intensional function constant from ${\bf c}$.
\\$F^*(q{\bf d})$ is $q({\bf t})$.
\\$F^p_f$ is $f({\bf t}) = 1$.
\\$(F^p_f)^*(g{\bf d})$ is $f({\bf t}) = 1 \land g({\bf t}) = 1$.
\\By definition of $q$ and since $\mathcal{I'} \models f({\bf t}) = 1 \land g({\bf t}) = 1$, $\mathcal{I'} \models q({\bf t})$ and thus $\mathcal{I'} \models F^*(q{\bf d})$ in this case.

\ \\Case $4$: $F$ is $G \odot H$ where $\odot \in \{\land, \lor\}$.
\\By I.H. on $G$ and $H$.

\ \\Case $5$: $F$ is $G \rightarrow H$.
\\By I.H. on $G$ and $H$.

\ \\Case $6$: $F$ is $Q{\bf x}G({\bf x})$ where $Q \in \{\forall, \exists\}$.
\\By I.H. on $G$.\qed
\EOC


\subsection{Proof of Corollary~\ref{cor:elim-p2}}

For two interpretations $I$ of signature $\sigma_1$ and $J$ of signature $\sigma_2$, by $I \cup J$ we denote the interpretation of signature $\sigma_1 \cup \sigma_2$ and universe $|I| \cup |J|$ that interprets all symbols occurring only in $\sigma_1$ in the same way $I$ does and similarly for $\sigma_2$ and $J$. For symbols appearing in both $\sigma_1$ and $\sigma_2$, $I$ must interpret these the same as $J$ does, in which case $I \cup J$ also interprets the symbol in this way.

\noindent{\textbf {Corollary~\ref{cor:elim-p2} \optional{cor:elim-p2}}}\
\ 
{\sl
\begin{itemize}
\item[(a)] An interpretation $I$ of the signature of $F$ is a model  of~$\sm[F; p{\bf c}]$ iff $I^p_f$ is a model of~$\sm[F^p_f\land\i{DF}_f\land\i{FC}_f; f{\bf c}]$.
\item[(b)] An interpretation $J$ of the signature of $F^p_f$ is a model of~$\sm[F^p_f\land\i{DF}_f\land\i{FC}_f ; f{\bf c}]$ iff $J=I^p_f$ for some model $I$ of~$\sm[F; p{\bf c}]$.
\end{itemize}
}

\proof

\noindent
(a$\Rightarrow$) Assume $I$ of the signature of $F$ is a model of $\sm[F;p{\bf c}]$. By definition of~$I^p_f$, $I \cup I^p_f \models \forall {\bf x}(f({\bf x}) = 1 \leftrightarrow p({\bf x}))\land \i{FC}_f$. Since $I \models \sm[F;p{\bf c}]$, it must be that $I \cup I^p_f \models \sm[F;p{\bf c}]$ and further by Theorem~\ref{thm:elim-p}, 
$I \cup I^p_f \models \sm[F^p_f\land\i{DF}_f; f{\bf c}]$. By Theorem~\ref{thm:constraint}, 
we have $I \cup I^p_f \models \sm[F^p_f\land\i{DF}_f\land\i{FC}_f; f{\bf c}]$.
Finally, since the signature of $I$ does not contain $f$, we conclude $I^p_f \models \sm[F^p_f\land\i{DF}_f\land\i{FC}_f; f{\bf c}]$.

\smallskip\noindent
(a$\Leftarrow$) Assume $I^p_f$ is a model of $\sm[F^p_f\land\i{DF}_f\land\i{FC}_f; f{\bf c}]$.
By Theorem~\ref{thm:constraint}, $I^p_f$ is a model of $\sm[F^p_f\land\i{DF}_f; f{\bf c}]$.
 By definition of $I^p_f$, $I \cup I^p_f \models \forall {\bf x}(f({\bf x}) = 1 \leftrightarrow p({\bf x}))\land\i{FC}_f$. Since $I^p_f \models \sm[F^p_f\land\i{DF}_f; f{\bf c}]$, it must be that $I \cup I^p_f \models \sm[F^p_f\land\i{DF}_f; f{\bf c}]$ and further by Theorem~\ref{thm:elim-p}, $I \cup I^p_f \models \sm[F; p{\bf c}]$. Finally, since the signature of $I^p_f$ does not contain $p$, we conclude $I \models \sm[F; p{\bf c}]$.

\smallskip\noindent
(b$\Rightarrow$) Assume an interpretation $J$ of the signature of $F^p_f$ is a model of $\sm[F^p_f\land\i{DF}_f\land\i{FC}_f; f{\bf c}]$. Let $I = J^f_p$, where $J^f_p$ denotes the interpretation of the signature $F$ obtained from $J$ by replacing $f^J$ with the set $p^I$ that consists of the tuples $\langle \xi_1, \dots, \xi_n \rangle$ for all $\xi_1, \dots, \xi_n$ from the universe of $J$ such that $f^J(\xi_1, \dots, \xi_n) = 1$. By definition of $I$, $I \cup J \models \forall {\bf x}(f({\bf x}) = 1 \leftrightarrow p({\bf x}))$. Since $J \models \sm[F^p_f\land\i{FC}_f\land\i{DF}_f; f{\bf c}]$, it must be that $I \cup J \models \sm[F^p_f\land\i{DF}_f\land\i{FC}_f; f{\bf c}]$. Since $\i{FC}_f$ is comprised of constraints, by Theorem~\ref{thm:constraint}, $I \cup J \models \sm[F^p_f\land\i{DF}_f; f{\bf c}] \land\i{FC}_f$. In particular, $I \cup J \models \sm[F^p_f\land\i{DF}_f; f{\bf c}]$ and further by Theorem~\ref{thm:elim-p}, $I \cup J \models \sm[F; p{\bf c}]$. Finally, since the signature of $J$ does not contain $p$, we conclude $I \models \sm[F; p{\bf c}]$.

\smallskip\noindent
(b$\Leftarrow$) Take any $I$ such that $J = I^p_f$ and $I \models \sm[F; p{\bf c}]$. By definition of $I^p_f$, $I \cup J \models \forall {\bf x}(f({\bf x}) = 1 \leftrightarrow p({\bf x}))\land \i{FC}_f$. Since $I \models \sm[F;p{\bf c}]$, it must be that $I \cup J \models \sm[F;p{\bf c}]$ and further by Theorem~\ref{thm:elim-p}, 
$I \cup J \models \sm[F^p_f\land\i{DF}_f; f{\bf c}]$. Since the signature of $I$ does not contain $f$, we conclude $J \models \sm[F^p_f\land\i{DF}_f; f{\bf c}]$. Finally, since by definition of $I^p_f$, $J \models \i{FC}_f$, and since $\i{FC}_f$ is comprised of constraints, by Theorem~\ref{thm:constraint} we conclude $J \models \sm[F^p_f\land\i{DF}_f\land\i{FC}_f; f{\bf c}]$.
\qed

\BOCC
\cblu
\subsection{Proof of Theorem~\ref{thm:fi}}

\noindent{\textbf{Theorem~\ref{thm:fi} \optional{thm:fi}}}\
\ 
{\sl 
Let $\Pi$ be a program of signature $\sigma$.
\bi
\item If $X$ is a stable model of $\Pi$, then $I_X$ is a stable model of the functional image of $\Pi$.
\item If $I$ is a stable model of the functional image of $\Pi$, then $X_I$ is a stable model of $\Pi$.
\ei 
}
\bigskip

\proof Let ${\bf p}$ denote all of the atoms in $\sigma$ and let ${\bf f}$ denote all of the corresponding object constants in the signature of the functional image of $\Pi$. We first note that $I_X$ is the same as $I^{\bf p}_{\bf f}$. We also note that the added rules 
\[ 
   \{A\mvis 0\}^{\rm ch}
\]
and 
\[
\ba c
  0\ne 1, \\
  x=0\lor x=1. 
\ea
\]
 are precisely $\i{DF}_f\land\i{FC}_f$ when considering their first-order representation. Finally, we note then that the first-order representation of functional image of $\Pi$ is exactly $F^{\bf p}_{\bf f} \land \i{DF}_f\land\i{FC}_f$ where $F$ is the first-order representation of $\Pi$. Then, the claim follows from multiple applications of Corollary~\ref{cor:elim-p2} for each $p$ in ${\bf p}$ and the corresponding $f$ in ${\bf f}$. \qed

\cbla
\EOCC


\subsection{Proof of Theorem~\ref{thm:elim-f}}

\noindent{\bf Theorem~\ref{thm:elim-f} \optional{thm:elim-f}}\
\ 
{\sl
For any $f$-plain formula $F$, the set of formulas consisting of 
\beq
\forall {\bf x}y (p({\bf x},y)\lrar f({\bf x})=y)
\eeq{fpequiv} 
and
$\exists xy (x\ne y)$ entails 
$$\sm[F; f{\bf c}]\lrar\sm[F^f_p
; p{\bf c}].$$}

\proof
For any interpretation $I$ of signature $\sigma \supseteq \{f,p,{\bf c}\}$ satisfying (\ref{fpequiv}), it is clear that $I\models F$ iff $I\models F^f_p$ since $F^f_p$ is simply the result of replacing all $f({\bf x}) = y$ with $p({\bf x},y)$. Thus it only remains to be shown that 
$$I\models \exists \wh f \wh {\bf c} ((\wh f \wh{\bf c} < f{\bf c}) \land F^*(\wh f\wh{\bf c}))$$
iff
$$I \models \exists \wh p \wh{\bf c} ((\wh p \wh{\bf c} < p{\bf c}) \land (F^f_p)^*(\wh p\wh{\bf c})).$$

Let $\sigma'= \sigma\cup\{g,q,\bD\}$ be an extended signature such that $g,q,\bD$ are similar to $f,p,\bC$ respectively, and do not belong to $\sigma$.

($\Rightarrow$) 
Assume $I\models \exists \wh f \wh{\bf c} ((\wh f \wh{\bf c} < f{\bf c}) \land F^*(\wh f,\wh{\bf c}))$.
This is equivalent to saying that there is an interpretation $J$ of $\sigma$ that agrees with $I$ on all constants other than $f$ and $\bC$ such that $\mathcal{I} = J^{f\bC}_{g\bD}\cup I$ of signature $\sigma'$ satisfies $(g\bD<f\bC)\land F^*(g\bD)$.

It is sufficient to show that there is an interpretation $K$ of $\sigma$ that agrees with $J$ on all constants other than $p$ such that 
$\mathcal{I'} = K^{p\bC}_{q\bD}\cup I$ of signature $\sigma'$ satisfies 
$(q{\bf d} < p{\bf c}) \land (F^f_p)^*(q\bD)$.
We define the interpretation of $K$ on $p$ as follows: 
\begin{displaymath}
p^K({\vec{\xi}},\xi') = \left\{ 
\begin{array}{lll} 
\true & &\text{ if } \mathcal{I} \models f({\vec{\xi}}) = \xi ' \land g({\vec{\xi}}) = \xi ' \\ 
\false & &\text{otherwise}.
\end{array}
\right.
\end{displaymath}

We first show that if $\mathcal{I} \models (g{\bf d} < f{\bf c})$ then $\mathcal{I'} \models (q{\bf d} < p{\bf c})$:
\\Observe that from the definition of $p^K$, it follows that $\mathcal{I} \models \forall {\bf x}y (q({\bf x},y) \rightarrow f({\bf x}) = y)$ and from (\ref{fpequiv}), this is equivalent to $\forall {\bf x}y (q({\bf x},y) \rightarrow p({\bf x}, y))$ or simply $q \leq p$. Thus, since $\mathcal{I'} \models {\bf{d}}^{pred} \leq {\bf{c}}^{pred}$, we have $\mathcal{I'} \models q{\bf{d}}^{pred} \leq p{\bf{c}}^{pred}$.

\bi
\item Case 1: $\mathcal{I} \models \forall {\bf x}y(f({\bf x}) = y \leftrightarrow g({\bf x}) = y)$.
\\In this case it then must be the case that $\mathcal{I} \models {\bf d} \neq {\bf c}$. Thus it follows that 
$\mathcal{I'}\models q{\bf d} \neq p{\bf c}$. Consequently, we conclude that
\[
\mathcal{I'} \models (q{\bf{d}}^{pred} \leq p{\bf{c}}^{pred}) \land q{\bf d} \neq p{\bf c}
\]
or simply, $\mathcal{I'} \models (q{\bf d} < p{\bf c})$.

\item Case 2: $\mathcal{I} \models \neg \forall {\bf x}y(f({\bf x}) = y \leftrightarrow g({\bf x}) = y)$.
\\In this case it then must be the case that for some ${\bf t}$ and $c$ that $\mathcal{I} \models f({\bf t}) = c \land g({\bf t}) \neq c$. By the definition of $p^K$, this means that $q({\bf t},c)^{\mathcal{I'}} = \false$ but by (\ref{fpequiv}), $p({\bf t},c)^{\mathcal{I'}} = \true$. Therefore, $\mathcal{I'} \models p \neq q$ and thus $\mathcal{I'} \models q{\bf d} \neq p{\bf c}$. Consequently, we conclude
\[
\mathcal{I'} \models (q{\bf{d}}^{pred} \leq p{\bf{c}}^{pred}) \land q{\bf d} \neq p{\bf c}
\]
or simply, $\mathcal{I'} \models (q{\bf d} < p{\bf c})$.
\ei

We now show that $\mathcal{I} \models (F^f_p)^*(q{\bf d})$ by proving the following:

\noindent
{\bf Claim:} $\mathcal{I} \models F^*(g{\bf d})$  iff $\mathcal{I'}\models (F^f_p)^*(q{\bf d})$

The proof of the claim is by induction on the structure of $F$.
\bi
\item Case 1: $F$ is an atomic formula not containing $f$. 
$F^f_p$ is exactly $F$ thus $F^*(g{\bf d})$ is exactly $(F^f_p)^*(q{\bf d})$ so certainly the claim holds.

\item Case 2: $F$ is $f({\bf t}) = t_1$.
$F^*(g{\bf d})$ is $f({\bf t}) = t_1 \land g({\bf t}) = t_1$.
$F^f_p$ is $p({\bf t},t_1)$ and $(F^f_p)^*(q{\bf d})$ is $q({\bf t},t_1)$.
By the definition of $p^K$, it is clear that $\mathcal{I}\models f({\bf t})=t_1\land g({\bf t})=t_1$ iff $\mathcal{I'}\models q({\bf t},t_1)$, so certainly the claim holds.

\item The other cases are straightforward from I.H.
\ei

($\Leftarrow$) Assume $\mathcal{I}\models\exists \wh p \wh{\bf c} ((\wh p \wh{\bf c} < p{\bf c}) \land (F^f_p)^*(\wh p\wh{\bf c}))$. 
This is equivalent to saying that there is an interpretation $J$ of $\sigma$ that agrees with $I$ on all constants other than $p$ and $\bC$ such that $\mathcal{I}=J^{p\bC}_{q\bD}\cup I$ of signature $\sigma'$ satisfies $(q\bD<p\bC)\land (F^f_p)^*(q\bD)$.

It is sufficient to show that there is an interpretation $K$ of $\sigma$ that agrees with $J$ on all constants other than $f$ such that $\mathcal{I'}=K^{f\bC}_{g\bD}\cup I$ of signature $\sigma'$ satisfies $(g\bD<f\bC)\land F^*(g\bD)$. 
We define the interpretation of $K$ on $f$ as follows:
\begin{displaymath}
f^K({\vec{\xi}}) = \left\{ 
\begin{array}{lll} 
\xi ' & &\text{ if } \mathcal{I} \models p({\vec{\xi}},\xi ') \land q({\vec{\xi}},\xi ') \\ 
\xi '' &  &\text{ if } \mathcal{I} \models p({\vec{\xi}},\xi ') \land \neg q({\vec{\xi}},\xi ') \text{ where } \xi ' \neq \xi ''.
\end{array}
\right.
\end{displaymath}

Note that the assumption that there are at least two elements in the universe is essential to this definition. This definition is sound due to $(\ref{fpequiv})$ entailing $\forall {\vec{\xi}}\exists \xi '(p(\vec{\xi},\xi '))$.

We first show if $\mathcal{I}\models (q{\bf d} < p{\bf c})$ then $\mathcal{I'} \models (g{\bf d} < f{\bf c})$:
\\Observe that $\mathcal{I} \models (q{\bf d} < p{\bf c})$ by definition entails 
$\mathcal{I} \models (q{\bf d}^{pred} \leq p{\bf c}^{pred})$ and further by definition, 
$\mathcal{I} \models ({\bf d}^{pred} \leq {\bf c}^{pred})$ and then since $f$ and $g$ are not predicates, $\mathcal{I'} \models ((g{\bf d})^{pred} \leq (f{\bf c})^{pred})$.

\bi
\item Case 1: $\mathcal{I} \models \forall {\bf x}y(p({\bf x}, y) \leftrightarrow q({\bf x}, y))$.
In this case, $\mathcal{I} \models  (p = q)$ so for it to be the case that $\mathcal{I} \models (q{\bf d} < p{\bf c})$, it must be that $\mathcal{I} \models \neg({\bf c} = {\bf d})$. It then follows that $\mathcal{I'} \models \neg(f{\bf c} = g{\bf d})$. Consequently, in this case, $\mathcal{I'} \models ((g{\bf d})^{pred} \leq (f{\bf c})^{pred}) \land \neg(f{\bf c} = g{\bf d})$ or simply $\mathcal{I'} \models (g{\bf d} < f{\bf c})$.

\item Case 2: $\mathcal{I} \models \neg \forall {\bf x}y(p({\bf x}, y) \leftrightarrow q({\bf x}, y))$.
In this case, since $\mathcal{I} \models (q \leq p)$, then it follows that $\exists {\bf x}y(p({\bf x},y) \land \neg q({\bf x},y))$. It follows from the definition of $p^K$ that $\mathcal{I'} \models \exists {\bf x}yz((p({\bf x}, y) \leftrightarrow g({\bf x}) = z)\land y \neq z)$ and then from (\ref{fpequiv}), it follows that $\mathcal{I'} \models \exists {\bf x}yz((f({\bf x}) = y \leftrightarrow g({\bf x}) = z)\land y \neq z)$ or simply $\mathcal{I'} \models f \neq g$. It then follows that $\mathcal{I'} \models \neg(f{\bf c} = g{\bf d})$. Consequently, in this case $\mathcal{I'} \models ((g{\bf d})^{pred} \leq (f{\bf c})^{pred}) \land \neg(f{\bf c} = g{\bf d})$ or simply $\mathcal{I'} \models (g{\bf d} < f{\bf c})$.
\ei

Next, the proof of $\mathcal{I'} \models F^*(g{\bf d})$ is by induction similar to the proof of the claim above. 
\BOC
\ \\Case 1: $F$ is an atomic formula not containing $f$. 
\\$F^f_p$ is exactly $F$ thus $F^*(g{\bf d})$ is exactly $(F^f_p)^*(q{\bf d})$ so certainly the claim holds.

\ \\Case 2: $F$ is $f({\bf t}) = c$.
\\$F^*(g{\bf d})$ is $f({\bf t}) = c \land g({\bf t}) = c$.
\\$F^f_p$ is $p({\bf t},c)$.
\\$(F^f_p)^*(q{\bf d})$ is $q({\bf t},c)$.
\\Since $\mathcal{I} \models q({\bf t},c)$, then $\mathcal{I} \models p({\bf t},c)$ since it is assumed that $\mathcal{I} \models (q \leq p)$. From (\ref{fpequiv}), it follows that $\mathcal{I} \models f({\bf t})=c$ and from the definition of $g$, it follows that $\mathcal{I} \models g({\bf t})=c$. 

\ \\Case $3$: $F$ is $G \odot H$ where $\odot \in \{\land, \lor\}$.
\\By I.H. on $G$ and $H$.

\ \\Case $4$: $F$ is $G \rightarrow H$.
\\By I.H. on $G$ and $H$.

\ \\Case $5$: $F$ is $Q{\bf x}G({\bf x})$ where $Q \in \{\forall, \exists\}$.
\\By I.H. on $G$.\qed
\EOC

\subsection{Proof of Corollary~\ref{cor:elim-f2}}

\noindent{\textbf {Corollary~\ref{cor:elim-f2}\optional{cor:elim-f2}}}\ \ 
{\sl
Let $F$ be an $f$-plain sentence. 
\begin{itemize}
\item[(a)] An interpretation $I$ of the signature of $F$ that satisfies
$\exists xy (x\ne y)$ is a model of $\sm[F;f{\bf c}]$ iff $I^f_p$ is a
model of $\sm[F^f_p\land\i{UEC}_p;\ p{\bf c}]$.
\item[(b)] An interpretation $J$ of the signature of $F^f_p$ that satisfies
$\exists xy (x\ne y)$ is a model of
$\sm[F^f_p\land\i{UEC}_p;\ p{\bf c}]$ iff $J = I^f_p$ for some model $I$ of 
$\sm[F;f{\bf c}]$.
\end{itemize}
}
\medskip

\proof

\noindent
(a$\Rightarrow$) Assume $I \models \fsm[F;f{\bf c}] \land \exists xy(x \neq y)$. Since $I \models \exists xy(x \neq y)$, $I \cup I^f_p \models \exists xy(x \neq y)$ since by definition of $I^f_p$, $I$ and $I^f_p$ share the same universe. 

By definition of $I^f_p$, $I \cup I^f_p \models (\ref{fpequiv})$. 
Since $I\models\fsm[F;f{\bf c}]$, we have $I\cup I^f_p\models \fsm[F;f{\bf c}]$ and by Theorem~\ref{thm:elim-f}, we have $I\cup I^f_p\models\fsm[F^f_p ;p{\bf c}]$. It's clear that $I\models\i{UEC}_p$, so by Theorem~\ref{thm:constraint}, we have $I\cup I^f_p\models\fsm[F^f_p\land\i{UEC}_p;p\bC]$. Since the signature of $I$ does not contain $f$, we conclude $I^f_p\models\sm[F^f_p\land\i{UEC}_p; p\bC]$.

\smallskip\noindent
(a$\Leftarrow$) Assume $I \models \exists xy(x \neq y)$ and $I^f_p \models \fsm[F^f_p\land\i{UEC}_p ;p{\bf c}]$. By Theorem~\ref{thm:constraint}, $I^f_p\models\fsm[F^f_p; p\bC]$.
Since $I\models\exists xy(x \neq y)$, we have $I \cup I^f_p \models \exists xy(x \neq y)$ since by definition of $I^f_p$, $I$ and $I^f_p$ share the same universe. 

By definition of $I^f_p$, $I \cup I^f_p \models (\ref{fpequiv})$. 
Since $I^f_p \models \fsm[F^f_p ;p{\bf c}]$, we have $I \cup I^f_p \models \fsm[F^f_p ;p{\bf c}]$ and by Theorem~\ref{thm:elim-f}, we have $I \cup I^f_p \models \fsm[F;f{\bf c}]$. 
Since the signature of $I^f_p$ does contain $f$, we conclude $I \models \fsm[F;f{\bf c}]$.

\smallskip\noindent
(b$\Rightarrow$) Assume $J \models \exists xy(x \neq y)$ and $J \models \fsm[F^f_p \land \i{UEC}_p;p{\bf c}]$. Let $I = J^p_f$ where $J^p_f$ denotes the interpretation of the signature of $F$ obtained from $J$ by replacing the set $p^J$ with the function $f^I$ such that $f^I(\xi_1,\dots,\xi_k) = \xi_{k+1}$ for all tuples $\langle \xi_1,\dots,\xi_k,\xi_{k+1}\rangle$ in $p^J$. This is a valid definition of a function since we assume $J \models \fsm[F^f_p \land \i{UEC}_p;p{\bf c}]$, from which we obtain by Theorem~\ref{thm:constraint} that $J \models \fsm[F^f_p ;p{\bf c}] \land\i{UEC}_p$ and specifically, $J \models \i{UEC}_p$. Clearly, $J = I^f_p$ so it only remains to be shown that $I \models \fsm[F;f{\bf c}]$.

Since $I$ and $J$ have the same universe and $J \models \exists xy(x \neq y)$, it follows that $I \cup J \models \exists xy(x \neq y)$. Also by the definition of $J^p_f$, we have $I \cup J \models (\ref{fpequiv})$. Thus by Theorem~\ref{thm:elim-f}, $I \cup J \models \fsm[F;f{\bf c}] \leftrightarrow \fsm[F^f_p 
;p{\bf c}]$. 

Since we assume $J \models \fsm[F^f_p 
;p{\bf c}]$, it is the case that $I \cup J \models \fsm[F^f_p 
;p{\bf c}]$ and thus it must be the case that $I \cup J \models \fsm[F;f{\bf c}]$. Now since the signature of $J$ does not contain $f$, we conclude $I \models \fsm[F;f{\bf c}]$. 

\smallskip\noindent
(b$\Leftarrow$)Take any $I$ such that $J = I^f_p$ and $I \models \fsm[F;f{\bf c}]$. Since $J \models \exists xy(x \neq y)$ and $I$ and $J$ share the same universe, $I \cup J \models \exists xy(x \neq y)$. By definition of $J = I^f_p$, $I \cup J \models (\ref{fpequiv})$. Thus by Theorem~\ref{thm:elim-f}, $I \cup J \models \fsm[F;f{\bf c}] \leftrightarrow \fsm[F^f_p 
;p{\bf c}]$. 

Since we assume $I \models \fsm[F;f{\bf c}]$, it is the case that $I \cup J \models \fsm[F;f{\bf c}]$ and thus it must be the case that $I \cup J \models \fsm[F^f_p
;p{\bf c}]$. Further, due to the nature of functions, (\ref{fpequiv}) entails $\i{UEC}_p$ so $I \cup J \models \i{UEC}_p$. However since the signature of $I$ does not contain $p$, we conclude $J \models \fsm[F^f_p ;p{\bf c}]\land \i{UEC}_p$ and since $\i{UEC}_p$ is comprised of constraints only, by Theorem~\ref{thm:constraint} $J \models \fsm[F^f_p \land \i{UEC}_p;p{\bf c}]$. \qed

\subsection{Proof of Theorem~\ref{thm:head-cplain}}

\noindent{\textbf {Theorem~\ref{thm:head-cplain}\optional{thm:head-cplain}}}\ \ {\sl 
For any head-$\bC$-plain sentence $F$ that is tight on $\bC$ and any
interpretation $I$ satisfying $\exists xy(x \neq y)$, 
we have $I\models \sm[F;\bC]$ iff $I\models\sm[\i{UF}_\bC(F);\bC]$.
}

\proof
It is easy to check that the completion of $\i{UF}_\bC(F)$ relative to
$\bC$ is equivalent to the completion of $F$ relative to $\bC$. By
Theorem \ref{thm:completion}, we conclude that
$\sm[\i{UF}_\bC(F); \bC]$ is equivalent to $\sm[F;\bC]$.
\qed


\subsection{Proof of Theorem~\ref{thm:nomodular}}

For any formula $F$ containing object constants $f$ and $g$, we call it {\em ${fg}$-indistinguishable} if every occurrence of $f$ and $g$ in $F$ is in a subformula of the form $(f=t)\wedge (g=t)$ that is ${fg}$-plain. For any interpretations $I$ and $J$ of $F$, we say $I$ and $J$ satisfy the relation $R(I, J)$ if
\begin{itemize}
\item $|I|=|J|$,
\item $I(f) \ne I(g)$,
\item $J(f) \ne J(g)$, and
\item for all symbols $c$ other than $f$ and $g$, $I(c)=J(c)$.
\end{itemize}

\begin{lemma}\label{lem:fg-indistinguishable}
If a formula $F$ is $fg$-indistinguishable, then for any interpretations $I$ and $J$ such that $R(I,J)$, $F^I=F^J$.
\end{lemma}

\begin{proof}
Notice that any $fg$-indistinguishable formula is built on atomic formulas not containing $f$ and $g$, and formula of the form $(f=t)\land (g=t)$, using propositional connectives and quantifiers. The proof is by induction on such formulas. 

\BOC
\begin{itemize}
\item $F$ is an atom (or $\bot$ or $\top$) that contains neither $f$ nor $g$. Clear.

\item $F$ is $(f=t')\and (g=t')$ for some ground $t'$. Clear since $I(f)\ne I(g)$ and $J(f)\ne J(g)$, it must be that $F^I=F^J=\false$.

\item $F$ is $\neg G$, where $G$ is ${fg}$-indistinguishable. For any $I$ and $J$ satisfying $R(I, J)$, by I.H., $G^I=G^J$, so $F^I=F^J$.

\item $F$ is $G\odot H$, where $G$ and $H$ are both ${fg}$-indistinguishable and $\odot$ is $\wedge$, $\vee$, or $\rightarrow$. For any $I$ and $J$ satisfying $R(I, J)$, by I.H, $G^I=G^J$ and $H^I=H^J$, so $F^I=F^J$.

\item $F$ is $\exists x G(x)$, where $G(x)$ is \textbf{f}-indistinguishable. Suppose $\xi_1, \dots, \xi_n$ are symbols representing elements in $|I|$. Then $F^I=(G(\xi_1)\vee\dots\vee G(\xi_n))^I$. Since $|I|=|J|$, $F^J=(G(\xi_1)\vee\dots\vee G(\xi_n))^J$. By I.H., for each $1 \leq i \leq n$, $G(\xi_i)^I=G(\xi_i)^J$. So $F^I=F^J$.

\item $F$ is $\forall x G(x)$, where $G(x)$ is \textbf{f}-indistinguishable. Suppose $\xi_1, \dots, \xi_n$ are symbols representing elements in $|I|$. Then $F^I=(G(\xi_1)\wedge\dots\wedge G(\xi_n))^I$. Since $|I|=|J|$, $F^J=(G(\xi_1)\wedge\dots\wedge G(\xi_n))^J$. By I.H., for each $1 \leq i \leq n$, $G(\xi_i)^I=G(\xi_i)^J$. So $F^I=F^J$.
\end{itemize}
\EOC
\end{proof}

\bigskip
\noindent{\textbf {Theorem~\ref{thm:nomodular}}}\ \ 
{\sl 
For any set $\bC$ of constants, there is no strongly equivalent transformation that turns an arbitrary sentence into a $\bC$-plain sentence.
}\medskip

\proof 
The proof follows from the claim. 

{\bf Claim: } 
There is no ${f}$-plain formula that is strongly equivalent to 
$p(f)\land p(1)\land p(2)\land\neg p(3)$. 

Let $F$ be  $p(f)\land p(1)\land p(2)\land \neg p(3)$. Then $F^*(g)$ is $p(f)\land p(g)\land p(1)\land p(2)\land\neg p(3)$.
Let $I = \{p(1),p(2), f\mvis 1, g\mvis 2\}$ and $J = \{p(1), p(2), f\mvis 1, g\mvis 3\}$ (numbers are interpreted as themselves). It is easy to check that $I\models F^*(g)$ and $J\not\models F^*(g)$.

Assume for the sake of contradiction that there is a ${f}$-plain formula $G$ that is strongly equivalent to $F$. Since $G$ is $f$-plain, $G^*(g)$ is $fg$-indistinguishable. Since $R(I,J)$ holds, by Lemma~\ref{lem:fg-indistinguishable}, $I\models G^*(g)$ iff $J\models G^*(g)$, but this contradicts Theorem~\ref{thm:strong}.
\qed


\subsection{Proof of Theorem~\ref{thm:cm-fsm}}
\noindent{\textbf {Theorem~\ref{thm:cm-fsm} \optional{thm:cm-fsm}}}\
\ 
{\sl 
For any definite causal theory $T$, 
$I\models\cm[T; \bF]$ iff $I\models\sm[\i{Tr}(T); \bF]$.
}\medskip

\proof
Assume that, without loss of generality, the rules
(\ref{definite2})--(\ref{definite3}) have no free variables. 
It is sufficient to prove that under the assumption that $I$ satisfies
$T$, for every rule~(\ref{definite2}), $J^\bF_\bG\cup I$ satisfies 
\[ 
   B\, \rar\, g({\bf t})\mvis t_1
\]
iff
$J^\bF_\bG\cup I$ satisfies 
\[ 
   (\neg\neg B)^*(\bG)\, \rar\, 
        g({\bf t})\mvis t_1\land f({\bf t})\mvis t_1.
\]
Indeed, this is true since $B$ is equivalent to $(\neg\neg
B)^*(\bG)$ (Lemma~\ref{lem:neg}), and $I$ satisfies~$T$.
\qed
\cbla

\subsection{Proof of Theorem~\ref{thm:if-fsm}}

\noindent{\textbf {Theorem~\ref{thm:if-fsm} \optional{thm:if-fsm}}}\
{\sl 
$I\models\fsm[T; \bF]$ iff $I\models\lif[T; \bF]$.
}

\BOC
Let $T$ be an IF-program whose rules have the form 
\beq
  f({\bf t})=t_1 \ar\neg\neg B
\eeq
(above is (\ref{if-fsm-r})) where $f$ is an intensional function constant, ${\bf t}$ and $t_1$ do
not contain intensional function constants, and $B$ is an arbitrary
formula. We identify $T$ with the corresponding first-order formula.

\EOC

\proof 
We wish to show that 
$I \models T \land \neg\exists \wh{\bf f}(\wh{\bf f} < {\bf f} \land F^*(\wh{\bf f}))$ 
iff 
$I \models T \land \neg\exists \wh{\bf f}(\wh{\bf f} \neq {\bf f} \land F^\dia(\wh{\bf f}))$. 
The first conjunctive terms are identical and if $I \not \models T$ then the claim holds. 

Let us assume then, that $I \models T$. By definition, $\wh{\bf f} < {\bf f}$ is equivalent to $\wh{\bf f} \neq {\bf f}$. What remains to be shown is the correspondence between $F^*(\wh{\bf f})$ and $F^\dia(\wh{\bf f})$. 

Consider any list of functions ${\bf g}$ of the same length as ${\bf f}$. Let $\mathcal{I} =  J^{\bf f}_{\bf g}\cup I$ be an interpretation of an extended signature $\sigma' = \sigma \cup {\bf g}$ where $J$ is an interpretation of $\sigma$ and $J$ and $I$ agree on functions not belonging to ${\bf f}$.

Consider any rule $f({\bf t}) = t_1 \leftarrow \neg \neg B$ from $T$. The corresponding rule in 
$F^*({\bf g})$ is equivalent to 
$$f({\bf t}) = t_1 \land g({\bf t}) = t_1 \leftarrow B.$$
The corresponding rule in 
$F^\dia({\bf g})$ is equivalent to 
$$g({\bf t}) = t_1  \leftarrow B.$$
Now we consider cases
\begin{itemize}
\item $I \not \models B$. Clearly, both versions of the rule are vacuously satisfied by $\mathcal{I}$.
\item $I \models B$. Then, since $I \models T$ it must be that $I \models f({\bf t}) = t_1$ and so
the corresponding rule in 
$F^*({\bf g})$ is further equivalent to 
$$g({\bf t}) = t_1 \leftarrow B$$
which is equivalent to the corresponding rule in $F^\dia({\bf g})$ and so certainly $\mathcal{I}$ satisfies both corresponding rules or neither.
\end{itemize}
Thus, $\mathcal{I} \models F^*({\bf g})$ iff $\mathcal{I} \models F^\dia({\bf g})$ and so the claim holds. \qed


\subsection{Proof of Theorem~\ref{thm:ms2us}}

\begin{lemma}\label{lem:unisort1}\optional{lem:unisort1}
Given a formula $F$ of many-sorted signature $\sigma$ and an interpretation $I$ of $\sigma$, $I \models gr_I[F]$ iff $I^{ns} \models gr_{I^{ns}}[F^{ns}]$.
\end{lemma}

\proof By induction on the structure of $F$. \qed

%
\BOCC
\begin{itemize}
\item $F$ is $p({\bf t})$ where each $t_i$ in ${\bf t}$ is comprised of ground terms from the extended signature $\sigma^I$. $gr_I[F]$ is also $p({\bf t})$.
\\$F^{ns}$ is $p({\bf t})$. $gr_{I^{ns}}[F^{ns}]$ is also $p({\bf t})$.
By the definition of $I^{ns}$, $p({\bf t})^I = p({\bf t})^{I^{ns}}$ since ${\bf t}$ must be comprised of terms from the corresponding argument sorts of $p$ and so the claim holds.

\item $F$ is $t_1 = t_2$ where each $t_i$ is comprised of ground terms from the extended signature $\sigma^I$. $gr_I[F]$ is also $t_1 = t_2$ .
$F^{ns}$ is $t_1 = t_2$. $gr_{I^{ns}}[F^{ns}]$ is also $t_1 = t_2$.
By the definition of $I^{ns}$, $t_1^I = t_1^{I^{ns}}$ and $t_2^I = t_2^{I^{ns}}$ since the subterms of $t_1$ and $t_2$ must be comprised of terms from the corresponding argument sorts and so the claim holds.

\item $F$ is $G \odot H$ where $\odot \in \{\land, \lor, \rightarrow\}$. $gr_I[F]$ is $gr_I[G] \odot gr_I[H]$.
$F^{ns}$ is $G^{ns} \odot H^{ns}$. $gr_{I^{ns}}[F^{ns}]$ is $gr_{I^{ns}}[G^{ns}] \odot gr_{I^{ns}}[H^{ns}]$ so the claim follows by induction on $G$ and $H$. 

\item $F$ is $\exists x G(x)$. $gr_I[F]$ is $\{gr_I[G(\xi^\dia)] : \xi \in |I|^s\}^\lor$ where $s$ is the sort of $x$.
\\$F^{ns}$ is $\exists y (G(y)^{ns} \land {\sort s}(y))$. $gr_{I^{ns}}[F^{ns}]$ is $\{gr_{I^{ns}}[G(\xi^\dia)^{ns}] \land {\sort s}(\xi^\dia) : \xi \in |I^{ns}|\}^\lor$.
\\($\Rightarrow$) Assume $I \models gr_I[F]$. That is, assume there is some
$\xi \in |I|^s$ where $s$ is the sort of $x$ such that $I \models gr_I[G(\xi^\dia)]$. By
definition of $I^{ns}$, since $\xi \in |I|^s$, then $I^{ns} \models {\sort s}(\xi^\dia)$. 
So then, the claim follows by I.H. on $G(\xi^\dia)$.

\ \\($\Leftarrow$) Assume $I^{ns} \models gr_{I^{ns}}[F^{ns}]$. That is, 
assume there is some $\xi \in |I^{ns}|$ such that $I^{ns} \models gr_{I^{ns}}[G(\xi^\dia)^{ns}] \land {\sort s}(\xi^\dia)$. By definition of $I^{ns}$, since $I^{ns} \models {\sort s}(\xi^\dia)$, then $\xi \in |I|^s$. Then, the claim follows by I.H. on $G(\xi^\dia)$.

\item $F$ is $\forall x G(x)$. $gr_I[F]$ is $\{gr_I[G(\xi^\dia)] : \xi \in |I|^s\}^\land$ where $s$ is the sort of $x$.
\\$F^{ns}$ is $\forall y ({\sort s}(y) \rightarrow G(y)^{ns})$. $gr_{I^{ns}}[F^{ns}]$ is $\{{\sort s}(\xi^\dia) \rightarrow gr_{I^{ns}}[G(\xi^\dia)^{ns}] : \xi \in |I^{ns}|\}^\land$.
\\($\Rightarrow$) Assume $I \models gr_I[F]$. That is, for every
$\xi \in |I|^s$ where $s$ is the sort of $x$, assume that $I \models gr_I[G(\xi^\dia)]$. Note that for every $\xi \in |I^{ns}|$ such that $I^{ns} \not \models {\sort s}(\xi^\dia)$, we have that $I^{ns}$ vacuously satisfies ${\sort s}(\xi^\dia) \rightarrow gr_{I^{ns}}[G(\xi^\dia)^{ns}]$. By definition of $I^{ns}$, since $\xi \in |I|^s$ iff $I^{ns} \models {\sort s}(\xi^\dia)$
 the claim follows by I.H. on $G(\xi^\dia)$ for every $\xi \in |I|^s$.

\ \\($\Leftarrow$) Assume $I^{ns} \models gr_{I^{ns}}[F^{ns}]$. That is, 
assume for every $\xi \in |I^{ns}|$ that $I^{ns} \models {\sort s}(\xi^\dia) \rightarrow gr_{I^{ns}}[G(\xi^\dia)^{ns}]$. This means that for every $\xi$ such that $I^{ns} \models {\sort s}(\xi^\dia)$, it must be that $I^{ns} \models gr_{I^{ns}}[G(\xi^\dia)^{ns}]$.

Now, by definition of $I^{ns}$, for any $\xi$ such that $I^{ns} \models {\sort s}(\xi^\dia)$, we have that $\xi \in |I|^s$. So then, the claim follows by I.H. on $G(\xi^\dia)$ for every $\xi \in |I|^s$. 
\end{itemize}
\qed
\EOCC

\begin{lemma}\label{lem:unisort2}\optional{lem:unisort2}
Given a formula $F$ of many-sorted signature $\sigma$, interpretations $I$ and $J$ of $\sigma$ and an interpretation $K$ of $\sigma^{ns}$ such that 
\begin{itemize}
\item for every sort $s$ in $\sigma$, $|I|^s = |J|^s = s^K$,
\item for every predicate and function constant $c$ and for every tuple ${\bfxi}$ composed of elements from $|I^{ns}|$ such that $\xi_i \in |I|^{args_i}$  for every $\xi_i \in {\bfxi}$ , where $args_i$ is the $i$-th argument sort of $c$, we have 
$c({\bfxi})^K = c({\bfxi})^J$,
\item for every predicate and function constant $c$ and for every tuple ${\bfxi}$ composed of elements from $|I^{ns}|$ such that $\xi_i \notin |I|^{args_i}$ for some $\xi_i \in |I|^{args_i}$, where $args_i$ is the $i$-th argument sort of $c$, we have $c({\bfxi})^K = c({\bfxi})^{I^{ns}}$,
\end{itemize}
$J$ is a model of $gr_I[F]^{\mu I}$ iff $K$ is a model of $gr_{I^{ns}}[F^{ns}]^{\mu I^{ns}}$.
\end{lemma}

\proof By induction on the structure of~$F$. \qed

\BOCC
\begin{itemize}
\item $F$ is $p({\bf t})$ where each $t_i$ in ${\bf t}$ is comprised of ground terms from the extended signature $\sigma^I$.
\\$F^{ns}$ is $p({\bf t})$.
\\We consider two cases:
\begin{itemize}
\item If $I \models p({\bf t})$, then $gr_I[F]^{\mu I}$ is $p({\bf t})$. By 
Lemma~\ref{lem:unisort1}
, it follows 
that $I^{ns} \models p({\bf t})$ and so $gr_{I^{ns}}[F^{ns}]^{\mu I^{ns}}$ is $p({\bf t})$. Thus, in this case, 
$J$ is a model of $gr_I[F]^{\mu I}$ iff $K$ is a model of $gr_{I^{ns}}[F^{ns}]^{\mu I^{ns}}$. 
\item If $I \not \models p({\bf t})$, then $gr_I[F]^{\mu I}$ is $\bot$. By 
Lemma~\ref{lem:unisort1}
, it follows
that $I^{ns} \not \models p({\bf t})$ and so $gr_{I^{ns}}[F^{ns}]^{\mu I^{ns}}$ is also $\bot$. Thus, in this case, 
$J$ is not a model of $gr_I[F]^{\mu I}$ and $K$ is not a model of $gr_{I^{ns}}[F^{ns}]^{\mu I^{ns}}$ so the claim follows. 
\end{itemize}

\item $F$ is $t_1 = t_2$ where each $t_i$ is comprised of ground terms from the extended signature $\sigma^I$. 
\\$F^{ns}$ is $t_1 = t_2$.
\\We consider two cases:
\begin{itemize}
\item If $(t_1)^I = (t_2)^I$, then $gr_I[F]^{\mu I}$ is $t_1 = t_2$. By 
Lemma~\ref{lem:unisort1}
, it follows that
$(t_1)^{I^{ns}} = (t_2)^{I^{ns}}$ and so $gr_{I^{ns}}[F^{ns}]^{\mu I^{ns}}$ is $t_1 = t_2$. Thus, in this case by the second item in the requirement of this lemma, 
$J$ is a model of $gr_I[F]^{\mu I}$ iff $K$ is a model of $gr_{I^{ns}}[F^{ns}]^{\mu I^{ns}}$. 
\item If $(t_1)^I \neq (t_2)^I$, then $gr_I[F]^{\mu I}$ is $\bot$. By 
Lemma~\ref{lem:unisort1}
, it follows that 
$(t_1)^{I^{ns}} \neq (t_2)^{I^{ns}}$ and so $gr_{I^{ns}}[F^{ns}]^{\mu I^{ns}}$ is also $\bot$. Thus, in this case, 
$J$ is not a model of $gr_I[F]^{\mu I}$ and $K$ is not a model of $gr_{I^{ns}}[F^{ns}]^{\mu I^{ns}}$ so the claim follows. 
\end{itemize}

\item $F$ is $G \odot H$ where $\odot \in \{\land, \lor, \rightarrow\}$.
\\$F^{ns}$ is $G^{ns} \odot H^{ns}$. We consider two cases:
\begin{itemize}
\item If $I \models G \odot H$, then $gr_I[F]^{\mu I}$ is $gr_I[G]^{\mu I} \odot gr_I[H]^{\mu I}$. By Lemma~\ref{lem:unisort1}, $I^{ns} \models G^{ns} \odot H^{ns}$ and so
$gr_{I^{ns}}[F^{ns}]^{\mu I^{ns}}$ is $gr_{I^{ns}}[G^{ns}]^{\mu I^{ns}} \odot gr_{I^{ns}}[H^{ns}]^{\mu I^{ns}}$ so the claim follows by induction on $G$ and $H$. 
\item If $I \not \models G \odot H$ then $gr_I[F]^{\mu I}$ is $\bot$. By Lemma~\ref{lem:unisort1}, $I^{ns} \not \models G^{ns} \odot H^{ns}$ and so
$(F^{ns})^{\mu I^{ns}}$ is $\bot$. Thus, in this case, $J$ is not a model of $gr_I[F]^{\mu I}$ and $K$ is not a model of $gr_{I^{ns}}[F^{ns}]^{\mu I^{ns}}$ so the claim follows. 
\end{itemize}

\item $F$ is $\exists x (G(x))$ where the sorted variable $x$ has sort ${\sort s}$.
\\$F^{ns}$ is $\exists y (G(y)^{ns} \land {\sort s}(y))$ (note that the variable here is unsorted). 
\\$gr_I[F]$ is $\{gr_I[G(\xi^\dia)] : \xi \in |I|^{\sort s}\}^\lor$.
\\$gr_{I^{ns}}[F^{ns}]$ is $\{gr_{I^{ns}}[G(\xi^\dia)^{ns}] \land {\sort s}(\xi^\dia) : \xi \in |I^{ns}|\}^\lor$.
\\$gr_I[F]^{\mu I}$ is equivalent to 
$$\{gr_I[G(\xi^\dia)]^{\mu I} : \xi \in |I|^{\sort s}\ \text{ and } I \models gr_I[G(\xi^\dia)]\}^\lor.$$
\\$gr_{I^{ns}}[F^{ns}]^{\mu I^{ns}}$ is equivalent to 
$$\{gr_{I^{ns}}[G(\xi^\dia)^{ns}]^{\mu I^{ns}} \land {\sort s}(\xi^\dia): \xi \in |I^{ns}|\ \text{ and } I^{ns} \models gr_{I^{ns}}[G(\xi^\dia)^{ns}] \land {\sort s}(\xi^\dia)\}^\lor.$$ 
Further, since $I^{ns} \models {\sort s}(\xi^\dia)$ iff $\xi$ is from $|I|^s$ and by the first item in the requirement of this lemma, $K \models gr_{I^{ns}}[F^{ns}]^{\mu I^{ns}}$ iff 
$$K \models \{gr_{I^{ns}}[G(\xi^\dia)^{ns}]^{\mu I^{ns}} : \xi \in |I|^{\sort s}\ \text{ and } I^{ns} \models gr_{I^{ns}}[G(\xi^\dia)^{ns}]\}^\lor.$$ Then, by I.H. on each $G(\xi^\dia)$ such that $\xi \in |I|^{\sort s}$ and $I \models G(\xi^\dia)$, we have that $J \models gr_I[G(\xi^\dia)]^{\mu I}$ iff $K \models gr_{I^{ns}}[G(\xi^\dia)^{ns}]^{\mu I^{ns}}$, from 
which the claim then follows.

\item $F$ is $\forall x (G(x))$ where the sorted variable $x$ has sort ${\sort s}$.
\\$F^{ns}$ is $\forall y ({\sort s}(y) \rightarrow G(y))$ (note that the variable here is unsorted). 
\\$gr_I[F]$ is $\{gr_I[G(\xi^\dia)] : \xi \in |I|^{\sort s}\}^\land$.
\\$gr_{I^{ns}}[F^{ns}]$ is $\{{\sort s}(\xi^\dia) \rightarrow gr_{I^{ns}}[G(\xi^\dia)^{ns}] : \xi \in |I^{ns}|\}^\land$.

We consider two cases:
\begin{itemize}
\item If $I \models G(\xi^\dia)$ for every $\xi \in |I|^{\sort s}$, then $gr_I[F]^{\mu I}$ is equivalent to 
$$\{gr_I[G(\xi^\dia)]^{\mu I} : \xi \in |I|^{\sort s}\}^\land.$$ 
For every $\xi \notin |I|^{\sort s}$, $I^{ns} \not \models {\sort s}(\xi^\dia)$ and so in $gr_{I^{ns}}[F^{ns}]^{\mu I^{ns}}$, the implications corresponding to such $\xi$ are vacuously satisfied and so $gr_{I^{ns}}[F^{ns}]^{\mu I^{ns}}$ is equivalent to 
$$\{{\sort s}(\xi^\dia)^{I^{ns}} \rightarrow gr_{I^{ns}}[G(\xi^\dia)^{ns}]^{\mu I^{ns}} : \xi \in |I|^{\sort s} \text{ and } I^{ns} \models gr_{I^{ns}}[G(\xi^\dia)^{ns}] \}^\land.$$
Since $\xi \in |I|^{\sort s}$ iff $I^{ns} \models {\sort s}(\xi^\dia)$ and since by Lemma~\ref{lem:unisort1}, $I^{ns} \models gr_{I^{ns}}[G(\xi^\dia)^{ns}]$ for every $\xi \in |I|^{\sort s}$, $K \models gr_{I^{ns}}[F^{ns}]^{\mu I^{ns}} $ iff
$$K \models \{gr_{I^{ns}}[G(\xi^\dia)^{ns}]^{I^{ns}} : \xi \in |I|^s \}^\land.$$ Then, by I.H. on each $G(\xi^\dia)$ such that $\xi \in |I|^{\sort s}$, we have that $J \models gr_I[G(\xi^\dia)]^{\mu I}$ iff $K \models gr_{I^{ns}}[G(\xi^\dia)^{ns}]^{\mu I^{ns}}$, from 
which the claim then follows.

\item If $I \not \models G(\xi^\dia)$ for some $\xi \in |I|^{\sort s}$, then $gr_I[F]^{\mu I}$ is $\bot$. Since $\xi \in |I|^{\sort s}$, $I^{ns} \models {\sort s}(\xi^\dia)$ but by Lemma~\ref{lem:unisort1}, $I^{ns} \not \models gr_{I^{ns}}[G(\xi^\dia)^{ns}]$ so $gr_{I^{ns}}[F^{ns}]^{\mu I^{ns}}$ is $\bot$. In this case, 
$J$ is not a model of $gr_I[F]^{\mu I}$ and $K$ is not a model of $gr_{I^{ns}}[F^{ns}]^{\mu I^{ns}}$ so the claim follows. \qed
\end{itemize}

\end{itemize}
\EOCC

\begin{lemma}\label{lem:mssat-relation}\optional{lem:mssat-relation}
Given a formula $F$ of many-sorted signature $\sigma$ and two interpretations $L$ and $L_1$ of $\sigma^{ns}$ such that $R(L,L_1)$, if $L \models F^{ns} \land SF_\sigma$, then $L_1 \models F^{ns} \land SF_\sigma$.
\end{lemma}

\proof
Assume that $L \models F^{ns} \land SF_\sigma$.
We first show that $L_1 \models SF_\sigma$. Since $R(L,L_1)$, $L$ and $L_1$ agree on all sort predicates ${\sort s}$ corresponding to sorts $s \in \sigma$. Thus, $L_1$ clearly satisfies the first two items of $SF_\sigma$. We now consider the third item of $SF_\sigma$. For tuples $\xi_1,\dots,\xi_k$ such that each $\xi_i \in args_i$ where $args_i$ is the $i$-th argument sort of $f$, since $R(L,L_1)$, $L$ and $L_1$ agree on $f(\xi_1,\dots,\xi_k)$ so $L_1$ satisfies the implication. For all other tuples, the implication is vacuously satisfied. Finally, the fourth and fifth items of $SF_\sigma$ are tautologies in classical logic so we conclude that $L_1 \models SF_\sigma$. 

Next, $L_1 \models F^{ns}$ can be shown by induction on the structure of $F^{ns}$.
\qed

\BOCC
\begin{itemize}
\item $F^{ns}$ is $p({\bf t})$ where ${\bf t}$ is a ground term from the extended signature $\sigma^I$. Since every $t_i \in {\bf t}$ must be from the $i$-th argument sort of $p$, it follows from $R(L,L_1)$ that $L_1 \models F^{ns}$.

\item $F^{ns}$ is $t_1 = t_2$ where $t_1$ and $t_2$ are ground terms from the extended signature $\sigma^I$. Since every subterm of $t_1$ and $t_2$ must be from the the appropriate sort, it follows from $R(L,L_1)$ that $L_1 \models F^{ns}$.

\item $F^{ns}$ is $G^{ns} \odot H^{ns}$ where $\odot \in \{\land, \lor, \rightarrow\}$.
The claim follows by I.H. on $G^{ns}$ and $H^{ns}$.

\item $F^{ns}$ is $\exists y (G(y) \land {\sort s}(y))$. Since we assume that $L \models F^{ns}$, there is some $\xi \in |I^{ns}|$ such that $L \models G(\xi^\dia) \land {\sort s}(\xi^\dia)$. Further, since $L \models {\sort s}(\xi^\dia)$ iff $\xi \in |I|^s$, the claim follows by I.H. on $G(\xi^\dia)$. 

\item $F^{ns}$ is $\forall y ({\sort s}(y) \rightarrow G(y))$. Since we assume that $L \models F^{ns}$, for every $\xi \in |I^{ns}|$ we have $L \models {\sort s}(\xi^\dia) \rightarrow G(\xi^\dia)$. For every $\xi \notin |I|^s$, $L_1$ vacuously satisfies ${\sort s}(\xi^\dia) \rightarrow G(\xi^\dia)$. For every $\xi \in |I|^s$, since $L_1 \models {\sort s}(\xi^\dia)$ iff $\xi \in |I|^s$, the claim follows by I.H. on every $G(\xi^\dia)$ such that $\xi \in |I|^s$. 
\end{itemize}
\EOCC

\begin{lemma}\label{lem:mssm-relation}\optional{lem:mssm-relation}
Given a formula $F$ of many-sorted signature $\sigma$, a set of function and predicate constants ${\bf c}$ from $\sigma$ and two interpretations $L$ and $L_1$ of $\sigma^{ns}$ such that $R(L,L_1)$, if $L$ is a stable model of $F^{ns} \land SF_\sigma$ w.r.t. ${\bf c}$, then $L_1$ is a stable model of $F^{ns} \land SF_\sigma$ w.r.t. ${\bf c}$.
\end{lemma}

\proof Omitted. The proof is long but not complicated. \qed

\BOCC
\proof
We first note ${\bf c}$ contains function and predicate constants from $\sigma$ and thus contains none of the sort predicates introduced in $\sigma^{ns}$.

We assume that $L$ is a stable model of $F^{ns} \land SF_\sigma$, and wish to show that $L_1$ is a stable model of $F^{ns} \land SF_\sigma$.
That is, given that $L \models F^{ns}\land SF_\sigma$ and there is no interpretation $K$ such that $K <^{\bf c} L$ and $K \models gr_L[F^{ns}\land SF_\sigma]^{\mu L}$, we wish to show that there is no 
interpretation $K_1$ such that $K_1 <^{\bf c} L_1$ and $K_1 \models gr_{L_1}[F^{ns}\land SF_\sigma]^{\mu L_1}$. 
Equivalently, we will show that if there is an interpretation $K_1$ such that $K_1 <^{\bf c} L_1$ and $K_1 \models gr_{L_1}[F^{ns}\land SF_\sigma]^{\mu L_1}$, then there is an interpretation $K$ such that $K <^{\bf c} L$ and $K \models gr_L[F^{ns}\land SF_\sigma]^{\mu L}$.

Assume that there is an interpretation $K_1$ such that $K_1 <^{\bf c} L_1$ and $K_1 \models gr_{L_1}[F^{ns}\land SF_\sigma]^{\mu L_1}$, we construct $K$ as follows.
\begin{itemize}
\item $|K| = |K_1|$,
\item ${\sort s}^K = {\sort s}^{K_1}$ for every ${\sort s}$ corresponding to a sort $s \in \sigma$,
\item $c(\xi_1,\dots,\xi_k)^K = c(\xi_1,\dots,\xi_k)^{K_1}$ for every tuple $\xi_1,\dots,\xi_k$ such that $\xi_i \in s_i$ where $s_i$ is the $i$-th argument sort of $c$,
\item $c(\xi_1,\dots,\xi_k)^K = c(\xi_1,\dots,\xi_k)^L$ for every tuple $\xi_1,\dots,\xi_k$ such that $\xi_i \notin s_i$ for some $i$ where $s_i$ is the $i$-th argument sort of $c$.
\end{itemize}

We first show that $K <^{\bf c} L$. By definition $|K| = |K_1|$. From $K_1 <^{\bf c} L_1$, it follows that $|K| = |L_1|$. Then since $R(L_1,L)$, it follows that $|K| = |L|$. By definition of $K$, it follows that ${\sort s}^K = {\sort s}^{K_1}$ for every ${\sort s}$ corresponding to a sort $s \in \sigma$. Then, since $K_1 <^{\bf c} L_1$ and since $R(L_1,L)$, it follows that ${\sort s}^K = {\sort s}^{L}$. Now, for any function or predicate $c$ and any tuple $\xi_1,\dots,\xi_k$ such that $\xi_i \notin s_i$ for some $i$ where $s_i$ is the $i$-th argument sort of $c$, by definition, $c(\xi_1,\dots,\xi_k)^K = c(\xi_1,\dots,\xi_k)^L$. Finally, for every function or predicate $c$ and every tuple $\xi_1,\dots,\xi_k$ such that $\xi_i \in s_i$ where $s_i$ is the $i$-th argument sort of $c$, since $R(L,L_1)$, it is clear that $c(\xi_1,\dots,\xi_k)^{L_1} = c(\xi_1,\dots,\xi_k)^L$. We also have by definition, $c(\xi_1,\dots,\xi_k)^K = c(\xi_1,\dots,\xi_k)^{K_1}$ for such predicate (functions) and tuples.

Now since we assume that $K_1 <^{\bf c} L_1$, there must be some function or predicate constant $c$ and some tuple $\xi_1,\dots,\xi_k$ such that $c(\xi_1,\dots,\xi_k)^{K_1} \neq c(\xi_1,\dots,\xi_k)^{L_1}$. Now by definition of $K_1 <^{\bf c} L_1$, $K_1$ and $L_1$ agree on all of the sort predicates ${\sort s}$ coming from sorts $s \in \sigma$. Further, since $K_1 \models (SF_\sigma)^{\mu L_1}$, the fourth and fifth items of $(SF_\sigma)^{\mu L_1}$ force $K_1$ to agree with $L_1$ on all functions (predicates) and tuples such that some tuple is not of the correct sort. Thus, it must be that the tuple $\xi_1,\dots,\xi_k$ such that $c(\xi_1,\dots,\xi_k)^{K_1} \neq c(\xi_1,\dots,\xi_k)^{L_1}$ has that every $\xi_i$ belongs to the appropriate sort. Thus, by the observation before that $c(\xi_1,\dots,\xi_k)^{L_1} = c(\xi_1,\dots,\xi_k)^L$ and $c(\xi_1,\dots,\xi_k)^K = c(\xi_1,\dots,\xi_k)^{K_1}$, it follows that $K <^{\bf c} L$.

Now, we show that $K \models gr_L[SF_\sigma]^{\mu L}$ by considering each item of $SF_\sigma$. We first note that since $K_1 \models gr_{L_1}[SF_\sigma]^{\mu L_1}$, it must be that $L_1 \models gr_{L_1}[SF_\sigma]$. Thus by Lemma~\ref{lem:mssat-relation}, we have that $L \models gr_L[SF_\sigma]$.

\begin{itemize}
\item Item 1: $\forall y({\sort s}_i(y) \rightarrow {\sort s}_j(y))$ for every two sorts $s_i$ and $s_j$ in $\sigma$ such that $s_i$ is a subsort of $s_j$.

From $K_1 <^{\bf c} L_1$, it follows that ${\sort s}_i(\xi)^{K_1} = {\sort s}_i(\xi)^{L_1}$ for every predicate ${\sort s}$ corresponding to a sort $s \in \sigma$ and for every $\xi$ in $|L_1| = |K_1|$. By definition of $K$, and since $R(L,L_1)$, we then have that ${\sort s}_i(\xi)^K = {\sort s}_i(\xi)^{K_1} = {\sort s}_i(\xi)^{L_1} = {\sort s}_i(\xi)^L$ so clearly the claim holds 
for this item.

\item Item 2: $\exists y ({\sort s}(y))$ for every sort $s$ in $\sigma$. 

By the same argument in Item 1, ${\sort s}_i(\xi)^K = {\sort s}_i(\xi)^{K_1} = {\sort s}_i(\xi)^{L_1} = {\sort s}_i(\xi)^L$ so clearly the claim holds 
for this item.

\item the formulas 
$\forall y_1\dots y_{k} ({\sort args}_1(y_1) \land \dots \land {\sort args}_k(y_k) \rightarrow {\sort vals}(f(y_1,\dots, y_{k})))$ for each function constant $f$ in $\sigma$ where the arity of $f$ is $k$ and the $i$-th argument sort of $f$ is $args_i$ and the value sort of $f$ is $vals$.

By $R(L,L_1)$, for every $\xi_1,\dots,\xi_k$ such that $\xi_i \in args_i$, $f(\xi_1,\dots,\xi_k)^L = f(\xi_1,\dots,\xi_k)^{L_1})$. Then, by definition of $K$, $f(\xi_1,\dots,\xi_k)^K = f(\xi_1,\dots,\xi_k)^{K_1})$ so the claim holds for this item.

\item the formulas 
$\forall y_1\dots y_{k+1} (\neg {\sort args}_1(y_1) \lor \dots \lor \neg {\sort args}_k(y_k) \rightarrow \{f(y_1,\dots, y_{k}) = y_{k+1}\})$ 

for each function constant $f$ in $\sigma$ where the arity of $f$ is $k$ and the $i$-th argument sort of $f$ is $args_i$. 

By definition of $K$, $f(\xi_1,\dots,\xi_k)^K = f(\xi_1,\dots,\xi_k)^L$ and since the reduct of these formulas is only satisfied when $K$ agrees with $L$ for these tuples, the claim holds for this item.

\item the formulas 
$\forall y_1\dots y_{k} (\neg {\sort args}_1(y_1) \lor \dots \lor \neg {\sort args}_k(y_k) \rightarrow \{p(y_1,\dots, y_{k})\})$ 
for each function constant $f$ in $\sigma$ where the arity of $f$ is $k$ and the $i$-th argument sort of $f$ is $args_i$. 

By definition of $K$, $p(\xi_1,\dots,\xi_k)^K = p(\xi_1,\dots,\xi_k)^L$ and since the reduct of these formulas is only satisfied when $K$ agrees with $L$ for these tuples, the claim holds for this item.
\end{itemize}

Finally, we show that $K \models gr_L[F^{ns}]^{\mu L}$ iff $K_1 \models gr_{L_1}[F^{ns}]^{\mu L_1}$ by induction on $F^{ns}$ and will conclude that since we assume $K_1 \models gr_{L_1}[F^{ns}]^{\mu L_1}$, that $K \models gr_L[F^{ns}]^{\mu L}$.

\begin{itemize}
\item $F^{ns}$ is $p({\bf t})$ where each element of ${\bf t}$ is a ground term from the extended signature $\sigma^I$ and belongs to the corresponding argument sort of $p$.
\\$gr_L[F^{ns}]^{\mu L}$ is the same as $gr_{L_1}[F^{ns}]^{\mu L_1}$ by Lemma~\ref{lem:mssat-relation}. If $L_1 \not \models p({\bf t})$ then $gr_{L_1}[F^{ns}]^{\mu L_1}$ is $\bot$ neither $K$ nor $K_1$ satisfy this reduct so the claim holds. If instead $L_1 \models p({\bf t})$ then $gr_{L_1}[F^{ns}]^{\mu L_1}$ is $p({\bf t})$.

Then, by definition of $K$, since $p({\bf t})^K = p({\bf t})^{K_1}$, clearly the claim holds.

\item $F^{ns}$ is $f_1({\bf t_1}) = f_2({\bf t_2})$ where each element of ${\bf t_1}$ and ${\bf t_2}$ is a ground term of the extended signature $\sigma^I$ and belongs to the corresponding argument sort of $f_1$ and $f_2$ respectively.
\\$gr_L[F^{ns}]^{\mu L}$ is the same as $gr_{L_1}[F^{ns}]^{\mu L_1}$ by Lemma~\ref{lem:mssat-relation}. If $L_1 \not \models f_1({\bf t_1}) = f_2({\bf t_2})$ then $gr_{L_1}[F^{ns}]^{\mu L_1}$ is $\bot$ neither $K$ nor $K_1$ satisfy this reduct so the claim holds. If instead $L_1 \models f_1({\bf t_1}) = f_2({\bf t_2})$ then $gr_{L_1}[F^{ns}]^{\mu L_1}$ is $f_1({\bf t_1}) = f_2({\bf t_2})$.

Then, by definition of $K$, since $f_1({\bf t_1})^K = f_1({\bf t_1})^{K_1}$ and $f_2({\bf t_2})^K = f_2({\bf t_2})^{K_1}$, clearly the claim holds.

\item $F^{ns}$ is $G^{ns} \odot H^{ns}$ where $\odot \in \{\land, \lor, \rightarrow \}$.
$gr_L[F^{ns}]^{\mu L}$ is $gr_L[G^{ns}]^{\mu L} \odot gr_L[H^{ns}]^{\mu L}$ and $gr_{L_1}[F^{ns}]^{\mu L_1}$ is $gr_{L_1}[G^{ns}]^{\mu L_1} \odot gr_{L_1}[H^{ns}]^{\mu L_1}$ so the claim follows by I.H. on $G^{ns}$ and $H^{ns}$.

\item $F^{ns}$ is $\exists y(G(y)^{ns} \land {\sort s}(y))$.
\\$gr_L[F^{ns}]^{\mu L}$ is equivalent to $\{gr_L[G(\xi^\dia)^{ns}]^{\mu L} : L \models {\sort s}(\xi^\dia)\}^\lor$ and
\\$gr_{L_1}[F^{ns}]^{\mu L_1}$ is equivalent to $\{gr_{L_1}[G(\xi^\dia)^{ns}]^{\mu L_1} : L_1 \models {\sort s}(\xi^\dia)\}^\lor$. Since 
$R(L,L_1)$, we have that ${\sort s}^L = {\sort s}^{L_1}$ and so the claim follows by 
I.H. on each $G(\xi^\dia)^{ns}$ such that $L \models {\sort s}(\xi^\dia)$.

\item $F^{ns}$ is $\forall y({\sort s}(y) \rightarrow G(y)^{ns})$.
We consider two cases:
\begin{itemize}
\item If $L \not \models G(\xi^\dia)^{ns}$ for some $\xi$ such that $L \models {\sort s}(\xi^\dia)$, then $gr_L[F^{ns}]^{\mu L}$ is $\bot$. By Lemma~\ref{lem:mssat-relation}, we have that $L_1 \not \models G(\xi^\dia)^{ns}$ and so $gr_{L_1}[F^{ns}]^{\mu L_1}$ is $\bot$. Thus neither $K$ nor $K_1$ satisfies the reduct and so the claim holds in this case.
\item Otherwise, $L \models G(\xi^\dia)^{ns}$ for every $\xi$ such that $L \models {\sort s}(\xi^\dia)$.
\\$gr_L[F^{ns}]^{\mu L}$ is equivalent to $\{gr_L[G(\xi^\dia)^{ns}]^L : L \models {\sort s}(\xi^\dia)\}^\land$ and
\\$gr_{L_1}[F^{ns}]^{\mu L_1}$ is equivalent to $\{gr_{L_1}[G(\xi^\dia)^{ns}]^{L_1} : L_1 \models {\sort s}(\xi^\dia)\}^\land$. Since 
$R(L,L_1)$, we have that ${\sort s}^L = {\sort s}^{L_1}$ and so the claim follows by 
I.H. on each $G(\xi^\dia)^{ns}$ such that $L \models {\sort s}(\xi^\dia)$. \qed
\end{itemize}

\end{itemize}
\EOCC

\noindent{\textbf {Theorem~\ref{thm:ms2us} \optional{thm:ms2us}}}\
\ 
{\sl
Let $F$ be a formula of a many-sorted signature $\sigma$, and let ${\bf c}$ be a set of function and predicate constants.
\begin{itemize}
\item[(a)] If an interpretation $I$ of signature $\sigma$ is a model of $\fsm[F;{\bf c}]$, then $I^{ns}$ is a model of $\fsm[F^{ns} \land SF_\sigma;{\bf c}]$.
\item[(b)] If an interpretation $L$ of signature $\sigma^{ns}$ is a model of $\fsm[F^{ns} \land SF_\sigma;{\bf c}]$ then there is some interpretation $I$ of signature $\sigma$ such that $I$ is a model of $\fsm[F;{\bf c}]$ and $R(L,I^{ns})$.
\end{itemize}
}

\BOCC
\begin{itemize}
\item[(a)] 
An interpretation $I$ of signature $\sigma$ is a model of $\fsm[F;{\bf c}]$ iff $I^{ns}$ is a model of $\fsm[F^{ns} \land SF_\sigma;{\bf c}]$.
\item[(b)]
An interpretation $L_1$ of signature $\sigma^{ns}$ is a model of $\fsm[F^{ns} \land SF_\sigma;{\bf c}]$ iff there is some interpretation $L$ of signature $\sigma^{ns}$ such that $R(L,L_1)$ and $L = I^{ns}$ for some model $I$ of $\fsm[F;{\bf c}]$.
\end{itemize}

Given a formula $F$ of a many-sorted signature $\sigma$, and a set of function and predicate constants ${\bf c}$,
\\a) If an interpretation $I$ of signature $\sigma$ is a model of $\fsm[F;{\bf c}]$, then $I^{ns}$ is a model of $\fsm[F^{ns} \land SF_\sigma;{\bf c}]$.
\\b) If an interpretation $L$ of signature $\sigma^{ns}$ is a model of $\fsm[F^{ns} \land SF_\sigma;{\bf c}]$ then there is some interpretation $I$ of signature $\sigma$ such that $I$ is a model of $\fsm[F;{\bf c}]$ and $R(L,I^{ns})$.
\EOCC

\proof

(a) Consider an interpretation $I$ (of many-sorted signature $\sigma$) that is a stable model of $F$ w.r.t. ${\bf c}$. This means that $I \models F$ and there is no interpretation $J$ such that $J <^{\bf c} I$ and $J \models gr_I[F]^{\mu I}$. We wish to show that $I^{ns} \models F^{ns}\land SF_\sigma$ and there is no (unsorted) interpretation $K$ such that 
$K <^{\bf c} I^{ns}$ and $K \models gr_{I^{ns}}[F^{ns}\land SF_\sigma]^{\mu I^{ns}}$. From Lemma~\ref{lem:unisort1}, $I \models F$ iff $I^{ns} \models F^{ns}$. It follows from the definition of $I^{ns}$ that $I^{ns} \models SF_\sigma$ so 
we conclude that $I \models F$ iff $I^{ns} \models F^{ns}\land SF_\sigma$. For the second item, we will prove the contrapositive: if there is an (unsorted) interpretation $K$ such that 
$K <^{\bf c} I^{ns}$ and $K \models gr_{I^{ns}}[F^{ns}\land SF_\sigma]^{\mu I^{ns}}$, then there is a (many-sorted) interpretation $J$ such that $J <^{\bf c} I$ and $J \models gr_I[F]^{\mu I}$.

Assume there is an interpretation $K$ such that $K <^{\bf c} I^{ns}$ and $K \models gr_{I^{ns}}[F^{ns}\land SF_\sigma]^{\mu I^{ns}}$. We
obtain the interpretation $J$ as follows.
For every sort $s$ in $\sigma$, $|J|^s = |I|^s$.
For every predicate and function constant $c$ in $\sigma$ and every tuple $\bfxi$ such that each element $\xi_i \in |I|^{s_i}$ where $s_i$ is the sort of the $i$-th argument of $c$, we let $c^J(\bfxi) = c^K(\bfxi)$. For predicate constants, it is not hard to see that this is 
a valid assignment as atoms are either true or false regardless of considering many-sorted or unsorted logic.

We argue that this assignment is also valid for function constants. That is, $K$ does not map a function $f$ to a value outside of $|I|^s$ where $s$ is the value sort of $f$. This follows from the fact that $I^{ns} \models SF_\sigma$ and in particular, the third item of $SF_\sigma$. Thus, since $K \models gr_{I^{ns}}[F^{ns}\land SF_\sigma]^{\mu I^{ns}}$, it follows that $K$ too maps functions to elements of the appropriate sort.

We now show that $J <^{\bf c} I$. Since $K \models gr_{I^{ns}}[SF_\sigma]^{\mu I^{ns}}$, the fourth and fifth rules in $SF_\sigma$ are choice formulas that force $K$ to agree with $I^{ns}$ on every predicate and function constant $c$ for every tuple that has at least one element outside of the corresponding sort. For every predicate and function constant $c$ and all tuples that have all elements in the appropriate sort, $K$ and $J$ agree. Further, since $I$ and $I^{ns}$ agree on these as well, it follows immediately since $K <^{\bf c} I^{ns}$, that $J <^{\bf c} I$.

To apply Lemma~\ref{lem:unisort2}, we verify the conditions of the lemma. It is clear that the second condition is true. The first condition follows from the definition of $K <^{\bf c} I^{ns}$: since the sort predicates are not in ${\bf c}$, $K$ and $I^{ns}$ agree on these predicates. The third condition follows from the fact that since $K \models gr_{I^{ns}}[F^{ns}\land SF_\sigma]^{\mu I^{ns}}$ it follows that $K \models gr_{I^{ns}}[SF_\sigma]^{\mu I^{ns}}$;  the fourth and fifth rules in $SF_\sigma$ are choice formulas that force $K$ to agree with $I^{ns}$ for every tuple that has at least one element outside of the corresponding sort. Thus, by Lemma~\ref{lem:unisort2}, since $K \models gr_{I^{ns}}[F^{ns}\land SF_\sigma]^{\mu I^{ns}}$ and thus, $K \models gr_{I^{ns}}[F^{ns}]^{\mu I^{ns}}$, it follows that $J \models gr_I[F]^{\mu I}$.

(b) Given an interpretation $L$ that is a stable model of $F^{ns}\land SF_\sigma$ w.r.t. ${\bf c}$, we first obtain the interpretation $L_1$ of $\sigma^{ns}$ as follows.

\begin{itemize}
\item $|L_1| = |L|$; 
\item ${\sort s}^{L_1}$ = ${\sort s}^{L}$ for every ${\sort s}$ corresponding to a sort $s$ from $\sigma$; 
\item $c(\xi_1,\dots,\xi_k)^{L_1} = c(\xi_1,\dots,\xi_k)^{L}$ for every tuple $\xi_1,\dots,\xi_k$ such that $\xi_i \in s_i$ where $s_i$ is the $i$-th argument sort of $c$; 
\item $c(\xi_1,\dots,\xi_k)^{L_1} = |L_1|_0$ for every tuple $\xi_1,\dots,\xi_k$ such that $\xi_i \notin s_i$ for some $i$ where $s_i$ is the $i$-th argument sort of $c$.
\end{itemize}

It is easy to see that $R(L,L_1)$. By Lemma~\ref{lem:mssm-relation}, $L_1$ is a stable model of $F^{ns}\land SF_\sigma$ w.r.t. ${\bf c}$. We then obtain the interpretation $I$ of signature $\sigma$ as follows.

For every sort $s$ in $\sigma$, $|I|^s = s^{L_1}$.
For every predicate and function constant $c$ in $\sigma$ and every tuple $\bfxi$ such that $\xi_i \in |L|^{s_i}$ where $s_i$ is the sort of the $i$-th argument of $c$, we have $c(\bfxi)^I = c(\bfxi)^{L_1}$. For predicate constants, it is not hard to see that this is 
a valid assignment as atoms are either true or false regardless of considering many-sorted or unsorted logic.

We argue that this assignment is also valid for function constants. That is, $I$ does not map a function $f$ to a value outside of $|I|^s$ where $s$ is the value sort of $f$. This follows from the fact that $L_1 \models SF_\sigma$ (by Lemma~\ref{lem:mssat-relation}) and in particular, the third item of $SF_\sigma$. Thus, it follows that $I$ too maps functions to elements of the appropriate sort.

Now it is clear that $L_1 = I^{ns}$ and so we have $R(L,I^{ns})$. We now show that $I$ is a stable model of $F$.

We have an interpretation $I$ (of many-sorted signature $\sigma$) such that $I^{ns}$ is a stable model of $F^{ns} \land SF_\sigma$ w.r.t. ${\bf c}$. This means that $I^{ns} \models F^{ns} \land SF_\sigma$ and there is no interpretation $K$ such that $K <^{\bf c} I^{ns}$ and $K \models gr_{I^{ns}}[F^{ns}\land SF_\sigma]^{\mu I^{ns}}$. We wish to show that $I \models F$ and there is no interpretation $J$ such that 
$J <^{\bf c} I$ and $J \models gr_{I}[F]^{\mu I}$. From Lemma~\ref{lem:unisort1}, $I \models F$ iff $I^{ns} \models F^{ns}$ so we conclude that $I \models F$. For the second item, we will prove the contrapositive; if there is a (many-sorted) interpretation 
$J$ such that 
$J <^{\bf c} I$ and $J \models gr_{I}[F]^{\mu I}$, then there is an (unsorted) interpretation $K$ such that $K <^{\bf c} I^{ns}$ and $K \models gr_{I^{ns}}[F^{ns} \land SF_\sigma]^{\mu I^{ns}}$.

Assume there is an interpretation $J$ such that $J <^{\bf c} I$ and $J \models gr_{I}[F]^{\mu I}$. We obtain the interpretation $K$ be $J^{ns}$.

We now show that $K <^{\bf c} I^{ns}$. For every predicate and function constant $c$ for every tuple that has at least one element outside of the corresponding sort, by definition of $K = J^{ns}$, $c^K = c^{I^{ns}} = |I^{ns}|_0$ if $c$ is a function constant and $c^K = c^{I^{ns}} = \false$ if $c$ is a predicate constant. That is, for every predicate and function constant $c$ for every tuple that has at least one element outside of the corresponding sort, $K$ and $I^{ns}$ agree. For every predicate and function constant $c$ and all tuples of elements in the appropriate sort, $K$ and $J$ agree. Further, since $I$ and $I^{ns}$ agree on these as well, $K <^{\bf c} I^{ns}$ follows immediately from $J <^{\bf c} I$.

To apply Lemma~\ref{lem:unisort2}, we must verify the conditions of the lemma. It is clear that the second condition is true. The first condition follows from the definition of $K = J^{ns}$. The third condition follows from the observation above: 
by definition of $K = J^{ns}$, $c^K = c^{I^{ns}} = |I^{ns}|_0$ if $c$ is a function constant and $c^K = c^{I^{ns}} = \false$ if $c$ is a predicate constant. Thus, by Lemma~\ref{lem:unisort2}, since $J \models gr_I[F]^{\mu I}$, it follows that $K \models gr_{I^{ns}}[F^{ns}]^{\mu I^{ns}}$. 

Then, it is easy to see that by definition of $I^{ns}$, $I^{ns} \models SF_\sigma$. Then, by definition of $K = J^{ns}$, it is clear that $K \models SF_\sigma$. We will show that $K \models (SF_\sigma)^{\mu I^{ns}}$. 

Since $K$ and $I^{ns}$ agree on all sort predicates, it is clear that $K$ satisfies the formulas in the first two items of $(SF_\sigma)^{\mu I^{ns}}$. 

Since $K$ and $I^{ns}$ agree on all function constants $f$ for tuples $\xi_i,\dots,\xi_k$ such that each $\xi_i$ is in $|I|^{s_i}$ where $s_i$ is the $i$-th argument sort of $f$, it is clear that $K$ satisfies the third item of $(SF_\sigma)^{\mu I^{ns}}$.

The last two items of $(SF_\sigma)^{\mu I^{ns}}$ are only satisfied if $K$ agrees with $I^{ns}$ on all predicate (function) constants $c$ and all tuples $\xi_1,\dots,\xi_k$ such that some $\xi_i$ is not in $|I|^{s_i}$ where $s_i$ is the $i$-th argument sort of $c$. However, by definition of $K = J^{ns}$ and $I^{ns}$, both $K$ and $I^{ns}$ map this to $|I^ns|_0$ if $c$ is a function constant or $\false$ if $c$ is a predicate constant so $K$ satisfies these items. So we conclude that $K \models gr_{I^{ns}}[F^{ns}\land SF_\sigma]^{\mu I^{ns}}$.\qed


\subsection{Proof of Theorem~\ref{thm:clingcon}}

\begin{lemma}\label{lem:clingcon}
Let $\Pi$ be a clingcon program with CSP $(V,D,C)$, let ${\cal T}$ be the background theory conforming to $(V,D,C)$, let ${\bf p}$ be the set of all propositional constants occurring in~$\Pi$, let $I$ be a ${\cal T}$-interpretation $\langle I^f,X\rangle$ and let $J$ be an interpretation $\langle I^f, Y\rangle$ such that $Y\subset X$. If $I\models \Pi$, then $Y\models \Pi^X_{I^f}$ iff $J\models \Pi^\mu{I}$.
\end{lemma}

\proof Assume $I\models \Pi$.

($\Rightarrow$) Assume $Y\models \Pi^X_{I^f}$. This means that $Y$ satisfies every rule in the reduct $\Pi^X_{I_f}$. For each rule $r$ of the form~\eqref{clingcon-rule} in $\Pi$, there are two cases:
\begin{itemize} 
\item Case 1: $X\models B$ and $I^f\models \i{Cn}$. 
In this case, $r^X_{I_f}$ is 
\beq
  a \leftarrow B,
\eeq{lem:clingcon-1}
and $r^\mu{I}$ is equivalent to 
\beq  
   a^\mu{I} \ar B^\mu{I}
\eeq{lem:clingcon-2}
under the assumption $I\models\Pi$.
\bi
\item Subcase 1: $I\models B$. Since $I\models\Pi$, it must be that $I\models a$.
  Consequently, \eqref{lem:clingcon-2} is the same as \eqref{lem:clingcon-1}, so it follows that $J\models r^\mu{I}$.

\item Subcase 2: $I\not\models B$. Since $B^\mu{I}=\bot$, clearly, $J\models r^\mu{I}$. 
\ei

\item Case 2: $X\not\models B$ or $I^f\not\models \i{Cn}$. Clearly, $r^\mu{I}$ is equivalent to $\top$, so $J\models r^\mu{I}$.
\ei

\medskip
($\Leftarrow$) Assume $J\models \Pi^\mu{I}$. 
For each rule $r$ of the form~\eqref{clingcon-rule} in $\Pi$, there are two cases:
\bi
\item Case 1: $I\not\models N\land \i{Cn}$. 
In this case, the reduct $r^X_{I_f}$ is empty. Clearly, $Y\models r^X_{I_f}$.

\item Case 2: $I\models N\land \i{Cn}$. The reduct $r^X_{I_f}$ is $a\ar B$.
\bi
\item Subcase 1: $I\models B$. $r^\mu{I}$ is equivalent to $a^\mu{I}\ar (B\land N\land \i{Cn})^\mu{I}$. Since $J\models r^\mu{I}$, it must be that $a^\mu{I}=a$ and $J\models a$.
  Consequently, $Y\models a$, so $Y\models r^X_{I_f}$.

\item Subcase 2: $I\not\models B$ (i.e., $X\not\models B$). Since $Y\subset X$, we have $Y\not\models B$ so $Y\models r^X_{I_f}$. 
\ei
\ei

\qed

\noindent{\textbf {Theorem~\ref{thm:clingcon} \optional{thm:clingcon}}}\ \ 
{\sl
Let $\Pi$ be a clingcon program with CSP $(V,D,C)$,
let ${\bf p}$ be the set of all propositional constants occurring in~$\Pi$,
let ${\cal T}$ be the background theory conforming to $(V,D,C)$,  and let $\langle I^f,X\rangle$ be a ${\cal T}$-interpretation. Set $X$ is a constraint answer set of~$\Pi$ relative
to~$I^f$ iff $\langle I^f, X\rangle$ is a ${\cal T}$-stable model of $\Pi$ relative to~${\bf p}$.
}

\proof
\[ 
\text{$X$ is a constraint answer set of $\Pi$ relative to $I^f$}
\] 
iff
\[
\text{$X$ satisfies $\Pi^X_{I_f}$, and no proper subset $Y$ of $X$ satisfies $\Pi^X_{I_f}$}
\]
iff (by Lemma~\ref{lem:clingcon})
\[
\text{$\langle I^f, X\rangle$ is a ${\cal T}$-model of $\Pi$, and no interpretation $J$ such that $J<^{\bf p} \langle I^f, X\rangle$ satisfies $\Pi^\mu{I}$}
\]
iff 
\[
\text{$\langle I^f, X\rangle$ is a ${\cal T}$-stable model of $\Pi$ relative to~${\bf p}$}.
\]
\qed


\subsection{Proof of Theorem~\ref{thm:niemelafsm}}

\begin{lemma}\label{lem:ljninterp}\optional{lem:ljninterp}
For any ASP(LC) program~$\Pi$, any LJN interpretation $(X,T)$, and any ${\cal T}$-interpretation $I = \langle I^f,Y \rangle$, the following conditions are equivalent:  
\begin{itemize}
\item $I \models T \cup \overline{T}$;
\item For every theory atom $t$ occurring in $\Pi$, it holds that $(X,T) \models t$ iff $I \models t$.
\end{itemize}
\end{lemma}

\proof
\bi
\item[(i)] Assume $I\models T\cup\overline{T}$. Take any theory atom $t$ occurring in $\Pi$.
\\($\Rightarrow$) Assume $(X,T) \models t$. It is immediate that $t \in T$ and so by the assumption on~$I$, we have $I\models t$. 
\\($\Leftarrow$) Assume $I \models t$. Since $I\models T$, it follows that $t \in T$ and so $(X,T) \models t$. 

\item[(ii)]  Assume that, for every theory atom $t$ occurring in $\Pi$, it holds that $(X,T) \models t$ iff $I\models t$.
By definition of $(X,T) \models t$, for every $t$ occurring in $\Pi$, it follows that $t \in T$ iff $I \models t$. Thus $I \models T$ and $I \models \overline{T}$ so $I \models T \cup \overline{T}$. 
\ei
\qed

\begin{lemma}\label{lem:niemelafsmreduct}\optional{lem:niemelafsmreduct}
Given an ASP(LC) program $\Pi$, two LJN-interpretations $(X,T)$ and $(Y,T)$ such that $(X,T) \models \Pi$ and $Y \subseteq X$, and two ${\cal T}$-interpretations $I = \langle I^f,X \rangle $ and $J = \langle I^f,Y\rangle$ such that $I \models \Pi$, and $I^f \models T \cup \overline{T}$,
It holds that $Y \models \Pi^{(X,T)}$ iff $J \models \Pi^\mu{I}$.
\end{lemma}

\proof
($\Rightarrow$) Assume $Y\models \Pi^{(X,T)}$. This means that $Y$ satisfies every rule in the reduct $\Pi^{(X,T)}$. For each rule $r$ of the form (\ref{nlplc}) in $\Pi$, there are two cases:
\begin{itemize} 
\item Case 1: $(X,T) \models N \land LC$.
\\In this case, the corresponding rule in the reduct $\Pi^{(X,T)}$ is 
$$a \leftarrow B.$$
On the other hand, $r^\mu{I}$ has two cases:
\begin{itemize}
\item Subcase 1: $I \models B$. 
\\Since we assume $I \models \Pi$, it must be that $I \models a$. By Lemma~\ref{lem:ljninterp}, since $(X,T) \models t$ for all $t$ in $LC$, so too does $I$ and so $I \models LC$. In this case, 
$r^\mu{I}$ is 
$$a \leftarrow B,\top,\dots,\top,LC^\mu{I}.$$
Since $I$ and $J$ interpret object constants in the same way and $I \models LC^\mu{I}$, we have $J \models LC^\mu{I}$. Thus by definition of $J$, it follows that $J \models B$ iff $Y \models B$ and $J \models a$ iff $Y \models a$, so the claim holds.


\item Subcase 2: $I \not \models B$. The reduct $r^\mu{I}$ is either $a \leftarrow \bot$ or $\bot \leftarrow \bot$ and in either case, $J \models r^\mu{I}$. 

\end{itemize}

\item Case 2: $(X,T) \not \models N \land LC$.
\\By the condition of $I$ and by Lemma~\ref{lem:ljninterp}, $I \not \models N \land LC$ so $r^\mu{I}$ is $a \leftarrow \bot$ or $\bot \leftarrow \bot$ depending on whether $I \models a$. Thus, $J$ trivially satisfies $r^\mu{I}$.

\end{itemize}

($\Leftarrow$) Assume $J \models \Pi^\mu{I}$. This means that $J$ satisfies every rule in $\Pi^\mu{I}$. For any rule $r$ of the form (\ref{nlplc}) in $\Pi$, there are two cases.
\begin{itemize}

\item Case 1: $I \not \models N \land LC$.
\\By the condition of $I$ and by Lemma~\ref{lem:ljninterp}, $(X,T) \not \models N \land LC$. Thus the reduct $\Pi^{(X,T)}$ does not contain a corresponding rule so there is nothing for $Y$ to satisfy.

\item Case 2: $I \models N \land LC$. \\By the condition of $I$ and by Lemma~\ref{lem:ljninterp}, $(X,T) \models N \land LC$ so the reduct $r^{(X,T)}$ is $a \leftarrow B$.
\begin{itemize}
\item Subcase 1: $I\not\models B$.
\\ By the condition of $I$, $X \not \models B$ and since $Y \subseteq X$, $Y \not \models B$. Thus, $Y \models r^{(X,T)}$.

\item Subcase 2: $I \models B$. 
\\ Since $I \models \Pi$, it must be that $I \models a$ so the reduct $r^\mu{I}$ is $a \leftarrow B \land LC^\mu{I}$. Now since $J$ and $I$ agree on every object constant and since $I \models LC^\mu{I}$, we have $J \models LC^\mu{I}$. Thus, $J \models r^\mu{I}$ iff $J \models a \leftarrow B$. Since we assume $J\models \Pi^I$, we conclude $J \models a \leftarrow B$. Now by definition of $J$, it follows that $Y \models r^{(X,T)}$.
\end{itemize}
\end{itemize}
\qed

\bigskip
\noindent{\textbf {Theorem~\ref{thm:niemelafsm} \optional{thm:niemelafsm}}}\
\ 
{\sl 
Let $\Pi$ be an ASP(LC) program of signature $\langle \sigma^p, \sigma^f\rangle$ where $\sigma^p$ is a set of propositional constants, and let $\sigma^f$ be a set of object constants, and let $I^f$ be an interpretation of $\sigma^f$.
\begin{itemize}
\item[(a)] If $(X,T)$ is an LJN-answer set of $\Pi$, then for any ${\cal T}$-interpretation $I$ such that $I^f\models T\cup\overline{T}$, we have $\langle I^f, X\rangle\models\fsm[\Pi;\sigma^p]$.

\item[(b)] For any ${\cal T}$-interpretation $I=\langle I^f, X\rangle$, 
  if  $\langle I^f, X\rangle\models\fsm[\Pi;\sigma^p]$, then an LJN-interpretation 
   $(X, T)$ where $$T=\{t\mid \text{$t$ is a theory atom in $\Pi$ such that $I^f\models t$}\}$$ is an LJN-answer set of $\Pi$.
\end{itemize}
}


\proof
In this proof, we refer to the reduct-based  characterization of a  stable model from \cite{bartholomew13onthestable}.

$(a)$ Assume $(X,T)$ is an LJN-answer set of $\Pi$. Take any ${\cal T}$-interpretation $I = \langle I^f, X \rangle$ such that $I^f \models_{bg} T \cup \overline{T}$.

Now for any atom $p$, by the condition of $I$, we have $I \models p$ iff $(X,T) \models p$. Similarly, for any theory atom $t$ occurring in $\Pi$, by the condition of $I$ and by Lemma~\ref{lem:ljninterp}, $I \models t$ iff $(X,T) \models t$. Thus, since $(X,T) \models \Pi$, $I \models \Pi$. 

We must now show that there is no interpretation $J$ such that $J <^{\sigma_p} I$ and $J \models \Pi^\mu{I}$. Take any $J <^{\sigma_p} I$. That is, $J = \langle I^f, Y \rangle$ such that $Y \subset X$. By Lemma~\ref{lem:niemelafsmreduct}, $J \models \Pi^\mu{I}$ iff $Y \models \Pi^{(X,T)}$ but since $(X,T)$ is an LJN-answer set of $\Pi$, $Y \not \models  \Pi^{(X,T)}$ and thus $J \not \models \Pi^\mu{I}$ so $I$ is a stable model of $\Pi$. 

($b$) Assume $I = \langle I^f, X \rangle$ is a stable model of $\Pi$. 

Now for any atom $p$, by definition of $(X,T)$, $(X,T) \models p$ iff $I \models p$. Similarly, for any theory atom $t$ occurring in $\Pi$, by the condition of $I$ and Lemma~\ref{lem:ljninterp}, $(X,T) \models t$ iff $I \models t$. Thus, since $I \models \Pi$, $(X,T) \models \Pi$. 

We must now show that there is no set of atoms $Y$ such that $Y \subset X$ and $Y \models \Pi^{(X,T)}$. Take any $Y \subset X$. By Lemma~\ref{lem:niemelafsmreduct}, $Y \models \Pi^{(X,T)}$ iff $J \models \Pi^\mu{I}$ where $J = \langle I^f,Y \rangle$. Since $J <^{\sigma^p} I$ and $I$ is a stable model of $\Pi$, 
$J \not \models \Pi^\mu{I}$. Thus $Y \not \models \Pi^{(X,T)}$ and so $(X,T)$ is an LJN-answer set of $\Pi$. \qed

\BOCCC
\subsection{Proof of Theorem~\ref{thm:lw-fsm}}
\cblu

\begin{lemma}\label{lem:lw-fsm-reduct}
Let 
$P$ be a LW-program, 
$F$ be the first-order representation of $P$, 
$I$ be an interpretation of the signature $\sigma$ of $P$ such that $I$ is a model of $P$.
and ${\bf p}$ be the list of all predicates in $\sigma$.
For any interpretation $J$ such that $J <^{\bf p} I$ and any set of atoms $K$ such that for any 
atom $A$, we have $A \in K$ iff $A^J = \true$, then $J \models F^\mu{I}$ iff $K$ satisfies $P^I$.
\end{lemma}

\proof

Since $F$ is comprised of a conjunctions of the form 
$$
B_1 \land \dots \land B_m \land (\neg \ C_1) \land \dots \land (\neg\ C_n) \rightarrow A
$$
$F^\mu{I}$ is by definition a conjunction of 
$$
(B_1 \land \dots \land B_m \land (\neg \ C_1) \land \dots \land (\neg\ C_n) \rightarrow A)^\mu{I}
$$
and so we will consider each conjunctive subformula $G_i$ of $F$ separately. Each conjunctive term $G_i$ corresponds to a rule $r_i$ in $P$ of the form
$$
A \leftarrow B_1,\dots,B_m, not\ C_1,\dots, not\ C_n.
$$
and so we simply need to show that for any conjunctive subformula $G_i$ of $F$ and the corresponding rule $r_i$, we have $J \models G_i^\mu{I}$ iff $K \models r_i^I$. In the context of this comparison, we note that when a rule is removed in $P^I$, this is equivalent to replacing the rule with $\top$. We consider the following cases for a rule $r$ and the corresponding conjunctive subformula $G$ in $F$.

\bi
\ii Case 1: There is some atomic formula $B_i$ in $r$ such that $B_i$ contains an equality and is not satisfied by $I$.
\\ In this case $r^I$ is replaced with $\top$ and so $K \models r^I$. In $G^\mu{I}$, it may be the case that some $(\neg\ C_k)$ is replaced by $\bot$ but $B_i$ will certainly remain as a conjunctive term in the precedent of the implication $G$ and so since $J$ agrees with $I$ on all functions, we have $J \not \models B_i$ and so $J \models G^\mu{I}$.
\ii Case 2: There are no atomic formulas $B_i$ in $r$ such that $B_i$ contains an equality and is not satisfied by $I$, but there is a conjunctive term $not\ C_k$ in $r$ such that $C_k^I = \true$.
\\ In this case $r^I$ is replaced with $\top$ and so $K \models r^I$. In $G^\mu{I}$, $(\neg\ C_k)$ will be replaced with $\bot$ and so $J \models G^\mu{I}$.
\ii Case 3: There are no atomic formulas $B_i$ in $r$ such that $B_i$ contains an equality and is not satisfied by $I$, and there is no conjunctive term $not\ C_k$ in $r$ such that $C_k^I = \true$.
\\ In this case, $r^I$ is obtained from $r$ by replacing each $f(t_1,\dots,t_m)$ with $c$ where $f^I(t_1,\dots,t_m) = c$ and from removing all conjunctive terms containing equality. On the other hand, there are two cases for $G^\mu{I}$: is either $\bot$ if $I \not \models G$ or $G^\mu{I}$ is precisely $G$ otherwise. However, the assumption that $I$ is a model of $P$ means that the former case cannot arise and so $G^\mu{I}$ is precisely $G$. 

Now, since $J$ and $I$ agree on all functions, we have that $J \models G^\mu{I}$ iff $J \models H$ where $H$ is obtained from $G^\mu{I}$ by replacing $f(t_1,\dots,t_m)$ with $c$ where $f^I(t_1,\dots,t_m) = c$. Also since $J$ and $I$ agree on all functions, and since we assumed in this case that $I$ satisfies all conjunctive terms containing equality, we have $J \models H$ iff $J \models H'$ where $H'$ is obtained from $H$ by removing all conjunctive terms containing equality. Now the only remaining difference between $H'$ and $r^I$ is that every remaining conjunctive term $not\ A$ in $H'$ is absent in $r^I$. Note that since we assumed for this case that there is no conjunctive term $not\ C_k$ in $r$ such that $C_k^I = \true$, it must be that every such conjunctive term is such that $A^I = \false$. However, since we have that $J <^{\bf p} I$, it must be that $A^J = \false$ and so $J \models \neg A$. Thus, $J \models H'$ iff $J \models H''$ where $H''$ is obtained from $H'$ by removing all conjunctive terms of the form $not\ A$. Now $H''$ is exactly the first-order representation of $r^I$ and since $J$ and $K$ agree on all predicates, it is clear that $J \models G^\mu{I}$ iff $K \models r^I$. \qed
\ei

\noindent{\textbf {Theorem~\ref{thm:lw-fsm} \optional{thm:lw-fsm}}}\
\ 
{\sl 
Let $P$ be an LW-program and let $F$ be the FOL-representation of the
set of rules in $P$. The following conditions are equivalent to each
other:
\begin{itemize}
\item[(a)]  $I$ is an answer set of $P$ in the sense of~\cite{lin08answer};
\item[(b)]  $I$ is a $P$-interpretation that satisfies $\fsm[F; {\bf
    p}]$ where ${\bf p}$ is the list of all predicate constants
  occurring in $F$. 
\end{itemize}
}\medskip

\proof
We will use the reduct-based characterization of the $\fsm$ semantics in this proof. When programs are restricted 

($\Rightarrow$) Let us assume $I$ is an answer set of $P$ in the sense of~\cite{lin08answer}.
We wish to show that $I$ satisfies $\fsm[F; {\bf
    p}]$ where ${\bf p}$ is the list of all predicate constants
  occurring in $F$. That is, we assume $I$ satisfies every rule in $P$ and there is no subset $K$ of atoms in $I$ such that $K$ is a model of $P^I$ and we wish to show that $I \models F$ and no interpretation $J$ such that $J <^{\bf p} I$ satisfies $F^\mu{I}$. Since $I$ is an answer set of $P$, $I$satisfies every rule of $P$ and so it immediately follows that $I \models F$. So it only remains to be shown that if there is no subset $K$ of atoms in $I$ such that $K$ is a model of $P^I$, then there is no interpretation $J$ such that $J <^{\bf p} I$ satisfies $F^\mu{I}$. To show this, we will consider the contrapositive; we assume there is some interpretation $J$ such that $J <^{\bf p} I$ satisfies $F^\mu{I}$ and will show that there is a subset $K$ of the atoms in $I$ such that $K$ is a model of $P^I$.

We first note that since $J <^{\bf p} I$, $J$ and $I$ differ only predicates so that $J^{pred}$ is a subset of $I^{pred}$. Thus, we will take $K = J^{pred}$ so that $K$ is a subset of the atoms in $I$ and show that $K$ is a model of $P^I$. Then the claim follows by Lemma~\ref{lem:lw-fsm-reduct}. 	

($\Leftarrow$) Let us assume $I$ satisfies $\fsm[F; {\bf p}]$ where ${\bf p}$ is the list of all predicate constants
occurring in $F$. We wish to show that $I$ is an answer set of $P$ in the sense of~\cite{lin08answer}. That is, we assume $I \models F$ and no interpretation $J$ such that $J <^{\bf p} I$ satisfies $F^\mu{I}$ and we wish to show that $I$ satisfies every rule in $P$ and there is no subset $K$ of atoms in $I$ such that $K$ is a model of $P^I$. Since we assume $I \models F$, then it follows that $I$satisfies every rule of $P$. So it only remains to be shown that if there is no interpretation $J$ such that $J <^{\bf p} I$ satisfies $F^\mu{I}$, then there is no subset $K$ of atoms in $I$ such that $K$ is a model of $P^I$. To show this, we will consider the contrapositive; we assume there a subset $K$ of the atoms in $I$ such that $K$ is a model of $P^I$ and we will show that there is some interpretation $J$ such that $J <^{\bf p} I$ satisfies $F^\mu{I}$.

We will take $J$ such that $J$ and $I$ agree on all functions and such that for any atomic formula $A$, we have $A^J = \true$ iff $A \in K$. It is clear that since $K$ is a subset of the atoms in $I$, that $J <^{\bf p} I$. The claim then follows by Lemma~\ref{lem:lw-fsm-reduct}. \qed

\cbla
\EOCCC


\subsection{Proof of Theorem~\ref{thm:lw-fsm}}

The proof of the theorem is rather obvious once we view the type declarations of LW-program as a special case of the many-sorted signature declarations. So we omit the proof here.

\end{document}